\documentclass[11pt]{article}
\usepackage[margin=1in]{geometry}
\usepackage{siunitx}
\usepackage{multirow}
\usepackage[numbers]{natbib}

\def\colorful{1}
\ifnum\colorful=1

\newcommand{\snote}[1]{\footnote{{\bf [Sushrut: {#1}\bf ] }}}
\newcommand{\gnote}[1]{\footnote{{\bf [Guyang: {#1}\bf ] }}}

\else

\newcommand{\inote}[1]{}
\newcommand{\anote}[1]{}
\newcommand{\tnote}[1]{}
\newcommand{\snote}[1]{}
\newcommand{\lnote}[1]{}
\newcommand{\gnote}[1]{}

\fi

\usepackage{listings}
\usepackage{xcolor} 
\usepackage{caption}
\usepackage[ruled, vlined, linesnumbered]{algorithm2e}
\usepackage[utf8]{inputenc}
\usepackage[T1]{fontenc}
\usepackage{url}
\usepackage{booktabs}
\usepackage{amsfonts}
\usepackage{nicefrac}
\usepackage{microtype}       
\usepackage{float}
\usepackage{amssymb, graphicx}
\usepackage{mathtools} 
\usepackage{amsthm}
\usepackage{thm-restate}

\usepackage[hypertexnames=false]{hyperref}       
\usepackage{cleveref} 
\usepackage{enumitem}

\newcounter{thm}
\numberwithin{thm}{section}

\theoremstyle{plain}
\newtheorem{theorem}[thm]{Theorem}

\newtheorem{lemma}[thm]{Lemma}

\newtheorem{claim}[thm]{Claim}
\newtheorem{assumption}[thm]{Assumption}
\newtheorem{fact}[thm]{Fact}

\theoremstyle{definition}
\newtheorem{definition}[thm]{Definition}
\newtheorem{problem}[thm]{Problem}

\theoremstyle{remark} 

\crefname{fact}{fact}{facts}
\crefname{claim}{claim}{claims}
\crefname{assumption}{assumption}{assumptions}
\crefname{problem}{problem}{problems}
\crefname{theorem}{theorem}{theorems}
\crefname{corollary}{corollary}{corollaries}
\crefname{lemma}{lemma}{lemmas}
\crefname{proposition}{proposition}{propositions}
\crefname{definition}{definition}{definitions}
\crefname{remark}{remark}{remarks}
\crefalias{AlgoLine}{line}

\usepackage{times}
\usepackage{bm}
\usepackage{dsfont}

\definecolor{darkgreen}{rgb}{0.0, 0.8, 0.0}

\def\1{\bm{1}}

\def\eps{{\epsilon}}

\def\vzero{{\bm{0}}}
\def\vone{{\bm{1}}}

\def\vlambda{{\bm{\lambda}}}
\def\va{{\bm{a}}}

\def\vl{{\bm{\ell}}}

\def\vu{{\bm{u}}}
\def\vv{{\bm{v}}}
\def\vw{{\bm{w}}}
\def\vx{{\bm{x}}}
\def\vy{{\bm{y}}}

\usepackage{mathrsfs}
\usepackage[cal=boondoxo]{mathalfa}
\def\ep{\widehat{\mathcal{p}}}

\def\pp{{\mathcal{p}}}

\def\vxh{\widehat{\bm{x}}}

\def\mI{{\bm{I}}}

\def\cB{{\mathcal{B}}}

\def\cD{{\mathcal{D}}}
\def\cE{{\mathcal{E}}}
\def\cF{{\mathcal{F}}}
\def\cG{{\mathcal{G}}}

\def\cL{{\mathcal{L}}}

\def\cN{{\mathcal{N}}}

\newcommand{\E}{\mathbb{E}}

\newcommand{\R}{\mathbb{R}}

\newcommand{\Ind}{\mathbb{I}}

\DeclareMathOperator*{\argmax}{arg\,max}
\DeclareMathOperator*{\argmin}{arg\,min}

\newcommand{\innp}[1]{\langle #1 \rangle}

\newcommand{\Gap}{\mathrm{Gap}}

\newcommand{\dd}{\mathrm{d}}

\DeclareMathOperator{\op}{op}
\usepackage{xparse}
\NewDocumentCommand{\DeclarePairedDelimiterWithOptionalStar}{ m m m m }{
  \DeclareDocumentCommand{#1}{ s o m }{\IfBooleanTF{##1}{#2#3*{##3}#4}{\IfValueTF{##2}{#2#3[##2]{##3}#4}{#2#3{##3}#4}}}}

\DeclarePairedDelimiter\brackets{(}{)}
\DeclarePairedDelimiter\norm{\|}{\|}
\DeclarePairedDelimiter\abs{|}{|}
\DeclarePairedDelimiterX{\divx}[2]{(}{)}{#1\;\delimsize\|\;#2}

\DeclarePairedDelimiterWithOptionalStar{\bigO}{O}{\brackets}{}
\DeclarePairedDelimiterWithOptionalStar{\bigtO}{\widetilde{O}}{\brackets}{}
\DeclarePairedDelimiterWithOptionalStar{\bigOmega}{\Omega}{\brackets}{}
\DeclarePairedDelimiterWithOptionalStar{\bigtOmega}{\widetilde{\Omega}}{\brackets}{}
\DeclarePairedDelimiterWithOptionalStar{\opnorm}{}{\norm}{_{\op}}
\DeclarePairedDelimiterWithOptionalStar{\fnorm}{}{\norm}{_{F}}

\makeatletter
\providecommand*{\diff}{
  \@ifnextchar^{\DIfF}{\DIfF^{}}}

\def\DIfF^#1{\mathop{\mathrm{\mathstrut d}}\nolimits^{#1}\gobblespace}
\def\gobblespace{\futurelet\diffarg\opspace}
\def\opspace{\let\DiffSpace\!\ifx\diffarg(\let\DiffSpace\relax
  \else
  \ifx\diffarg[\let\DiffSpace\relax
  \else
  \ifx\diffarg\{\let\DiffSpace\relax
  \fi\fi\fi\DiffSpace}

\makeatother

\newcommand{\lambdai}{\lambda_{[i]}}
\newcommand{\ppi}{{\pp_{[i]}}}

\newcommand{\vxij}{\vx_{[i]}^{(j)}}

\newcommand{\yij}{y_{[i]}^{(j)}}

\newcommand{\epi}{\ep_{[i]}}
\newcommand{\elambda}{\widehat{\lambda}}
\newcommand{\evlambda}{\widehat{\vlambda}}

\DeclareMathOperator{\OPT}{OPT}
\DeclareMathOperator{\mOPT}{OPT_m}
\DeclareMathOperator{\LHS}{LHS}
\newcommand{\OPThat}{\widehat{\OPT}}
\newcommand{\mOPThat}{\widehat{\mOPT}}

\def\mathbf{\bm}

\title{Robust Learning of a Group DRO Neuron}
\author{
Guyang Cao\thanks{Supported in part by the NSF CAREER Award CCF-2440563.} \\
University of Wisconsin-Madison\\
{\tt kafka\_lover@cs.wisc.edu }\\
\and 
Shuyao Li\thanks{Supported in part by  AFOSR Awards FA9550-21-1-0084 and FA9550-24-1-0076, and by NSF CAREER Award CCF-2440563.} \\
University of Wisconsin-Madison\\
{\tt shuyao.li@wisc.edu}\\
\and 
Sushrut Karmalkar\\
Microsoft Research, Cambridge\\
{\tt skarmalkar@microsoft.com}\\
\and 
Jelena Diakonikolas\thanks{Supported in part by the Air Force Office of Scientific Research under award number FA9550-24-1-0076, the NSF CAREER Award CCF-2440563, and the NSF MFAI Award DMS-2502282. Any opinions, findings and conclusions or recommendations expressed in this material are those of the author(s) and do not necessarily reflect the views of the U.S.\ Department of Defense.}\\
University of Wisconsin-Madison\\
{\tt jelena@cs.wisc.edu}
}
\date{}

\begin{document}
\maketitle

\setcounter{page}{0}

\thispagestyle{empty}

\begin{abstract}
We study the problem of learning a single neuron under standard squared loss in the presence of arbitrary label noise and group-level distributional shifts, for a broad family of covariate distributions. Our goal is to identify a ``best-fit'' neuron parameterized by $\mathbf{w}_*$ that performs well under the most challenging reweighting of the groups. Specifically, we address a Group Distributionally Robust Optimization problem: given sample access to $K$ distinct distributions $\pp_{[1]},\dots,\pp_{[K]}$, we seek to approximate $\mathbf{w}_*$ that minimizes the worst-case objective over convex combinations of group distributions $\boldsymbol{\lambda} \in \Delta_K$, where the objective is $\sum_{i \in [K]}\lambdai\,\E_{(\mathbf x,y)\sim\pp_{[i]}}(\sigma(\mathbf w\cdot\mathbf x)-y)^2 - \nu d_f(\boldsymbol\lambda,\tfrac1K\mathbf1)$ and $d_f$ is an $f$-divergence that imposes (optional) penalty on deviations from uniform group weights, scaled by a parameter $\nu \geq 0$. 
We develop a computationally efficient primal-dual algorithm that outputs a vector $\widehat{\mathbf w}$ that is constant-factor
competitive with $\mathbf{w}_*$ under the worst-case group weighting.  
Our analytical framework directly confronts the inherent nonconvexity of the loss function, providing robust learning guarantees in the face of arbitrary label corruptions and group-specific distributional shifts. The implementation of the dual extrapolation update motivated by our algorithmic framework shows promise on LLM pre-training benchmarks. 
\end{abstract}

\newpage 
\section{Introduction}\label{sec:introduction}

The challenge of ensuring model robustness against distributional shifts 
is a central theme in modern machine learning. 
Group Distributionally Robust Optimization (Group DRO) 
has emerged as a principled framework to address such challenges, 
particularly in settings with heterogeneous data from distinct subpopulations (or groups) 
\citep{hashimoto2018fairness,oren2019distributionally,sagawa2020groupDRO}. 
The objective of Group DRO is to learn a model that minimizes 
its loss under the worst-case reweighting of these groups, 
thereby guarding against poor performance on any single subpopulation. 
Group DRO has enjoyed significant empirical success in large-scale applications; 
for instance, \citep{xie2023doremi, xia2024shearedllamaacceleratinglanguage} leverage dynamic reweighting 
of data domains to improve the performance 
of large language models. 
Despite these practical advances, 
most of the existing theoretical guarantees for DRO are limited to convex optimization problems. 
This leaves a critical gap in our understanding 
of the nonconvex landscapes that characterize deep learning, 
where these methods are most impactful.

Alas, even \emph{without} distributional robustness, 
learning guarantees for nonconvex models are nontrivial.
A canonical example is the classical problem of learning a single neuron \citep{rosenblatt1958perceptron, nelder1972generalized}: namely, a function of the form $\sigma(\vw_*^\top \vx),$ where $\sigma$ is a known activation function (e.g., ReLU: $\sigma(t) = \max\{0, t\}$), $\vw_* \in \R^d$ is the unknown parameter vector, and $\vx \in \R^d$ is the data vector; the goal of a learner is to minimize the mean squared loss $\cL_2(\vw) := \E_{(\vx, y)\sim \cD}[(\sigma(\vw^\top \vx) - y)^2]$ over a centered Euclidean ball of given radius $W,$ denoted by $\cB(W)$.   

This problem is already computationally challenging
in general because of the inherent nonconvexity of the squared loss 
for common activations like ReLU. 
Without any distributional assumptions imposed on the labeled examples and for standard activations like the sigmoid and ReLU, robust learning is NP-hard even if we only require constant-factor approximation for the minimum mean squared loss  \citep{vsima2002training,MR18}. 
Under more structured conditions, tractability can be recovered: if the labels are \emph{realizable} (i.e., if for labeled pairs $(\vx, y)$, $y = \sigma(\vw_*^\top \vx)$ for some fixed activation $\sigma$ and parameter vector $\vw_*$) or exhibit zero-mean, bounded-variance noise, fairly mild assumptions on the $\vx$-marginal distribution and activation $\sigma$ suffice for minimizing the mean squared loss to error $\epsilon > 0$ in polynomial time (in $d, 1/\epsilon$ and other problem parameters)  
\citep{kalai2009isotron,kakade2011efficient,soltanolkotabi2017learning}.
However, once arbitrary label noise is introduced---a realistic setting accounting for possible model misspecification and non-structured noise classically studied in the context of agnostic learning \citep{Haussler:92,kearns1992toward}---the problem again becomes intractable: 
even in the Gaussian setting, achieving additive error $\epsilon > 0$ 
requires $d^{\mathrm{poly}(1/\epsilon)}$ time, 
ruling out any polynomial-time algorithm 
\citep{GoelKK19,DKZ20-sq-reg,GGK20,DKPZ21-SQ,diakonikolas2023near}. 
These hardness results highlight the necessity of 
designing efficient constant-factor approximation algorithms 
under structural assumptions imposed on the $\vx$-marginal distribution and the class of activation functions.

Our work addresses the problem of learning a single neuron 
in a setting that combines two significant challenges: 
adversarial label noise and group-level distributional shifts. 
We consider a scenario where training data originate from \(K\) groups, 
each associated with an unknown distribution \(\pp_{[1]},\dots,\pp_{[K]}\), 
while the group memberships of examples are known. 
The reference distribution, representing an unperturbed state, 
is an equal mixture \(\pp_0=\tfrac1K\sum_{i=1}^K\pp_{[i]}\). 

Our aim is to determine a weight vector \(\vw\in\cB(W)\) parameterizing a neuron  
whose loss is robust to the worst-case reweighting of these groups. 
We formalize this by seeking a \(\vw\) 
minimizing the squared loss under a worst-case 
convex combination of group distributions, 
where deviations from a uniform weighting are penalized by an \(f\)-divergence. 
This leads to the following optimization problem:
\[
\min_{\vw \in \cB(W)}\max_{\vlambda\in\Delta_K}
\sum_{i=1}^K\lambdai\,\E_{\pp_{[i]}}\bigl(\sigma(\vw^\top\vx)-y\bigr)^2
\;-\;\nu\,d_f\!\big(\vlambda,\tfrac1K\mathbf1\big).
\]
Here, \(\sigma:\R\to\R\) denotes a known activation function such as ReLU (see \Cref{def:activation} for the definition of activations captured by our results), 
\(\nu \ge 0\) is a regularization parameter controlling the robustness trade-off, 
and \(d_f\) is a strongly convex \(f\)-divergence that quantifies the discrepancy between 
the learned group-weight vector \(\vlambda\) and a uniform weighting. 
This formulation captures both 
\emph{adversarial label noise} (within groups) 
and \emph{distributional shifts} (across groups).

For this setting, we develop the first computationally efficient algorithm with provable guarantees. 
Our algorithm is primal-dual and it outputs an estimated parameter \(\widehat\vw\) 
that is competitive with the optimal parameter vector \(\vw_*\). 
Specifically, for the worst-case reweighting \(\vlambda^*\) 
corresponding to \(\vw_*\), our estimate satisfies, for a constant $C > 1$,
\begin{equation}\notag
    \begin{aligned}
        &\; \sum_{i \in [K]} \lambdai^* \E_{\pp_{[i]}}(\sigma(\widehat{\mathbf w}\cdot\mathbf x)-y)^2\\
        \le \; &C\,\max_{i \in [K]} \E_{\pp_{[i]}}(\sigma(\mathbf w_*\cdot\mathbf x)-y)^2 + \epsilon.
    \end{aligned}
\end{equation}
We recall here that, as discussed earlier in the introduction, even for $K=1$ (no group shifts), obtaining error with $C=1$ is not possible for polynomial-time algorithms in the considered setting. 
Our analytical framework handles the structured nonconvexity of the problem 
while utilizing memory-efficient dual-side extrapolation, 
making our approach well-suited for large-scale applications.

We remark here that the only prior work that handles both adversarial label noise and distributional shifts is the recent work \citep{Li2024shifts}. Compared to \citep{Li2024shifts}, our work applies to a different type of distributional shifts (group shifts in place of distributional ambiguity across all examples, which is better aligned with recent applications like \citep{xie2023doremi,xia2024shearedllamaacceleratinglanguage}); it removes higher-moment loss assumptions and applies even when there is no penalization ($\nu = 0$); it is not restricted to $\chi^2$-divergence; and it employs extrapolation (momentum) on the \emph{dual} instead of the \emph{primal} side, which has implications on ease of implementation as discussed later in the paper.  

In the rest of the section, we introduce the necessary background to formally state our main result and provide a technical overview. 
A detailed discussion of related literature is provided in \Cref{sec:related_work}. 

\subsection{Problem Setup}\label{subsec:problem_setup}

For two discrete probability vectors \(\vlambda, \vlambda' \in \Delta_K\), the \(f\)-divergence is defined as 
\(
d_f(\vlambda, \vlambda') = \sum_{i=1}^K \lambda'_{[i]} f\big(\frac{\lambda_{[i]}}{\lambda'_{[i]}}\big), 
\) 
where \(f:[0,\infty)\to\R_+\) is a convex function satisfying \(f(1)=0\). We focus on \(f\)-divergences for which the mapping \(\vlambda \mapsto d_f(\vlambda, \vlambda_0)\) is \emph{strongly convex} with respect to a relevant norm over the simplex \(\Delta_K\); these include some of the most commonly used examples like the \(\chi^2\)-divergence and Kullback-Leibler (KL) divergence.

Same as \citep{Li2024shifts}, our analysis applies to a broad class of activations that are convex and $(\alpha, \beta)$-unbounded. 
\begin{definition}[Unbounded Activation \citep{DKTZ22}]\label{def:activation}
A function \(\sigma:\R\to\R\) is \emph{\((\alpha,\beta)\)-unbounded} if it is non-decreasing and satisfies: (i) \(\sigma\) is \(\beta\)-Lipschitz continuous; (ii) for some \(\alpha > 0\), \(\sigma(t_1)-\sigma(t_2)\ge \alpha\,(t_1-t_2)\) holds for all \(t_1\ge t_2\ge0\); (iii) \(\sigma(0)=0\).
\end{definition}

To make this challenging nonconvex problem tractable, we build upon the approach of prior work on robust learning of a neuron~\citep{wang2023robustly-learning,Li2024shifts}---by imposing structural properties on the $\vx$-marginal distributions across groups. 
The first of those assumptions is subexponential concentration of 1D projections:
\begin{assumption}[Sub-Exponential Tails]\label{assump:concentration}
There is a constant \(B>0\) such that for every \(i \in [K]\) and every unit vector \(\vu \in \cB(1)\), for all \(r \ge 1\),
\[ \Pr_{\vx \sim \pp_{\vx[i]}}\bigl(|\vu \cdot \vx| \ge r\bigr) \le \exp(-r/B). \]
\end{assumption}
Additionally, similar to \citep{wang2023robustly-learning,Li2024shifts}, we impose the following margin condition. 
\begin{assumption}[Uniform Margin]\label{assump:margin}
There exist constants \(\zeta, \gamma \in (0,1]\) such that for every group \(i \in [K]\) and every vector \(\vw \in \R^d\),
\begin{equation}
    \E_{\vx \sim \pp_{\vx[i]}}\!\bigl[\vx\vx^\top \Ind\{\vw \cdot \vx \ge \gamma\|\vw\|_2\}\bigr] \succeq \zeta \mI_d. \label{eq:margin-w}
\end{equation}
\end{assumption}
This means that the covariance matrix within the ``margin region'' $\{\vx : \vw \cdot \vx \ge \gamma\|\vw\|_2\}$ is uniformly well-conditioned across all groups and weight vectors $\vw$. 
We remark here that \eqref{eq:margin-w} is required for \emph{every} weight vector $\vw$, making it a stronger condition than its counterparts in \citep{wang2023robustly-learning,Li2024shifts}, which only enforced such a condition for the target parameter $\vw_*$. Nevertheless, this stronger condition still captures all well-concentrated distributions discussed in \citep{wang2023robustly-learning} including well-behaved distributions from \citep{DKTZ22} (which  include log-concave and $s$-concave distributions), discrete Gaussians, and the uniform distribution on $\{-1, 0, 1\}^d$, because the definitions of those distributions do not involve $\vw_*$.

With these components, we can formally define our regularized Group DRO problem.
\begin{definition}[Loss, Risk, and \(\mOPT\)]\label{def:loss-risk-opt}
Let \(\cB(W)=\{\vw \in \R^d :\|\vw\|_2\le W\}\) be the feasible set for weight vectors and let \(\pp_{[1]},\dots,\pp_{[K]}\) be the \(K\) group distributions. For any \(\vw\in\cB(W)\) and \(\vlambda\in\Delta_K\), the regularized loss is:
\[
\bar L(\vw,\vlambda)
=\sum_{i=1}^K\lambda_{[i]}\,\E_{(\vx,y)\sim\pp_{[i]}}\!\bigl(\sigma(\vw\cdot\vx)-y\bigr)^2
-\nu\,d_f\big(\vlambda,\tfrac1K\mathbf1\big).
\]
The DRO risk for a given \(\vw\) is the maximum loss over all possible group weightings:
\begin{equation}\label{eq:population-risk}
    R(\vw)=\max_{\vlambda\in\Delta_K}\bar L(\vw,\vlambda).
\end{equation}
Let \(\vw_*=\argmin_{\vw\in\cB(W)}R(\vw)\) be the optimal weight vector and \(\vlambda^*\) be its corresponding worst-case group weights. Our performance benchmark is the unregularized loss of \(\vw_*\) on its single worst-performing group:
\[
\mOPT =\max_{i\in[K]}\E_{(\vx,y)\sim\pp_{[i]}}\!\bigl(\sigma(\vw_*\!\cdot\!\vx)-y\bigr)^2.
\]
\end{definition}

\begin{problem}[Robust Learning a Group DRO Neuron]\label{prob:group-dro}
Given a convex $(\alpha, \beta)$-unbounded activation $\sigma,$ error parameters \(\epsilon,\delta\in(0,1)\), regularization parameter \(\nu \ge 0\), weight radius \(W>0\), and sample access to labeled examples from each of the \(K\) group distributions \(\pp_{[1]},\dots,\pp_{[K]}\), the goal is to output a parameter vector \(\widehat{\vw}\in\cB(W)\) such that, with probability at least \(1-\delta\),
\(
\|\widehat{\vw}-\vw_*\|_2^2 \;\le\; C\,\mOPT \;+\;\epsilon
\)
for a universal constant $C > 1.$ 
\end{problem}
As we later argue (see \Cref{eq:square-loss-bound}), \(
\|\widehat{\vw}-\vw_*\|_2^2 \;\le\; C\,\mOPT \;+\;\epsilon
\) in turn implies a loss value bound of \(\sum_{i \in [K]} \lambda^*_{[i]} \E_{\pp_{[i]}}(\sigma(\widehat{\mathbf w}\cdot\mathbf x)-y)^2 \le C' \mOPT+ \epsilon\), where \(\vw_*\), \(\vlambda^*\), and \(\mOPT\) are from Definition~\ref{def:loss-risk-opt}, and \(C'> 1\) is a universal constant.

\subsection{Main Result}\label{subsec:main_result}
As stated earlier, our main contribution is an efficient algorithm for robust learning of a single neuron under adversarial label noise and group distributional shifts. Below, notation $\tilde{O}_c(\cdot)$ hides poly-logarithmic dependence in the argument and polynomial dependence in parameters $c$.

\begin{theorem}[Informal; see \Cref{thm:main-formal}]\label{thm:main-informal}
Suppose the underlying group distributions \(\ppi,\, i \in [K]\) satisfy appropriate margin and concentration properties (\Cref{assump:margin,assump:concentration}) and the learner is provided with \(\tilde{O}_{\beta,\nu, W}(\log K \log(1/\delta) d/\epsilon^2)\) samples from each of the \(\ppi,\, i \in [K]\). 
Then \Cref{alg:main}, after \(\tilde O_W(\min\{\log(1/\epsilon)\sqrt{1/\nu},\sqrt{K}/\epsilon\})\) iterations, each running in near-linear time in the sample size, returns \(\hat{\vw}\in\cB(W)\) that, with probability at least \(1-\delta\), satisfies \(\|\hat{\vw}-\vw_*\|_2^2 \le C\,\mOPT + \epsilon\), where \(C>1\) is an absolute constant independent of  \(d, K, W, \epsilon, \nu\).
\end{theorem} 

Notably, our analysis shows that the algorithm attains sample complexity  \(\tilde O(K d/\eps^2)\), which matches the known optimal rate for convex unregularized Group DRO problems \citep{soma2022optimal,Zhang2023Stochastic} up to log factors.

\subsection{Technical Overview}\label{subsec:technical_overview}

Using standard uniform convergence arguments, we reduce the problem to solving its empirical version with a sufficiently large per-group sample size. The resulting problem is a challenging non-bilinearly coupled  nonconvex-concave saddle-point problem over the empirical mixtures:
\[
\min_{\vw \in \cB(W)}\;\max_{\evlambda\in\Delta_K} L(\vw, \evlambda),
\]
where \(L(\vw,\evlambda)
=\sum_{i=1}^K\elambda_{[i]}\,\E_{(\vx,y)\sim\epi}\!\bigl(\sigma(\vw\cdot\vx)-y\bigr)^2
\;-\;\nu\,d_f\!\big(\evlambda,\tfrac1K\mathbf1\big)\) 
and $\epi$ represents the empirical distribution formed by samples drawn from group $i$.

Even for convex-concave objectives, the nonlinear-linear coupling  between the primal and the dual as in our setting makes the analysis of primal-dual-style methods challenging and is still actively explored in current research; see \citep{mehta2025min} and references therein. On the other hand, existing approaches to min-max problems with a nonconvex objective over the primal variables like \citep{lin2020gradient,zhang2020single,rafique2022weakly,li2025nonsmooth} only offer stationarity guarantees, which are known to be insufficient for formal learning guarantees of a neuron even in much less challenging settings without distributional shifts \citep{yehudai2020learning}. 

The most closely related work to ours, \citep{Li2024shifts}, handles a similar but distinct problem. They showed how the structured nonconvexity of learning a single neuron can be leveraged 
to obtain a provably convergent primal-dual algorithm for standard $\chi^2$-divergence-based DRO.
As in their paper, to analyze progress towards the solution, 
we define a \emph{gap function} that measures the suboptimality of an iterate pair 
\((\vw_t,\evlambda_t)\) relative to a hybrid reference point \((\vw_*,\evlambda^*)\), and track progress in the gap function per iteration involving one primal and one dual update.

A key technical challenge lies in bounding the primal gap due to the nonconvexity of the loss. 
Here, the analysis in \citep{Li2024shifts} is insufficient for our goals. 
The analysis is limited to DRO with a $\chi^2$ penalty, 
requires a large regularization parameter $\nu$, 
and imposes fourth-moment assumptions on the loss. 
Our work extends this to the more practical \emph{group} DRO setting, 
accommodates both KL and $\chi^2$ divergences, 
and removes these restrictive assumptions.

From a practical standpoint, the primal extrapolation used in \citep{Li2024shifts} is also undesirable.
Primal extrapolation is memory-intensive, 
because the algorithm has to remember 
two previous primal variables $\vw \in \R^d$ of dimension $d$ 
to compute the extrapolation step. 
It is also unclear how to implement primal extrapolation 
together with widely used off-the-shelf solvers such as Adam~\citep{kingma2015adam}.
Our algorithm instead performs extrapolation on the dual variable $\vlambda \in \R^K$. 
This is more efficient as typically $K \ll d$. 

However, dual-side extrapolation requires a different analysis. 
A key difficulty is that the extrapolated dual vector is not guaranteed to have nonnegative entries. In typical analyses involving extrapolation steps, extrapolation is either performed before further bounding the objective (see, e.g., \citep{chambolle2011first,kotsalis2022simple}) or it is done on both the primal and the dual side \citep{mehta2025min}, which in effect allows for the induced error terms to be telescoped or canceled out. Observe that the coupled term in the objective $L(\vw,\evlambda)$ is of the form $\evlambda^\top F(\vw),$ for a vector-valued mapping $F$ whose each coordinate $(F(\vw))_i$ is a nonconvex function of $\vw$. Even if each $(F(\vw))_i$ were convex (so that the mapping $\evlambda^\top F(\vw)$ is convex for feasible $\evlambda$), once we replace $\evlambda$ with an extrapolated vector $\bar{\vlambda}$ which is no longer guaranteed to be nonnegative, the mapping $\bar{\vlambda}^\top F(\vw)$ would no longer be guaranteed to be convex and so typical inequalities involving convexity would no longer apply to bounding $\bar{\vlambda}^\top F(\vw)$ below, which is needed in the analysis. We overcome this issue by first bounding $\evlambda^\top F(\vw)$ (using linearization, discussed next) and then  applying extrapolation. This introduces nontrivial error terms that we bound by repeatedly leveraging the structural properties of the considered problem (see \Cref{sec:gap-upper-bound2}).  

Additionally, since in our case $(F(\vw))_i$ is nonconvex, we need an appropriate strategy for bounding below ${\evlambda}^\top F(\vw)$. As mentioned earlier, it is possible to use the results from \citep{Li2024shifts} leveraging structured nonconvexity of the neuron mean squared loss for this purpose. However, as already discussed, this would lead to a requirement for much stronger assumptions. Instead, we prove a key technical result  (``linearization,'' \Cref{lemma:main-auxiliary-main-body}), which  leverages the specific group-wise structure of the problem to usefully bound below the mean squared loss function.
This result has two significant advantages:
(1) It avoids the complex, higher-order error terms that appear in prior work, 
	which depend on the divergence between dual iterates. 
	This simplifies the overall analysis considerably and allows us to address alternative divergences like KL.
    (2) It removes the requirement for a non-trivial lower bound on the regularization parameter \(\nu\). 
	This allows our framework to handle the \(\nu \to 0\) regime, 
	smoothly connecting our robust formulation to classical, non-regularized Group DRO.

By combining this improved bound with sharpness properties of the loss function on the target mixture distribution \(\pp^*\), 
we then carefully control the accumulated error terms and prove that the iterates converge to a solution that 
is competitive with the ``best-fit'' neuron parameterized by \(\vw_*\).

\section{Preliminaries}\label{sec:prelims}
For a positive integer $N$, $[N] := \{1, \dots, N\}$. 
If $\cE$ is a subset of some ambient universe then $\cE^c$ denotes its complement, and $\Ind_{\cE}(x)=1_{\{x\in\cE\}}$ is its indicator function.  
For vectors $\vx, \vxh \in \R^d$, $\langle\vx, \vxh\rangle = \vx \cdot \vxh = \vx^\top \vxh$ is their inner product, and $\|\vx\|_2$ is the $\ell_2$ norm; we write $\vx\le\vxh$ to mean $x^{(j)}\le\hat x^{(j)}$ coordinate-wise.  
$\mI_d$ denotes the $d \times d$ identity matrix.
$A \succeq B$ means that $A-B$ is positive semidefinite.
The iteration index is denoted by $t$. 
The group weight vector at iteration $t$ is $\vlambda_t = [\lambda_{t[1]}, \dots, \lambda_{t[K]}]^\top \in \Delta_K$, where $\Delta_K := \{\vlambda \in \R^K : \sum_{i=1}^K \lambda_{t[i]} = 1, \lambda_{t[i]} \ge 0 \text{ for all } i \in [K]\}$ is the probability simplex in $\R^K$.
 For two distributions \(\pp\) and \(\pp'\), we use \(\pp\ll\pp'\) to denote that \(\pp\) is absolutely continuous with respect to \(\pp'\), which means that for all measurable sets \(A\), \(\pp'(A)=0\) implies \(\pp(A)=0\). 
For probability measures $\pp\ll\pp'$, we write $\tfrac{\dd\pp}{\dd\pp'}$ for the Radon-Nikodym derivative, and  
 use $\chi^2(\pp,\pp'):= \int(\frac{\mathrm{d}\pp}{\mathrm{d}\pp'} - 1)^2\mathrm{d}\pp'$ and $\mathrm{KL}(\pp,\pp'):= \int\log (\frac{\mathrm{d}\pp}{\mathrm{d}\pp'})\mathrm{d}\pp$ to denote the chi-squared ($\chi^2$) and KL divergences of $\pp$ relative to $\pp'$.

We now state several facts used throughout our analysis. 
The population loss exhibits sharpness under our distributional assumptions \citep{wang2023robustly-learning}:

\begin{fact}[Population Sharpness and Moment Bounds {\citep{wang2023robustly-learning}}]\label{fact:sharpness}
Let $\pp$ satisfy \Cref{assump:concentration}. Define
$c_0 := \frac{\gamma \zeta \alpha}{6B\log(20B/\zeta^2)}$.
Then
\( \E_{\vx \sim \pp_{\vx[i]}}\bigl[\bigl(\sigma(\vw \cdot \vx)-\sigma(\vw_* \cdot \vx)\bigr) (\vw \cdot \vx-\vw_* \cdot \vx)\bigr] \ge c_0 \|\vw-\vw_*\|_2^2, \)
and for any unit vector $\vu$ and $\tau \in \{2,4\}$, we have 
$\E_{\vx \sim \pp_{\vx[i]}}\bigl[(\vu \cdot \vx)^\tau\bigr] \le 5B.$
\end{fact}

Uniform convergence arguments extend these bounds to empirical distributions (up to constants) when the per-group sample size $N$ is sufficiently large:

\begin{lemma}[Empirical Sharpness and Moment Bounds; Informal. See \Cref{lemma:sharpness_empirical_heavy_tailed}]\label{lemma:sharpness-empirical}
Under \Cref{assump:margin,assump:concentration}, 
if the per-group sample size $N/K$ is sufficiently large (dependent on $\beta, B, W, \nu, d, K, \delta$), then with high probability, for all groups $i \in [K]$, all $\vw \in \cB(3\|\vw_*\|_2)$ with $\|\vw - \vw_*\|_2 \ge \sqrt{\eps}$, and any unit vector $\vu$:
\begin{align}
    \;&\E_{\vx \sim \ep_{\vx[i]}}[(\sigma(\vw \cdot \vx) - \sigma(\vw_* \cdot \vx))(\vw \cdot \vx - \vw_* \cdot \vx)] \notag \\ 
    &\; \ge   (c_0/2)\|\vw - \vw_*\|_2^2, \label{eq:sharpness-empirical} \\
    \;&\E_{\vx \sim \ep_{\vx[i]}}[(\vx \cdot \vu)^{\tau}] \le 6B \quad  \text{for } \tau \in \{2, 4\}. \label{eq:moment-bounds-empirical}
\end{align}
\end{lemma}

A direct consequence of \Cref{lemma:sharpness-empirical}, using Cauchy-Schwarz inequality and $\beta$-Lipschitzness of $\sigma$, is the following two-sided bound. For $c_1 := c_0^2/(24B)$ and any $\vw \in \cB(W)$:
\begin{align}\label{eq:relu-square-bound}
c_1\|\vw-\vw_*\|_2^2 &\le \E_{\vx \sim \ep_{\vx[i]}}\bigl[\bigl(\sigma(\vw \cdot \vx)-\sigma(\vw_* \cdot \vx)\bigr)^2\bigr] \notag \\
&\le 6B\beta^2 \|\vw-\vw_*\|_2^2.
\end{align}

Similar to prior work, we assume that the labels are bounded by a sufficiently large parameter $M = O(W B \beta \log(\beta B W / \epsilon))$. This assumption is without loss of generality (as established by the following fact) and can be ensured by simple pre-processing of labeled examples given to the algorithm. Thus, in the rest of the paper, we assume that the labels are bounded by $M$.
\begin{fact}[Label Truncation {\citep{wang2023robustly-learning}}]\label{fact:truncation}
Under \Cref{assump:margin,assump:concentration}, let $y' = \mathrm{sign}(y)\max\{|y|,M\}$ with
$M = C_M W B \beta \log(\beta B W / \epsilon),$
for a sufficiently large constant $C_M$. Then for all $i \in [K]$ and all $\vw \in \cB(W)$, it holds that
$\E_{\ppi}(\sigma(\vw \cdot \vx)-y')^2 \le \E_{\ppi}(\sigma(\vw \cdot \vx)-y)^2 + \epsilon$. 
\end{fact}

We recall the following first-order optimality condition:
\begin{fact}[First-Order Optimality]\label{fact:firstOrderNecessary}
If $f$ is continuously differentiable on a closed convex set $\Omega$ and $\vx^* \in \Omega$ is a local maximizer of $f$ over $\Omega$, then $\langle \nabla f(\vx^*), \vy-\vx^* \rangle \le 0$ for all $\vy \in \Omega$. If $f$ is also concave, this condition implies that $\vx^*$ is a global maximizer.
\end{fact}

For a differentiable function $\phi: \R^N \to \R$, the Bregman divergence is $D_{\phi}(\vy, \vx) = \phi(\vy) - \phi(\vx) - \langle \nabla \phi(\vx), \vy - \vx \rangle$. Its invariance to affine addition is also useful:
\begin{restatable}{fact}{bregmanblind}\label{fact:bregman-linear-blind}
If \(\psi(\vx)=\phi(\vx)+\langle\va,\vx\rangle+b\), then
\(D_\psi(\vy,\vx)=D_\phi(\vy,\vx)\)
for all \(\vx,\vy\).
\end{restatable}

\section{Convergence Analysis}\label{sec:algor-conv-analys}

This section presents our primal-dual algorithm and establishes its convergence guarantees. For clarity, we first define the core notation used in the analysis. Let $\ell(\vw;\vx,y) := (\sigma(\vw\cdot\vx)-y)^2$ denote the squared loss. Since $\sigma$ may not be differentiable, our algorithm relies on the \emph{surrogate gradient}  $\vv(\vw;\vx,y) := 2\beta(\sigma(\vw\cdot\vx)-y)\vx$, which serves as a well-behaved proxy for the true gradient, as in \citep{kakade2011efficient,DGKKS20,wang2023robustly-learning,Li2024shifts}. 
Our analysis measures performance against \emph{empirical} benchmarks defined with respect to the \emph{population-optimal} weight vector $\vw_*$; this seeming mismatch is important to our convergence analysis. Define:
\[
\widehat\mOPT := \max_{i \in [K]}\E_{(\vx,y)\sim\epi}[\ell(\vw_*;\vx,y)].
\]
Let $\evlambda^* = \argmax_{\evlambda \in \Delta_K} L(\vw_*, \evlambda)$ be the worst-case group weighting for $\vw_*$ over the \emph{empirical} distributions $\epi$. 
We also write $\phi(\evlambda) := d_f(\evlambda, \tfrac{1}{K}\mathbf{1})$ for the divergence penalty.

\begin{algorithm}[htp]
\caption{A Primal-Dual Algorithm for Group DRO}
\label{alg:main}
\KwIn{Sample sets ${(\vxij, \yij)}_{j=1}^{N/K}$ for $i \in [K]$; parameters $\nu \ge 0, W > 0, \epsilon > 0, \beta, B, c_1$.}
\textbf{Initialization:} $\nu_0 = \epsilon/(4K)$, $A_{-1} = a_{-1} = A_0 = a_0 = 0$, $\vw_{-1} = \vw_0 = \mathbf{0}$, $\evlambda_{-1} = \evlambda_0 = \frac{\vone}{K}$. Set constants $C_4, C_W'$ based on \eqref{eq:aux-constants}.
\For{$t=1, \dots, n$}{
    $a_t = \min\{(1+\frac{c_1}{8C_4})^{t-1}\frac{1}{4C_4},\max\{(1+\frac{\sqrt{c_1\nu}}{4\sqrt{2}C'_W})^{t-1}\frac{\sqrt{\nu_0}}{4C_W'},\frac{c_1\nu_0}{(4\sqrt{2}C'_W)^2}t\}\}, A_t=A_{t-1}+a_t$\; \label{line:choice-of-at} 
    $\vv(\vw;\vx,y)=2\beta(\sigma(\vw\cdot\vx)-y)\,\vx$ \;\label{line:surrogate_gradient}
    $\bar{\vlambda}_{t-1} := \evlambda_{t-1} + \frac{a_{t-1}}{a_t}(\evlambda_{t-1}-\evlambda_{t-2})$ \label{line:extrapolation} \;
    $\vw_{t} := \argmin_{\vw} \big\{ a_t \sum_{i=1}^K \bar{\lambda}_{t-1[i]} \E_{\epi}[\langle\vv(\vw_{t-1};\vx,y), \vw\rangle] + \frac{1+0.5c_1A_t}{2}\|\vw-\vw_{t-1}\|_2^2 \big\}$ \label{line:w}\;
    \( \evlambda_{t} := \argmax_{\evlambda\in\Delta_{K}} \{ a_t L(\vw_t, \evlambda) - (\nu_0+\nu A_{t-1})D_{\phi}(\evlambda, \evlambda_{t-1}) \} \) \label{line:lambda_update}
}
\KwOut{The final weight vector $\vw_n$.}
\end{algorithm}

Our method, stated in \Cref{alg:main}, is an iterative primal-dual algorithm that maintains iterates $(\vw_t, \evlambda_t)$ and step sizes $a_t > 0$ with cumulative sums $A_t = \sum_{k=1}^t a_k$. Both updates are efficient, as they involve projections onto simple sets like the Euclidean ball $\cB(W)$ and the probability simplex $\Delta_K,$ and/or minimization of KL divergence over the probability simplex, which is computable in closed form.  
A key algorithmic feature is the use of \emph{extrapolation on the low-dimensional dual variable} $\evlambda \in \R^K$ to construct a gradient estimate for the primal update. 

Before moving onto overviewing our convergence analysis, we state our main result.

\subsection{Main Result}

Our main result is summarized in the following theorem. We first define problem parameters (constants) that factor into the step size definition (consequently, iteration complexity) and the constant factor approximation:
\begin{equation}\label{eq:aux-constants}
    \begin{aligned}
        C_{3} &:=31\beta\sqrt{B}/c_{1},\\
        C_{4} &:=27c_1+2163\beta^4B^2/c_1,\\
        C'_W &:=2\sqrt{3}\sqrt{6\beta^2 + C_M^2B\log^2\Big(\frac{\beta BW}{\eps}\Big)} \beta WB.
    \end{aligned}
\end{equation}
In \eqref{eq:aux-constants}, $C_3$ and $C_4$ can be treated as universal constants for typical single neuron learning problems. This is because, as argued in \citep{wang2023robustly-learning}, the constant of sharpness $c_1$ is a constant for any non-degenerate problem in the considered class. Similarly, for typical activations like ReLU, $\beta$ is a constant (specifically, for the case of ReLU, $\alpha = \beta = 1$). The parameter $B$ comes from one-dimensional concentration of projections of data vectors $\vx$ (see \Cref{assump:concentration}); for all standard examples (like Gaussians, isotropic log-concaves, discrete Gaussians, uniform distribution on $\{-1, 0, 1\}^d$ discussed in \citep{wang2023robustly-learning}), $B$ is a small universal constant. The ``constant'' $C_{W}'$  however cannot be treated as a universal constant as it depends on $W.$ Fortunately, $C_{W}'$ has no effect on the approximation ratio in the statement of \Cref{thm:main-formal}---instead, it only affects the iteration complexity, which is polynomial in all problem parameters, as claimed in the introduction.  

\begin{restatable}[Main Theorem]{theorem}{mainthmformal}\label{thm:main-formal}
Suppose the margin and concentration conditions (\Cref{assump:margin,assump:concentration}) hold. Let $c_1$ be the sharpness constant from \Cref{lemma:sharpness-empirical}, and set $\nu_0 := \epsilon/(4K)$. If the total number of samples is $N =\tilde O_{\beta,B,\nu}\bigl(\frac{KW^{4}d}{\epsilon^{2}}\log(\frac{1}{\delta})\bigr)$,
then with probability at least $1-\delta$, the output $(\vw_n,\evlambda_n)$ of \Cref{alg:main} satisfies the following inequality for all $n \ge 1$:
\begin{align}\label{eq:main-proof-inequality}
&\frac{1+0.5c_1A_n}{4}\|\vw_n-\vw_*\|_2^2 + (\nu_0+\nu A_n)D_{\phi}(\evlambda_n, \evlambda^*) \notag \\
\le\;& D_0 + \frac{120\beta^2B}{c_1} A_n (\mOPT + \epsilon),
\end{align}
where  $D_0 = \frac{1}{2}\|\vw_0 - \vw_*\|_{2}^{2} + \nu_0 D_{\phi}(\evlambda^*, \evlambda_0)$.

Furthermore, after a number of iterations $n$ scaling as
\begin{equation}
\begin{split}
n = O\Big( C'_W \min\{ \,
 \tfrac{1}{\sqrt{c_1 \nu}}\log\!\big(
   \tfrac{\|\vw_*\|_2^2}{2\eps}
   + \tfrac{D_{\phi}(\evlambda^*,\evlambda_0)}{4K}
 \big), \\
 \sqrt{
   \tfrac{2K\|\vw_*\|_2^2}{\eps^2}
   + \tfrac{D_{\phi}(\evlambda^*,\evlambda_0)}{\eps}
 } \,\} \Big)
\end{split}
\end{equation}
the output $\vw_n$ is guaranteed to satisfy
\begin{align}
   \|\vw_{n}-\vw_*\|_{2} \le\; & C_{3}\bigl(\sqrt{\mOPT}+\sqrt{\epsilon}\bigr),\label{eq:distance-to-opt-bound}\\
   \E_{(\vx,y)\sim \pp^{*}}\bigl[\ell(\vw_{n};\vx,y)\bigr] 
 \le \; &
\bigl(2+20B\beta^{2}C_{3}^{2}\bigr)\,\mOPT\notag\\
&+ 20\beta^{2} C_3^2B\epsilon,
\label{eq:square-loss-bound}
\end{align}
where $\pp^* = \sum_{i=1}^K \lambdai^* \pp_{[i]}$ is the worst-case population mixture and relevant constants are defined by \eqref{eq:aux-constants}. 
\end{restatable}

We remark here that we have made no attempt to optimize the constant factors in either the final approximation ratio or sample and iteration complexities---our focus was on establishing (any) constant factor approximation with a polynomial sample and computation algorithm, which already required highly technical arguments. We expect the absolute constant factors to be improvable. 

\subsection{Overview of the Analysis}

The analysis centers on bounding the \emph{primal-dual gap}-like function  \(\Gap(\vw_t,\evlambda_t)\) defined with respect to the hybrid reference point $(\vw_*, \evlambda^*)$ with population-optimal $\vw_*$ and worst-case $\evlambda^*$ via
\begin{align*}
\Gap(\vw_t,\evlambda_t) =\;& L(\vw_t,\evlambda^*) - L(\vw_*,\evlambda_t). 
\end{align*}
Observe that $\Gap(\vw_t,\evlambda_t)$ is the sum of the ``primal gap'' $L(\vw_t,\evlambda^*) - L(\vw_*,\evlambda^*)$ and the ``dual gap'' $L(\vw_*,\evlambda^*) - L(\vw_*,\evlambda_t)$. 
Because the loss is nonconvex in $\vw$, the primal gap, and thus $\Gap(\vw_t,\evlambda_t)$, is not necessarily nonnegative. Our strategy is to combine a lower bound on the gap derived from empirical sharpness with an upper bound derived from the analysis that motivates the algorithm's update rules. 

We note here that while tracking similar gap functions is common to the analysis of primal-dual methods (see, e.g., \citep{chambolle2011first,mehta2025min}), here the central challenges come from the loss nonconvexity, which makes bounding the relevant quantities highly nontrivial. In the rest of this section, we first state (in \Cref{lemma:gap-lower-bound}) a lower bound on the gap function, which is crucial to being able to establish the contraction of distance to target solutions. This bound is derived following a similar argument to \citep{Li2024shifts} and its proof is provided for completeness, in \Cref{app:gap-lower-bound}. 

The most challenging part of the analysis comes from bounding the gap function above, and, in particular, bounding below its dual gap component. The reason, as discussed in \Cref{subsec:technical_overview}, is that the loss function is nonconvex and we perform extrapolation of the dual updates, to ensure our algorithm is compatible with large-scale implementations.

\subsection{Gap Lower Bound} 

We begin the convergence analysis by establishing a lower bound on the gap function \(\Gap(\vw, \evlambda)\). 

\begin{restatable}[Gap Lower Bound]{lemma}{lemgaplb}\label{lemma:gap-lower-bound}
  Under \Cref{assump:margin,assump:concentration},  
if the per-group sample size $N/K$ is sufficiently large (see \Cref{lemma:sharpness_empirical_heavy_tailed}),  for all \(\vw \in \cB(3\norm{\vw_*}_{2})\) and all $\evlambda \in \Delta_K$, we have
  \(\Gap(\vw ,\evlambda) \ge -\frac{12\beta^2 {B}}{c_1}\mOPThat + \frac{c_1}{2} \|\vw - \vw_*\|_{2}^{2} + \nu D_{\phi}(\evlambda^{*}, \evlambda).\)
\end{restatable}

The application of empirical sharpness, which is crucial to establishing the gap lower bound in \Cref{lemma:gap-lower-bound}, is predicated on the iterates $\vw_t$ remaining in a neighborhood of a target parameter vector $\vw_*$. We establish this property by induction. The proof is provided in {\Cref{sec: restate-of-iteration-norm-ub}}. 

\begin{restatable}{lemma}{claimiterub}\label{claim:iteration-norm-ub}
    For all iterations $t\ge 0$ of \Cref{alg:main}, the iterates satisfy $\|\vw_t\|_2 \le 3\|\vw_*\|_2$.
\end{restatable}

\subsection{Gap Upper Bound}

Having obtained a lower bound on the gap function, we now derive a corresponding upper bound based on our algorithm updates. The core of the argument is to analyze the per-iteration progress of \Cref{alg:main} to construct a telescoping sum. 
The main technical result is the following proposition, with the full proof deferred to \Cref{app:gap-upper-bound}.

\begin{restatable}[Gap Upper Bound]{proposition}{lemgapub}\label{lemma:gap-upper-bound}
    Let the sequences $\{a_t\}$, $\{A_t\}$, $\{\vw_t\}$, and $\{\evlambda_t\}$ be generated by \Cref{alg:main}, following the convention $a_{-1}=a_0=A_{-1}=A_0=0$, $\vw_{-1}=\vw_0=\mathbf{0}$, and $\evlambda_{-1}=\evlambda_0=\tfrac{1}{K}\mathbf{1}$. Under \Cref{assump:margin,assump:concentration},  
if the per-group sample size $N/K$ is sufficiently large (see \Cref{lemma:sharpness_empirical_heavy_tailed}), for any $n \ge 1$, we have:
    \begin{align*}
    &\sum_{t=1}^n a_t \Gap(\vw_t,\evlambda_t) \\
    \le\; & \frac{1}{2}\|\vw_*-\vw_0\|_2^2 - \frac{1+0.5c_1A_n}{4}\|\vw_*-\vw_n\|_2^2 \\
    &\;+ \nu_0 D_{\phi}(\evlambda^{*}, \evlambda_0)-\frac{1+0.5c_1A_n}{4}\|\vw_n-\vw_{n-1}\|_2^2 \\
    &\;-(\nu_0+\nu A_n)D_{\phi}(\evlambda^{*}, \evlambda_{n}) + \frac{28\beta^2B}{c_1} A_n \mOPThat. \notag
    \end{align*}
\end{restatable}

The proof of this result is highly technical, requiring careful handling of several nontrivial error terms induced by the extrapolation of the dual update and by the nonconvexity of the loss. For this reason, it is deferred to \Cref{app:gap-upper-bound}. We however highlight a key structural result used for nontrivially bounding below the expected loss, in the following lemma. The proof is sketched in the main body for space constraints; the full proof can be found in the appendix. 

\begin{restatable}[Linearization]{lemma}{lemubaux}\label{lemma:main-auxiliary-main-body}
    For each group $i \in [K]$ and each iteration $t \in [n]$, the following bound holds:
    \begin{align*}
        \;&\E_{\epi}\bigl[2(\sigma(\vw_t\cdot \vx)-y)(\sigma(\vw_*\cdot \vx)-\sigma(\vw_t\cdot \vx))\bigr] \\
        \ge\;& \E_{\epi}\bigl[\langle\vv(\vw_t;\vx,y), \vw_*-\vw_t\rangle\bigr] - E_t,
    \end{align*}
    where the error term is $E_t := \frac{24\beta^2 B\mOPThat}{c_1}+\frac{c_1}{4}\|\vw_*-\vw_t\|_2^2$.
\end{restatable}
\begin{proof}[Proof Sketch] 
Define \(\cG = \{(\vx,y)\mid \sigma(\vw_t\cdot \vx)-y\ge0\}\). 
The proof begins by splitting the expectation on the LHS according to whether $\1_{\cG} = 1$ or $\1_{\cG} =0$ (i.e., $\1_{\cG^c} = 1$).
We then bound below each term relying upon  the convexity of \(\sigma(\vw_t\cdot \vx)\) and \(\sigma(\vw_*\cdot \vx)\), respectively.
For instance, when we apply convexity of $\sigma$ at $\vw_t\cdot\vx$, we get:\footnote{Although $\sigma$ may not be differentiable, any subderivative---guaranteed to exist by the assumptions on $\sigma$---suffices here.}
	\[
	\sigma(\vw_*\!\cdot\vx)-\sigma(\vw_t\!\cdot\vx)
	\;\ge\;
	\sigma'(\vw_t\!\cdot\vx)\,(\vw_*\!\cdot\vx-\vw_t\!\cdot\vx).
	\]
A similar inequality is used in the other case.
This leaves us with a lower bound of 
\begin{align*}
&\E_{\epi}\bigl[2(\sigma(\vw_t\!\cdot\vx)-y)\,\sigma'(\vw_t\!\cdot\vx)\,(\vw_*\!\cdot\vx-\vw_t\!\cdot\vx)\1_{\cG}\bigr]\\
		+
		&\E_{\epi}\bigl[2(\sigma(\vw_t\!\cdot\vx)-y)\,\sigma'(\vw_*\!\cdot\vx)\,(\vw_*\!\cdot\vx-\vw_t\!\cdot\vx)\1_{\cG^c}\bigr]. 
\end{align*}
By adding and subtracting \(\beta\) to the $\sigma'$ factors and recalling that $\vv(\vw_t;\vx,y)=2\beta(\sigma(\vw_t\!\cdot\vx)-y)\vx$, 
we can isolate the desired term $\E_{\epi}\bigl[\langle\vv(\vw_t;\vx,y), \vw_*-\vw_t\rangle\bigr]$. 
What remains to be done is to bound the two remaining error terms involving \(\sigma'(\cdot)-\beta\) factors. We argue that the error terms can be bounded by \(O(\beta)\E_{\ppi}[|\sigma(\vw_*\cdot\vx)-y||\vw_*\cdot\vx-\vw_t\cdot\vx|]\). Then, we further bound this error contribution by an application of Cauchy-Schwarz inequality, the second moment bound \Cref{eq:moment-bounds-empirical}, and the definition of $\mOPThat.$
\end{proof}

\subsection{Proof of Main Theorem}
\label{subsec:proof-main-theorem}

\begin{proof}[Proof of \Cref{thm:main-formal}]
The proof proceeds by combining the lower and upper bounds on the cumulative gap. Summing the per-iteration lower bound from \Cref{lemma:gap-lower-bound} from $t=1$ to $n$ and combining it with the upper bound from \Cref{lemma:gap-upper-bound} gives the following inequality:
\begin{align*}
 \; &\frac{1+0.5\,c_1A_n}{4}\|\vw_*-\vw_n\|_2^2
+\frac{c_1}{2}\sum_{t=1}^na_t\|\vw_t-\vw_*\|_2^2\\
&+ \nu\sum_{t=1}^n a_t\,D_{\phi}(\evlambda_t,\evlambda^*)
+ (\nu_0+\nu A_n)\,D_{\phi}(\evlambda^*,\evlambda_n)\\
&+\frac{1+0.5c_1A_n}{4}\|\vw_n-\vw_{n-1}\|_2^2 \\
\le &\; 
\tfrac12\|\vw_*-\vw_0\|_2^2+\nu_0\,D_{\phi}(\evlambda^*,\evlambda_0)
+ \frac{40\,\beta^2B}{c_1}\,\mOPThat\,A_n.   
\end{align*}
Dropping the nonnegative terms that involve $\|\vw_t-\vw_*\|_2^2$, $D_{\phi}(\evlambda_t,\evlambda^*)$, and \(\|\vw_n-\vw_{n-1}\|_2^2\) from LHS of the inequality above and combining with the high-probability bound $\mOPThat \le 3(\mOPT+\epsilon)$ we establish in \Cref{lemma:opt_opthat_heavy_tailed}, we arrive at the first result of the main theorem in \Cref{eq:main-proof-inequality}. 

Next, we derive the explicit convergence guarantees. From \eqref{eq:main-proof-inequality}, we can isolate the primal error:
\begin{align*}
\|\vw_n-\vw_*\|_2^2 \le \; & \frac{4D_0}{1+0.5c_1A_n} + \frac{160\beta^2B A_n}{c_1(1+0.5c_1A_n)}\mOPThat \\
\le \; & \frac{4D_0}{1+0.5c_1A_n} + \frac{320\beta^2B}{c_1^2}\mOPThat.
\end{align*}
The step-size schedule in \Cref{alg:main} is chosen so that for the number of iterations $n$ specified in the theorem, we have $A_n = \Omega(c_1(W^2+K)/\epsilon) = \Omega(c_1 D_0 / \epsilon)$. This makes the first term $O(\epsilon)$. Substituting the bound on $\mOPThat$ and taking the square root establishes the distance guarantee \eqref{eq:distance-to-opt-bound}.
Finally, to bound the risk \eqref{eq:square-loss-bound}, we use the decomposition $\ell(\vw_n;\vx,y) \le 2\ell(\vw_*;\vx,y) + 2(\sigma(\vw_n\cdot\vx)-\sigma(\vw_*\cdot\vx))^2$, which follows by an application of Young's inequality. Taking the expectation over the worst-case population distribution $\pp^*$ then leads to
\begin{align*}
&\; \E_{\pp^*}[\ell(\vw_n;\vx,y)] \\
\le&\;  2\E_{\pp^*}[\ell(\vw_*;\vx,y)] + 2\E_{\pp^*}[(\sigma(\vw_n\cdot\vx)-\sigma(\vw_*\cdot\vx))^2] \\
\le&\; 2\mOPT \\ 
&+ 2\beta^2 \|\vw_n-\vw_*\|_2^2 \cdot \max_{i \in [K]} \E_{\ppi}\Big[\Big(\frac{\vw_n-\vw_*}{\|\vw_n-\vw_*\|_2} \cdot \vx\Big)^2\Big] \\
\le&\; 2\mOPT + 2\beta^2\cdot 2C_3^2(\mOPT+\epsilon) \cdot 5B,
\end{align*}
where the last line uses the previously established bound \eqref{eq:distance-to-opt-bound} on $\|\vw_n-\vw_*\|_2^2$ and the moment bound from \Cref{fact:sharpness}. Rearranging terms yields the final result \eqref{eq:square-loss-bound}.
\end{proof}

\section{Experiments}\label{sec:experiments} 

While our main contributions are theoretical, we provide illustrative experiments to demonstrate that our core algorithmic idea—regularized dual-extrapolated reweighting—can be applied to large-scale language model pretraining. These results are not intended as comprehensive empirical validation, but rather serve as a preliminary exploration of how our theoretically-motivated update rule performs in practice.

\paragraph{Setup}
We isolate the impact of our reweighting algorithm by integrating our primal-dual method with KL-divergence regularization (\textbf{PD-KL}) directly into the Sheared LLaMA framework \citep{xia2024shearedllamaacceleratinglanguage}. Starting from the \texttt{Sheared-LLaMA-1.3B} model—a 1.3B parameter version pruned from LLaMA2-7B \citep{xia2024shearedllamaacceleratinglanguage}—we continue pre-training on RedPajama \citep{together2023redpajama}. Our baseline is the dynamic batch loading algorithm from \citep{xia2024shearedllamaacceleratinglanguage}, where we replace only its exponential ascent rule with our dual update (\Cref{alg:main}) that incorporates KL regularization and extrapolation. All other training parameters remain identical to ensure fair comparison. More training details are included in \Cref{sec:supply-experiments-details}.

Performance is evaluated on the same suite of 11 downstream tasks used in \citep[Table 2]{xia2024shearedllamaacceleratinglanguage}, under their prescribed few-shot settings. We report unweighted average accuracy across tasks at multiple checkpoints. By design, this setup attributes any performance differences directly to the choice of on-the-fly domain reweighting strategy. Results are presented from 16.8M to 92.4M tokens. 

\subsection{Main Results}

We compare our PD-KL algorithm against DoReMi \citep{xie2023doremi}---which is the training method used on the Sheared LLaMA model in \citep{xia2024shearedllamaacceleratinglanguage}.
As shown in \Cref{fig:compute-curve}, our PD-KL algorithm demonstrates a consistent performance improvement over the baseline dynamic batch loading across the training trajectory. At most checkpoints, PD-KL achieves a higher average downstream accuracy varying from 0.04\% to 0.96\%. Moreover, at 33.6M tokens our PD-KL model reaches 47.87\% average accuracy. DoReMi requires at least 1.5x more training time to achieve the same accuracy. This suggests that the principles of regularization and dual-side extrapolation, motivated by our theory, can offer practical benefits for training stability and downstream generalization even in highly complex models.

\begin{figure}[t]
  \centering
  \includegraphics[width=0.7\linewidth]{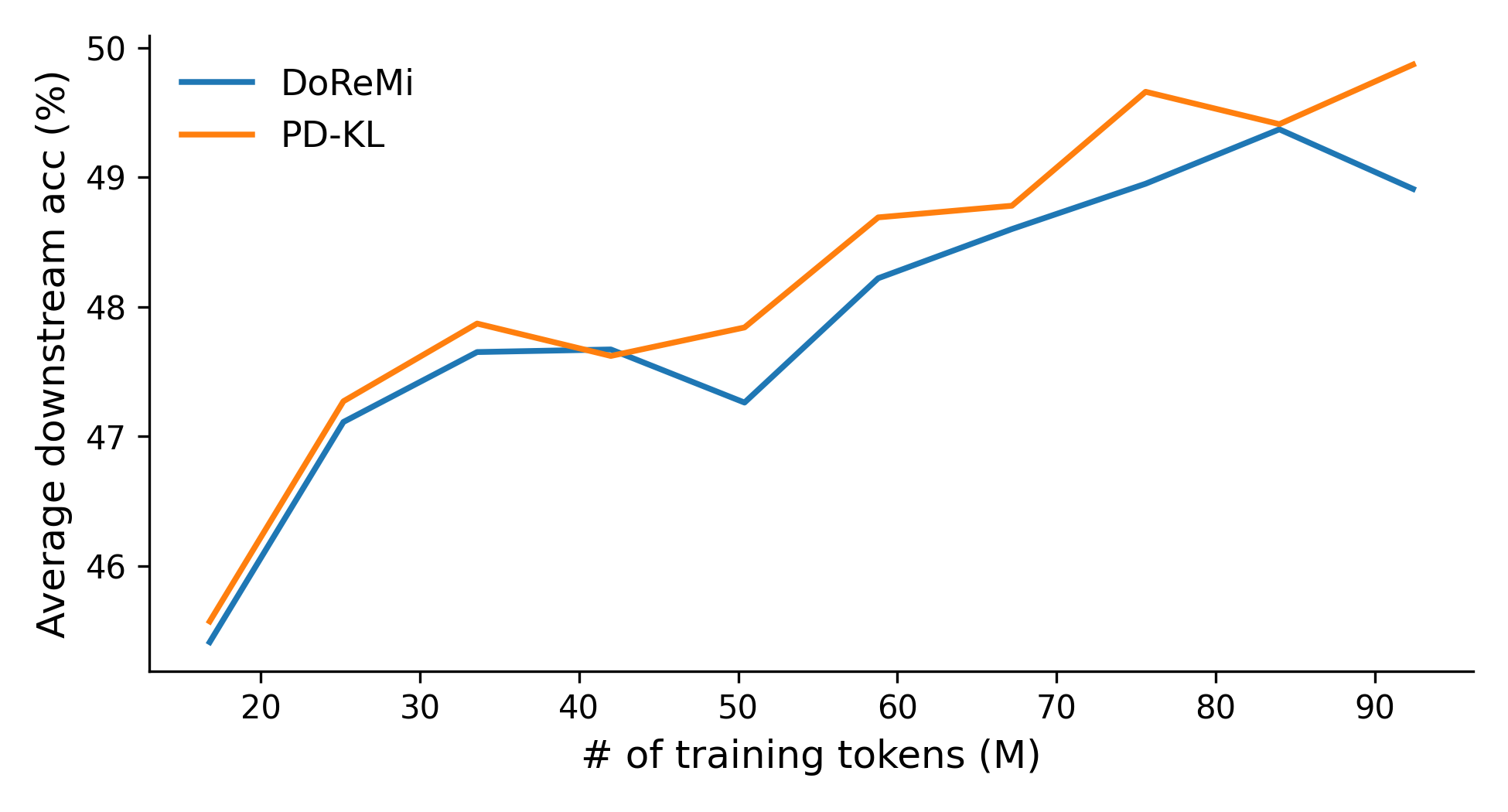}
  \caption{Compute-performance curve on \emph{Sheared-LLaMA-1.3B}. Y-axis is the unweighted overall accuracy scores, X-axis is the number of tokens trained.}
  \label{fig:compute-curve}
\end{figure}

\begin{table}[ht!]
  \centering
  \caption{Per-task results (\%) at 92.4M tokens.}
  \label{tab:task-breakdown}
  \sisetup{
    table-format=2.2,
    table-number-alignment=center
  }
  \resizebox{.6\textwidth}{!}{
  \begin{tabular}{l l l S S}
    \toprule
    Bucket & Task & Metric & \multicolumn{1}{c}{DoReMi} & \multicolumn{1}{c}{PD-KL} \\
    \midrule
    \textbf{Commonsense \& RC} & \textbf{ARC-E} & \texttt{acc\_norm}      & \textbf{50.00} & 49.83 \\
    \textbf{Commonsense \& RC} & \textbf{ARC-C(25)} & \texttt{acc\_norm}   & 29.95 & \textbf{30.38} \\
    \textbf{Commonsense \& RC} & \textbf{HellaSwag(10)} & \texttt{acc\_norm}& \textbf{54.78} & 54.62 \\
    \textbf{Commonsense \& RC} & \textbf{PICA} & \texttt{acc}             & 70.78 & \textbf{71.87} \\
    \textbf{Commonsense \& RC} & \textbf{SciQ} & \texttt{acc}             & 85.00 & \textbf{85.90} \\
    \textbf{Commonsense \& RC} & \textbf{WinoGrande} & \texttt{acc}         & 54.22 & \textbf{55.25} \\
    \textbf{Commonsense \& RC} & \textbf{WSC} & \texttt{acc}              & 36.54 & 36.54 \\
    \midrule
    \textbf{Continued \& LM}  & \textbf{BoolQ(32)} & \texttt{acc}          & 56.39 & \textbf{63.64} \\
    \textbf{Continued \& LM}  & \textbf{LogicQA} & \texttt{acc\_norm}     & \textbf{28.11} & 27.80 \\
    \textbf{Continued \& LM}  & \textbf{LAMBADA} & \texttt{acc}           & 48.61 & \textbf{50.63} \\
    \midrule
    \textbf{World Knowledge}  & \textbf{TruthfulQA(5)} & \texttt{acc}       & \textbf{23.62} & 22.15 \\
    \midrule
    \multicolumn{3}{c}{\textbf{Unweighted Mean}} 
      & \multicolumn{1}{c}{48.91} & \multicolumn{1}{c}{\textbf{49.87}} \\
    \bottomrule
  \end{tabular}
  }
\end{table}

\paragraph{Per-task breakdown.}
For transparency, \Cref{tab:task-breakdown} lists every task used in evaluation, its metric, and per-task scores for both models. To remain faithful to the reference and avoid choice-induced bias, we keep the same per-task metric (e.g., \texttt{acc} vs.\ \texttt{acc\_norm}) and the same shot setting as in \citep{xia2024shearedllamaacceleratinglanguage}, with one exception: following community conventions in large language modeling, we evaluate \mbox{\textbf{ARC-E}} using \texttt{acc\_norm} rather than \texttt{acc}.

\section{Conclusion}\label{sec:conclusion}
We studied the problem of robustly learning a single neuron in the group distributionally robust settings, where labels can be arbitrary and the goal is to be competitive with the ``best-fit'' model, as measured by the mean squared loss on the worst case group reweighting.  
The dual updates of our primal-dual approach can be seen as a novel group/domain reweighting. We hope that our work will encourage more research in this area. Possible future work includes extending either primal-dual updates or only dual updates to popular applications of language model pretraining and comparing whether it is competitive with current state-of-the-art domain reweighting algorithms. On the more theoretical side, it would  be interesting to strengthen the error guarantee to a constant factor error in terms of $\OPT$ rather than $\mOPT$ (for $\nu \gg 0$, since $\OPT = \mOPT$ for $\nu = 0$) and to generalize the technical approach to other related settings involving nonconvex risk minimization.

\newpage
\bibliographystyle{alpha}  
\bibliography{ref}

\newcommand{\etalchar}[1]{$^{#1}$}
\begin{thebibliography}{BDBC{\etalchar{+}}10}

\bibitem[AZ22]{agarwal2022minimax}
Alekh Agarwal and Tong Zhang.
\newblock Minimax regret optimization for robust machine learning under distribution shift.
\newblock In {\em Conference on Learning Theory}. PMLR, 2022.

\bibitem[BBS07]{bickel2007discriminative}
Steffen Bickel, Michael Br{\"u}ckner, and Tobias Scheffer.
\newblock Discriminative learning for differing training and test distributions.
\newblock In {\em Proceedings of the 24th international conference on Machine learning}, 2007.

\bibitem[BDBC{\etalchar{+}}10]{ben2010theory}
Shai Ben-David, John Blitzer, Koby Crammer, Alex Kulesza, Fernando Pereira, and Jennifer~Wortman Vaughan.
\newblock A theory of learning from different domains.
\newblock {\em Machine learning}, 79, 2010.

\bibitem[BMN21]{blanchet2021statistical}
Jose Blanchet, Karthyek Murthy, and Viet~Anh Nguyen.
\newblock Statistical analysis of wasserstein distributionally robust estimators.
\newblock In {\em Tutorials in Operations Research: Emerging Optimization Methods and Modeling Techniques with Applications}. INFORMS, 2021.

\bibitem[BTEGN09]{ben2009robust}
Aharon Ben-Tal, Laurent El~Ghaoui, and Arkadi Nemirovski.
\newblock {\em Robust optimization}, volume~28.
\newblock Princeton university press, 2009.

\bibitem[CKMY20]{chen2020classification}
Sitan Chen, Frederic Koehler, Ankur Moitra, and Morris Yau.
\newblock Classification under misspecification: Halfspaces, generalized linear models, and evolvability.
\newblock {\em Advances in Neural Information Processing Systems}, 33, 2020.

\bibitem[CP11]{chambolle2011first}
Antonin Chambolle and Thomas Pock.
\newblock A first-order primal-dual algorithm for convex problems with applications to imaging.
\newblock {\em Journal of mathematical imaging and vision}, 40, 2011.

\bibitem[CP18]{chen2018robust}
Ruidi Chen and Ioannis~Ch Paschalidis.
\newblock A robust learning approach for regression models based on distributionally robust optimization.
\newblock {\em Journal of Machine Learning Research}, 19(13), 2018.

\bibitem[DGK{\etalchar{+}}20]{DGKKS20}
Ilias Diakonikolas, Surbhi Goel, Sushrut Karmalkar, Adam~Russell Klivans, and Mahdi Soltanolkotabi.
\newblock Approximation schemes for relu regression.
\newblock In {\em Conference on Learning Theory, {COLT} 2020}, 2020.

\bibitem[DGT19]{diakonikolas2019distribution}
Ilias Diakonikolas, Themis Gouleakis, and Christos Tzamos.
\newblock Distribution-independent pac learning of halfspaces with massart noise.
\newblock {\em Advances in Neural Information Processing Systems}, 32, 2019.

\bibitem[DHP{\etalchar{+}}12]{dwork2012fairness}
Cynthia Dwork, Moritz Hardt, Toniann Pitassi, Omer Reingold, and Richard Zemel.
\newblock Fairness through awareness.
\newblock In {\em Proceedings of the 3rd Innovations in Theoretical Computer Science Conference}, 2012.

\bibitem[DKMR22]{diakonikolas22a-hardness}
Ilias Diakonikolas, Daniel Kane, Pasin Manurangsi, and Lisheng Ren.
\newblock Hardness of learning a single neuron with adversarial label noise.
\newblock In {\em Proceedings of The 25th International Conference on Artificial Intelligence and Statistics}, 2022.

\bibitem[DKPZ21]{DKPZ21-SQ}
Ilias Diakonikolas, Daniel~M Kane, Thanasis Pittas, and Nikos Zarifis.
\newblock The optimality of polynomial regression for agnostic learning under gaussian marginals in the sq model.
\newblock In {\em Conference on Learning Theory}, pages 1552--1584. PMLR, 2021.

\bibitem[DKR23]{diakonikolas2023near}
Ilias Diakonikolas, Daniel Kane, and Lisheng Ren.
\newblock Near-optimal cryptographic hardness of agnostically learning halfspaces and relu regression under gaussian marginals.
\newblock In {\em International Conference on Machine Learning}. PMLR, 2023.

\bibitem[DKTZ22]{DKTZ22}
I.~Diakonikolas, V.~Kontonis, C.~Tzamos, and N.~Zarifis.
\newblock Learning a single neuron with adversarial label noise via gradient descent.
\newblock In {\em Conference on Learning Theory (COLT)}, pages 4313--4361, 2022.

\bibitem[DKZ20]{DKZ20-sq-reg}
Ilias Diakonikolas, Daniel Kane, and Nikos Zarifis.
\newblock Near-optimal {SQ} lower bounds for agnostically learning halfspaces and relus under gaussian marginals.
\newblock In {\em Annual Conference on Neural Information Processing Systems}, 2020.

\bibitem[DN21]{duchi2021learning}
John~Charles Duchi and Hongseok Namkoong.
\newblock Learning models with uniform performance via distributionally robust optimization.
\newblock {\em The Annals of Statistics}, 49(3), 2021.

\bibitem[DPT21]{diakonikolas2021relu}
Ilias Diakonikolas, Jongho Park, and Christos Tzamos.
\newblock Relu regression with massart noise.
\newblock {\em Advances in Neural Information Processing Systems}, 34, 2021.

\bibitem[GGK20]{GGK20}
Surbhi Goel, Aravind Gollakota, and Adam Klivans.
\newblock Statistical-query lower bounds via functional gradients.
\newblock In {\em Annual Conference on Neural Information Processing Systems}, 2020.

\bibitem[GKK19]{GoelKK19}
Surbhi Goel, Sushrut Karmalkar, and Adam~R Klivans.
\newblock Time/accuracy tradeoffs for learning a relu with respect to gaussian marginals.
\newblock In {\em Advances in Neural Information Processing Systems 32}, 2019.

\bibitem[Hau92]{Haussler:92}
David Haussler.
\newblock {Decision theoretic generalizations of the PAC model for neural net and other learning applications}.
\newblock {\em Information and Computation}, 100:78--150, 1992.

\bibitem[HGB{\etalchar{+}}06]{huang2006correcting}
Jiayuan Huang, Arthur Gretton, Karsten Borgwardt, Bernhard Sch{\"o}lkopf, and Alex Smola.
\newblock Correcting sample selection bias by unlabeled data.
\newblock {\em Advances in neural information processing systems}, 19, 2006.

\bibitem[HJZ22]{haghtalab2022ondemand}
Nika Haghtalab, Michael~I. Jordan, and Eric Zhao.
\newblock On-demand sampling: Learning optimally from multiple distributions.
\newblock In {\em Advances in Neural Information Processing Systems}, volume~35, 2022.

\bibitem[HSNL18]{hashimoto2018fairness}
Tatsunori Hashimoto, Megha Srivastava, Hongseok Namkoong, and Percy Liang.
\newblock Fairness without demographics in repeated loss minimization.
\newblock In {\em International Conference on Machine Learning}. PMLR, 2018.

\bibitem[KB15]{kingma2015adam}
Diederik~P Kingma and Jimmy Ba.
\newblock Adam: A method for stochastic optimization.
\newblock In {\em International Conference on Learning Representations (ICLR)}, 2015.

\bibitem[KKSK11]{kakade2011efficient}
Sham~M Kakade, Varun Kanade, Ohad Shamir, and Adam Kalai.
\newblock Efficient learning of generalized linear and single index models with isotonic regression.
\newblock {\em Advances in Neural Information Processing Systems}, 24, 2011.

\bibitem[KLL22]{kotsalis2022simple}
Georgios Kotsalis, Guanghui Lan, and Tianjiao Li.
\newblock Simple and optimal methods for stochastic variational inequalities, {I}: operator extrapolation.
\newblock {\em SIAM Journal on Optimization}, 32(3):2041--2073, 2022.

\bibitem[KMM21]{karmakar2020study}
Sayar Karmakar, Anirbit Mukherjee, and Ramchandran Muthukumar.
\newblock A study of neural training with iterative non-gradient methods.
\newblock {\em SSRN Electronic Journal}, 01 2021.

\bibitem[KS09]{kalai2009isotron}
Adam~Tauman Kalai and Rajeev Sastry.
\newblock The isotron algorithm: High-dimensional isotonic regression.
\newblock In {\em COLT}, 2009.

\bibitem[KSA19]{kalan2019fitting}
Seyed Morteza~Mousavi Kalan, Mahdi Soltanolkotabi, and Salman Avestimehr.
\newblock Fitting relus via sgd and quantized sgd.
\newblock In {\em 2019 IEEE International Symposium on Information Theory (ISIT)}. IEEE, 2019.

\bibitem[KSS92]{kearns1992toward}
Michael~John Kearns, Robert~Elias Schapire, and Linda~Marie Sellie.
\newblock Toward efficient agnostic learning.
\newblock In {\em Proceedings of the fifth annual workshop on Computational learning theory}, 1992.

\bibitem[LJJ20]{lin2020gradient}
Tianyi Lin, Chi Jin, and Michael Jordan.
\newblock On gradient descent ascent for nonconvex-concave minimax problems.
\newblock In {\em International conference on machine learning}, pages 6083--6093. PMLR, 2020.

\bibitem[LKDD24]{Li2024shifts}
Shuyao Li, Sushrut Karmalkar, Ilias Diakonikolas, and Jelena Diakonikolas.
\newblock Learning a single neuron robustly to distributional shifts and adversarial label noise.
\newblock In A.~Globerson, L.~Mackey, D.~Belgrave, A.~Fan, U.~Paquet, J.~Tomczak, and C.~Zhang, editors, {\em Advances in Neural Information Processing Systems}, volume~37, pages 67383--67421. Curran Associates, Inc., 2024.

\bibitem[LZS25]{li2025nonsmooth}
Jiajin Li, Linglingzhi Zhu, and Anthony Man-Cho So.
\newblock Nonsmooth nonconvex--nonconcave minimax optimization: Primal--dual balancing and iteration complexity analysis.
\newblock {\em Mathematical Programming}, pages 1--51, 2025.

\bibitem[MDH25]{mehta2025min}
Ronak Mehta, Jelena Diakonikolas, and Zaid Harchaoui.
\newblock Min-max optimization with dual-linear coupling.
\newblock {\em arXiv preprint arXiv:2507.06328}, 2025.

\bibitem[MMR09]{mansour2009domain}
Yishay Mansour, Mehryar Mohri, and Afshin Rostamizadeh.
\newblock Domain adaptation: Learning bounds and algorithms.
\newblock In {\em Proceedings of the 22nd Annual Conference on Learning Theory (COLT)}, Montréal, Canada, 2009.

\bibitem[MR18]{MR18}
Pasin Manurangsi and Daniel Reichman.
\newblock The computational complexity of training relu (s).
\newblock {\em arXiv preprint arXiv:1810.04207}, 2018.

\bibitem[NW72]{nelder1972generalized}
John~Ashworth Nelder and Robert William~Macdonald Wedderburn.
\newblock Generalized linear models.
\newblock {\em Journal of the Royal Statistical Society: Series A (General)}, 135(3), 1972.

\bibitem[OSHL19]{oren2019distributionally}
Yonatan Oren, Shiori Sagawa, Tatsunori~B. Hashimoto, and Percy Liang.
\newblock Distributionally robust language modeling.
\newblock In Kentaro Inui, Jing Jiang, Vincent Ng, and Xiaojun Wan, editors, {\em Proceedings of the 2019 Conference on Empirical Methods in Natural Language Processing and the 9th International Joint Conference on Natural Language Processing (EMNLP-IJCNLP)}, pages 4227--4237, Hong Kong, China, November 2019. Association for Computational Linguistics.

\bibitem[PGLC15]{patel2015visual}
Vishal~M Patel, Raghuraman Gopalan, Ruonan Li, and Rama Chellappa.
\newblock Visual domain adaptation: A survey of recent advances.
\newblock {\em IEEE signal processing magazine}, 32(3), 2015.

\bibitem[PY09]{pan2009survey}
Sinno~Jialin Pan and Qiang Yang.
\newblock A survey on transfer learning.
\newblock {\em IEEE Transactions on knowledge and data engineering}, 22(10), 2009.

\bibitem[QGX{\etalchar{+}}21]{qi2021online}
Qi~Qi, Zhishuai Guo, Yi~Xu, Rong Jin, and Tianbao Yang.
\newblock An online method for a class of distributionally robust optimization with non-convex objectives.
\newblock {\em Advances in Neural Information Processing Systems}, 34:10067--10080, 2021.

\bibitem[RLLY22]{rafique2022weakly}
Hassan Rafique, Mingrui Liu, Qihang Lin, and Tianbao Yang.
\newblock Weakly-convex--concave min--max optimization: provable algorithms and applications in machine learning.
\newblock {\em Optimization Methods and Software}, 37(3):1087--1121, 2022.

\bibitem[Ros58]{rosenblatt1958perceptron}
Frank Rosenblatt.
\newblock The perceptron: a probabilistic model for information storage and organization in the brain.
\newblock {\em Psychological Review}, 65(6):386--408, 1958.

\bibitem[SGJ22]{soma2022optimal}
Tasuku Soma, Khashayar Gatmiry, and Stefanie Jegelka.
\newblock {Optimal algorithms for group distributionally robust optimization and beyond}, 2022.

\bibitem[Sha17]{shapiro2017distributionally}
Alexander Shapiro.
\newblock Distributionally robust stochastic programming.
\newblock {\em SIAM Journal on Optimization}, 27(4), 2017.

\bibitem[Shi00]{shimodaira2000improving}
Hidetoshi Shimodaira.
\newblock Improving predictive inference under covariate shift by weighting the log-likelihood function.
\newblock {\em Journal of statistical planning and inference}, 90(2), 2000.

\bibitem[Sim02]{vsima2002training}
Jiri Sima.
\newblock Training a single sigmoidal neuron is hard.
\newblock {\em Neural computation}, 14(11), 2002.

\bibitem[SKHL20]{sagawa2020groupDRO}
Shiori Sagawa, Pang~Wei Koh, Tatsunori~B. Hashimoto, and Percy Liang.
\newblock Distributionally robust neural networks for group shifts: On the importance of regularization for worst-case generalization.
\newblock {\em International Conference on Learning Representations (ICLR)}, 2020.

\bibitem[SND18]{sinha2018certifying}
Aman Sinha, Hongseok Namkoong, and John Duchi.
\newblock Certifying some distributional robustness with principled adversarial training.
\newblock In {\em International Conference on Learning Representations}, 2018.

\bibitem[Sol17]{soltanolkotabi2017learning}
Mahdi Soltanolkotabi.
\newblock Learning relus via gradient descent.
\newblock {\em Advances in neural information processing systems}, 30, 2017.

\bibitem[TC23]{together2023redpajama}
{}~Together~Computer.
\newblock Redpajama: An open source recipe to reproduce llama training dataset, 4 2023.

\bibitem[TSK{\etalchar{+}}18]{tan2018survey}
Chuanqi Tan, Fuchun Sun, Tao Kong, Wenchang Zhang, Chao Yang, and Chunfang Liu.
\newblock A survey on deep transfer learning.
\newblock In {\em Artificial Neural Networks and Machine Learning--ICANN 2018}. Springer, 2018.

\bibitem[WAL24]{wang2024fenchel}
Jun-Kun Wang, Jacob Abernethy, and Kfir~Y. Levy.
\newblock No-regret dynamics in the fenchel game: A unified framework for algorithmic convex optimization.
\newblock {\em Mathematical Programming}, 205(1-2):203--268, 2024.

\bibitem[WZDD23]{wang2023robustly-learning}
Puqian Wang, Nikos Zarifis, Ilias Diakonikolas, and Jelena Diakonikolas.
\newblock Robustly learning a single neuron via sharpness.
\newblock In {\em Proceedings of the 40th International Conference on Machine Learning}, volume 202 of {\em Proceedings of Machine Learning Research}, pages 36541--36577. PMLR, 7 2023.

\bibitem[XDKR20]{xu2020class}
Ziyu Xu, Chen Dan, Justin Khim, and Pradeep Ravikumar.
\newblock Class-weighted classification: Trade-offs and robust approaches.
\newblock In {\em International conference on machine learning}. PMLR, 2020.

\bibitem[XGZC]{shearedllama_code}
Mengzhou Xia, Tianyu Gao, Zhiyuan Zeng, and Danqi Chen.
\newblock Sheared llama: Accelerating language model pre-training.
\newblock \url{https://github.com/princeton-nlp/LLM-Shearing/blob/main/llmshearing/callbacks/dynamic_loading_callback.py}.

\bibitem[XGZC24]{xia2024shearedllamaacceleratinglanguage}
Mengzhou Xia, Tianyu Gao, Zhiyuan Zeng, and Danqi Chen.
\newblock Sheared llama: Accelerating language model pre-training via structured pruning.
\newblock In {\em International Conference on Learning Representations (ICLR), 2024}, 2024.

\bibitem[XPD{\etalchar{+}}23]{xie2023doremi}
Sang~Michael Xie, Hieu Pham, Xuanyi Dong, Nan Du, Hanxiao Liu, Yifeng Lu, Percy~S Liang, Quoc~V Le, Tengyu Ma, and Adams~Wei Yu.
\newblock Doremi: Optimizing data mixtures speeds up language model pretraining.
\newblock {\em Advances in Neural Information Processing Systems}, 36:69798--69818, 2023.

\bibitem[YS20]{yehudai2020learning}
Gilad Yehudai and Ohad Shamir.
\newblock Learning a single neuron with gradient methods.
\newblock In {\em Conference on Learning Theory}, pages 3756--3786, 2020.

\bibitem[ZXSL20]{zhang2020single}
Jiawei Zhang, Peijun Xiao, Ruoyu Sun, and Zhiquan Luo.
\newblock A single-loop smoothed gradient descent-ascent algorithm for nonconvex-concave min-max problems.
\newblock {\em Advances in neural information processing systems}, 33:7377--7389, 2020.

\bibitem[ZZZ{\etalchar{+}}23]{Zhang2023Stochastic}
Lijun Zhang, Peng Zhao, Zhen-Hua Zhuang, Tianbao Yang, and Zhi-Hua Zhou.
\newblock {Stochastic Approximation Approaches to Group Distributionally Robust Optimization}.
\newblock In {\em Advances in Neural Information Processing Systems (NeurIPS)}, 2023.

\end{thebibliography}

\clearpage
\appendix

\appendix
\crefalias{section}{appendix}
\crefalias{subsection}{appendix}

\onecolumn
\begin{center}{\LARGE\bfseries Supplementary Material}
\end{center}

\paragraph{Organization} 
In \Cref{sec:related_work}, we review additional related work.  
In \Cref{sec:suppl_prelims}, we collect the definitions and standard results in probability and optimization that we use throughout the paper. 
In \Cref{subsec:uniform-conv}, we prove that, with high probability, all the empirical expectations we need closely track their population counterparts.  
Finally, in~\Cref{app:gap-upper-bound}, we provide the full proof of the gap upper bound,~\Cref{lemma:gap-upper-bound}.  

\section{Additional Related Work}\label{sec:related_work}

\paragraph{Robustly Learning Neurons}
Generalized linear models (a.k.a.\ ``single neurons'') have long been a cornerstone of statistics and machine learning~\citep{nelder1972generalized}, and early algorithms such as the Isotron~\citep{kalai2009isotron} and Kakade-Shamir's efficient GLM learner~\citep{kakade2011efficient} provided the first guarantees for learning a single neuron in the presence of mild, structured label noise.  More recent work has focused on ReLU activations under progressively stronger noise models: for realizable or random additive noise, both gradient-based and spectral methods can recover or approximate the true neuron~\citep{soltanolkotabi2017learning,kalan2019fitting,yehudai2020learning}; in the fully agnostic setting, achieving purely additive error \(\eps > 0\) is intractable even for Gaussian inputs~\citep{DKZ20-sq-reg,GGK20,diakonikolas22a-hardness}, though constant-factor approximations are possible under log-concave or other structured marginals~\citep{DGKKS20,DKTZ22,wang2023robustly-learning}.  Extensions to semi-random or Massart-type noise require further specialized techniques~\citep{diakonikolas2021relu,karmakar2020study,chen2020classification,diakonikolas2019distribution} beyond the scope of the present work.  

In this work, we consider the strictly harder regime where an adversary may not only corrupt labels within each of \(K\) groups, but also shift the covariate distribution across groups.  We formulate this as a Group DRO problem over \(\pp_{[1]},\dots,\pp_{[K]}\), seeking \(\vw\) that minimizes the worst-case weighted squared loss plus an \(f\)-divergence penalty on the group weights.  

\paragraph{Distributionally Robust Optimization}
Covariate shift and related distributional mismatches have been studied extensively in the past---in covariate-shift correction~\citep{shimodaira2000improving,huang2006correcting,bickel2007discriminative}, label-proportion changes~\citep{dwork2012fairness,xu2020class}, domain adaptation and transfer learning \citetext{\citealp{mansour2009domain,pan2009survey,ben2010theory} \citealp{patel2015visual,tan2018survey}}, and classical DRO~\citep{ben2009robust,shapiro2017distributionally}.

Recent machine-learning work has applied DRO to language modeling~\citep{oren2019distributionally}, class-imbalance \allowbreak\citep{xu2020class}, group fairness~\citep{hashimoto2018fairness}, and robust regression~\citep{blanchet2021statistical,duchi2021learning,chen2018robust}. However, these methods typically assume convex, Lipschitz losses and structured label noise, and they do not address nonconvex problems like learning a neuron with respect to the square loss. On the other hand, even the work that did consider nonconvex loss functions like \citep{qi2021online,sinha2018certifying} only considered general, non-structured smooth loss functions, for which only convergence to stationary points can be guaranteed. In addition to  not exploring the structured nonconvexity of common learning tasks, for concrete learning tasks such as those considered in our work, convergence to stationary points is known to be insufficient---even for the special case of learning a ReLU neuron without any distributional ambiguity, stationary points of the squared loss may not lead to any formal learning guarantees \citep{yehudai2020learning}.

Most closely related to our work in this context is \citep{Li2024shifts}, which  explored the task of robustly learning a single neuron in a related, but distinct distributionally robust optimization setting, where there is a single reference distribution and each sample could be reweighted. The main body of the paper has already provided a comparison to \citep{Li2024shifts} both in terms of the results and the techniques, thus we omit repeating it here. 

\paragraph{Group DRO}
As a subclass of problems in Distributionally Robust Optimization, a line of work isolates distributional shift that manifests through a finite set of semantic or demographic groups and asks the learner to minimize the worst-case loss across those groups. \citep{sagawa2020groupDRO} first demonstrated its empirical effectiveness on over-parameterized neural networks, showing dramatic accuracy gains for minority groups. Subsequent theoretical work established sharp sample-complexity bounds under convex, unregularized losses \citep{haghtalab2022ondemand,soma2022optimal,Zhang2023Stochastic}. There are also related explorations into minimax regret optimization~\citep{agarwal2022minimax} and no-regret dynamics~\citep{wang2024fenchel}. They provide algorithmic blueprints using iterative methods like online gradient descent and stochastic optimization techniques, with established convergence guarantees and complexity bounds.

In the context of large-scale language model pretraining, DoReMi \citep{xie2023doremi} has empirically demonstrated the effectiveness of dynamically adjusting mixture weights across domains. 
Building on this line of work, LLM-Shearing \citep{xia2024shearedllamaacceleratinglanguage} further explored structured pruning as a means to accelerate pretraining while maintaining competitive performance of DoReMi.

Our setting addresses an even more aggressive setting, involving two robustness notions---agnostic label noise and adversarial covariate shifts across \(K\) groups—by penalizing deviations from uniform group weights via an \(f\)-divergence in a Group DRO formulation.  This allows us to obtain the first provable, polynomial-time algorithm for learning a nonconvex neuron model under both sources of adversarial perturbation. 

\section{Supplementary Preliminaries}
\label{sec:suppl_prelims}
{In the appendix, we recall that $\phi(\evlambda)$ generically to represent a divergence between $\evlambda$ and the uniform distribution, instantiated as either $\mathrm{KL}(\evlambda, \evlambda_0)$ or $\chi^2(\evlambda, \evlambda_0)$, where $\evlambda_0 = \tfrac{1}{K}\mathbf{1}$.} 

We state here several classical inequalities, and useful definitions and facts. 
\begin{fact}[Young's inequality]\label{fact:young}
	If $a \geq 0$ and $b \geq 0$ are {nonnegative real numbers} and if $p > 1$ and $q > 1$ are real numbers such that 
	\(
	\frac{1}{p} + \frac{1}{q} = 1,
	\)
	then
	\(
	ab \leq \frac{a^p}{p} + \frac{b^q}{q}.
	\) 
	Equality holds if and only if $a^p = b^q$.
\end{fact}

\begin{fact}[Hoeffding's Inequality]
	\label{fact:hoeffding}
	Let \(X_1, X_2, \ldots, X_n\) be independent random variables such that \(a_i \le X_i \le b_i\) almost surely for all \(i\). Let \(\overline{X} = \frac{1}{n} \sum_{i=1}^{n} X_i\). Then, for any \(t > 0\),
\[
	\Pr\left[\abs{\overline{X} - \E[\overline{X}]} \ge t\right] \le 2 \exp\left(-\frac{2n^2 t^2}{\sum_{i=1}^{n} (b_i - a_i)^2}\right).
	\]
\end{fact}

\begin{definition}[Total Variation Distance]
\label{def:tv}
Let \(P\) and \(Q\) be two probability distributions on a measurable space \((\Omega, \cF)\). The total variation (TV) distance between \(P\) and \(Q\) is defined as
\[
TV(P, Q) := \sup_{A \in \cF} |P(A) - Q(A)|,
\]
which represents the largest difference in the probabilities that the two distributions assign to the same event. If \(P\) and \(Q\) admit probability mass functions \(P(x)\) and \(Q(x)\), respectively, then the total variation distance admits the equivalent form:
\[
TV(P, Q) = \frac{1}{2} \sum_{x \in \Omega} |P(x) - Q(x)|.
\]
\end{definition}

\begin{lemma}[Pinsker's Inequality]
\label{lem:pinsker}
Let \(P\) and \(Q\) be two probability distributions on the same sample space. Then,
\[
TV(P, Q) \le \sqrt{ \frac{1}{2} \mathrm{KL}(P, Q) },
\]
where \(\mathrm{KL}(P, Q)\) denotes the Kullback-Leibler divergence between \(P\) and \(Q\). 
\end{lemma}

A key consequence of our distributional assumption in \Cref{assump:concentration} is a high-probability bound on the norm of the covariates, as formalized below.

\begin{lemma}\label{lemma:boundedness}
    Let $\delta \in (0, 1)$. Under \Cref{assump:concentration}, with probability at least $1-\delta$, we have that $\|\vx\|_2 \le S\sqrt{d}$ for all covariates $\vx$ in a total of $N$ samples from all groups, where $S=B\log(dN/\delta)$.
\end{lemma}

\begin{proof}
	Since each group has $N/K$ independent samples, applying Assumption~\ref{assump:concentration} gives for any one sample $\vxij$ and any unit $\vu$,  
	\[
	\Pr\bigl[|\vu\cdot\vxij|\ge S\bigr]\le e^{-S/B}.  
	\]
	Applying this inequality with $\vu$ chosen as each of the standard basis vectors and then applying a union bound over $d$ coordinates leads to  
	\(\Pr[\|\vxij\|_2\ge S\sqrt d]\le d\,e^{-S/B}.\)  Another union bound over all $N$ samples gives  
	\(\Pr[\exists\,\vx: \|\vx\|_2\ge S\sqrt d]\le dN\,e^{-S/B}.\)  
	Setting $S=B\log(dN/\delta)$ completes the argument.
\end{proof}

We will also need the following ``concentration-via-clipping'' argument, which we will apply in \Cref{subsec:uniform-conv}.

\begin{lemma}[One-Dimensional Concentration via Clipping]\label{lemma:concentration-via-clipping}
Let $Z$ be a real-valued random variable with mean $\mu = \E[Z]$. 
Assume $Z$ satisfies the tail bound $\Pr(|Z| > t) \le 2\exp(-t^{1/\tau}/B)$ for all $t \ge t_0$, where $t_0 > 0$ and $\tau \ge 1$ is a bounded integer. 
Let $Z_1, \dots, Z_n$ be $n$ i.i.d. copies of $Z$.

For any error $\epsilon > 0$, failure probability $\delta \in (0, 1)$, and sufficiently large $n$ satisfying  
\[
n = \tilde{O}\left( \max\left( \frac{B^{2\tau}}{\epsilon^2}, \frac{t_0^2}{\epsilon^2} \right) \log \left( \frac 1 \delta \right) \right),
\]
with probability at least $1-\delta$, we have:
\[
\left| \frac{1}{n}\sum_{i=1}^n Z_i - \mu \right| \le \epsilon.
\]
\end{lemma}
\begin{proof}
Let the clipping threshold be $S \geq t_0$, to be determined later, and define $Z_c := \max(-S, \min(Z, S))$. 
We decompose the total error using the triangle inequality:
\[
\left|\frac{1}{n}\sum_{i=1}^n Z_i - \mu\right| \le \underbrace{\left|\frac{1}{n}\sum_{i=1}^n (Z_i - Z_{i,c})\right|}_{\text{(I) Clipping Error}} + \underbrace{\left|\frac{1}{n}\sum_{i=1}^n Z_{i,c} - \E[Z_c]\right|}_{\text{(II) Concentration Error}} + \underbrace{|\E[Z_c] - \mu|}_{\text{(III) Bias Error}}.
\]
We now bound each of these three terms.

\textbf{Clipping term:} 
The first term is non-zero only if at least one sample $|Z_i| > S$. 
Let $\mathcal{E}_{\text{clip}}$ be this event. 
By a union bound over the samples and the tail bound on $Z$:
\[
\Pr(\mathcal{E}_{\text{clip}}) = \Pr(\exists i: |Z_i| > S) \le n \cdot \Pr(|Z| > S) \le 2n \exp(-S^{1/\tau}/B).
\]
This means with probability $1-2n \exp(-S^{1/\tau}/B)$, we have $\left|\frac{1}{n}\sum_{i=1}^n (Z_i - Z_{i,c})\right| = 0$.

\textbf{Bias term:} The bias $|\E[Z_c] - \mu| = |\E[Z_c - Z]|$ is bounded by the integral of the tail probability:
\[
|\E[Z_c - Z]| \le \E[|Z| \cdot \mathbb{I}(|Z|>S)] = \int_S^\infty \Pr(|Z|>t)dt \le \int_S^\infty 2e^{-t^{1/\tau}/B}dt,
\]
where the first inequality follows from the natural coupling between $Z_c$ and $Z$ assigning every point in $Z_c$ to the corresponding unclipped point in $Z$. 
If the point is unclipped, this contributes $0$,
if the point is clipped, then this means $|Z| > S$ and it contributes $|Z|-S < |Z|$. 
The equality is the standard way to express expectations of nonnegative random variables as tail integrals. 
The final inequality follows from the tail bound on $|Z|$. 

We now further upper bound the right hand side above.

Let $u = t^{1/\tau}/B$, so $t = (Bu)^\tau$ and $dt = \tau B (Bu)^{\tau-1} du$. 
The lower limit of integration becomes $u_0 = S^{1/\tau}/B$.
\[
\int_{u_0}^\infty 2e^{-u} \tau B (Bu)^{\tau-1} du = 2\tau B^\tau \int_{u_0}^\infty u^{\tau-1} e^{-u} du.
\]
For $u_0 \ge \tau-1$ (which will be satisfied by the choice of \(S\) later), 
we can bound $\int_{u_0}^\infty u^{\tau-1} e^{-u} du \le c_\tau u_0^{\tau-1} e^{-u_0}$ for some constant $c_\tau < 2$ depending on $\tau$ 
(by standard facts about exponential integrals). 
The bias is thus bounded by $2\tau B^\tau c_\tau u_0^{\tau -1} e^{-u_0} = 2\tau B^\tau c_\tau (S^{1/\tau}/B)^{\tau -1} e^{-(S^{1/\tau}/B)} = \Theta_\tau$.

\textbf{Concentration term:} Finally, we bound the concentration term. The clipped variable $Z_c$ is bounded in $[-S, S]$. We can apply Hoeffding's inequality to its empirical mean. With probability at least $1-\delta/2$:
\[
\left|\frac{1}{n}\sum_{i=1}^n Z_{i,c} - \E[Z_c]\right| \le \sqrt{\frac{2S^2\log(4/\delta)}{n}} = S\sqrt{\frac{2\log(4/\delta)}{n}}.
\]
Substituting $S = (B \log(4n/\delta))^\tau$, this deviation is bounded by $(B \log(4n/\delta))^\tau \sqrt{\frac{2\log(4/\delta)}{n}}$.

Accounting for everything, we see that, with high probability, the deviation is at most the sum of the concentration and bias errors. We need to choose the clipping threshold $S$ and the sample size $n$ to make this total error at most $\epsilon$, while keeping the failure probability at most $\delta$.

This requires satisfying the following conditions simultaneously: 
\begin{align}
    S &\ge (B \log(4n/\delta))^\tau \label{eq:cond_clip} \\
    c_\tau B S^{(\tau-1)/\tau} e^{-S^{1/\tau}/B} &\le \epsilon/2 \label{eq:cond_bias} \\
    n &\ge \frac{8 S^2}{\epsilon^2} \log(4/\delta) \label{eq:cond_conc}\\
	S &\ge t_0\\
    S &\ge (\tau-1)^{\tau} B^{\tau} \label{eq:u_0}
\end{align}
where the first condition ensures the clipping event probability is at most $\delta/2$, the second bounds the bias (where the constants have been absorbed into $c_\tau$), and the third bounds the concentration error of the clipped estimators. The second-last condition is just to ensure that the tail bound applies, and the final condition corresponds to $u_0 \geq \tau-1$. 

\textbf{Choosing $S$:} To satisfy all of these inequalities simultaneously, we define the clipping threshold $S$:
\[
S := \max\left( \left(B \log\left(\frac{C_\tau B^\tau}{\epsilon}\right)\right)^\tau, \left((\tau-1) ~B \log(4n/\delta)\right)^\tau, t_0 \right)
\]

for a sufficiently large constant $C_\tau$. The first term in the $\max$ is chosen to satisfy the bias condition \eqref{eq:cond_bias}, and the second term satisfies the clipping condition \eqref{eq:cond_clip} as well as \eqref{eq:u_0}. 
We now explain how to derive the first term.

Let us define the variable $X = S^{1/\tau}/B$. The inequality can be rewritten in terms of $X$ by substituting $S^{1/\tau} = BX$:
\begin{align*}
c_\tau B (BX)^{\tau-1} e^{-X} &\le \frac{\epsilon}{2} \\
c_\tau B^\tau X^{\tau-1} e^{-X} &\le \frac{\epsilon}{2}
\end{align*}
To satisfy this, we need to find a sufficiently large $X$. For any $\tau \ge 1$, there exists a (constant) threshold $X_0(\tau)$ such that for all $X \ge X_0(\tau)$, the polynomial term $X^{\tau-1}$ is bounded by an exponential, i.e., $X^{\tau-1} \le e^{X/2}$. Assuming $X$ is large enough to meet this condition, our inequality becomes:
\[
c_\tau B^\tau e^{X/2} e^{-X} = c_\tau B^\tau e^{-X/2} \le \frac{\epsilon}{2}
\]
Solving for $X$:
\begin{align*}
e^{-X/2} &\le \frac{\epsilon}{2 c_\tau B^\tau} \\
-X/2 &\le \log\left(\frac{\epsilon}{2 c_\tau B^\tau}\right) = -\log\left(\frac{2 c_\tau B^\tau}{\epsilon}\right) \\
X &\ge 2 \log\left(\frac{2 c_\tau B^\tau}{\epsilon}\right)
\end{align*}
This choice of $X$ is consistent with the initial assumption that $X$ is large (for small $\epsilon$), as the logarithm will be large. We can absorb the constant factors into a new constant $C_\tau$ and state that the condition is satisfied if
$
X \ge 2\log\left(C_\tau \frac{B^\tau}{\epsilon}\right).
$
Substituting back $X = S^{1/\tau}/B$, we get:
\[
\frac{S^{1/\tau}}{B} \ge 2\log\left(C_\tau \frac{B^\tau}{\epsilon}\right)
\]
This gives the required lower bound on $S$:
\[
S \ge \left(2B \log\left(C_\tau \frac{B^\tau}{\epsilon}\right)\right)^\tau
\]
This demonstrates that a choice of $S$ satisfying $S = \tilde{O}((B \log(B/\epsilon))^\tau)$ is sufficient to control the bias term.

\textbf{Solving for $n$:} Substituting this choice of $S$ into the concentration condition \eqref{eq:cond_conc} yields a lower bound on $n$. This gives three requirements on $n$, corresponding to which term in the $\max$ defining $S$ is dominant.
\begin{enumerate}
    \item If $S$ is determined by the bias term (the first term), then $n$ must satisfy:
    $n \ge \frac{8 \log(4/\delta)}{\epsilon^2} \left(B \log\left(\frac{C_\tau B^\tau}{\epsilon}\right)\right)^{2\tau}.$
    This contributes to $n = \tilde{O}\left(\frac{B^{2\tau}}{\epsilon^2} \log\left(\frac 1 \delta \right)\right)$ in the stated sample complexity.
    
\item If $S$ is determined by the clipping probability term (the second term), we get a recursive condition on $n$:
    \[ n \ge \frac{8 \log(4/\delta)}{\epsilon^2} (B \log(4n/\delta))^{2\tau}. \]
    This is a recursive inequality of the form $n \ge K \log^{2\tau}(n)$ where $K = \frac{8 B^{2\tau} \log(4/\delta)}{\epsilon^2}$; by the standard technique for solving such recursive inequalities this is satisfied by all $n \geq \max((2\tau e)^{2\tau}, 2K\log^{2\tau}(2K)) \allowbreak = \Theta(K\log^{2\tau}(K))$, and so it suffices that:
    \[ n = \tilde{O}\left(K \log^{2\tau}(K)\right) = \tilde{O}\left(\frac{B^{2\tau}}{\epsilon^2} \log^{2\tau}\left(\frac{B^{2\tau}}{\epsilon^2}\right)\right). \]
    Since the $\tilde{O}$ notation absorbs polylogarithmic factors, this simplifies to $n = \tilde{O}\left(\frac{B^{2\tau}}{\epsilon^2} \log(\frac 1 \delta)\right)$.
    
    \item If $S$ is determined by the threshold term $t_0$ (the third term), then $S = t_0$ and we need:
    $ n \ge \frac{8 \log(4/\delta)}{\epsilon^2} t_0^2. $
    This contributes $n = O \left(\frac{t_0^2}{\epsilon^2} \log(1/\delta)\right)$ to the sample complexity.
\end{enumerate}
The total sample complexity is the maximum of these three requirements, giving us the sample complexity stated in the lemma. 
\end{proof}

\section{Uniform Convergence and Population Approximation}\label{subsec:uniform-conv}

In Section~\ref{sec:algor-conv-analys} we carried out our main convergence analysis under the empirical distributions $\epi$.  
Here we show that, with high probability, all required empirical expectations closely match their population counterparts under each true distribution $\ppi$.  
Our approach combines 
\Cref{lemma:concentration-via-clipping} 
with a standard net argument after clipping the functions involved, to obtain convergence uniformly over the weight domain $\cB(W)$, 
and then applies this result to the sharpness and moment bounds as well as to control the difference between the empirical and population optima.

\begin{lemma}[Uniform Concentration for Heavy-Tailed Variables]
\label{lemma:uniform_concentration_heavy_tailed}
    Let $h:\cB(W)\times\R^d\to\R$ be a function. 
	Assume $\vx$ satisfies Assumption~\ref{assump:concentration} with parameter $B_x$,
	and for any $\vw \in \cB(W)$, the random variable $Z_\vw = h(\vw; \vx)$ 
	satisfies the tail bound $\Pr(|Z_\vw| > t) \le 2\exp(-t^{1/\tau}/B_h)$ for all $t \ge t_0$, 
	where $t_0 = \tilde{O}(WB_x\beta)$ may depend on problem parameters, $\tau \ge 1$ is a bounded integer and $B_h \ge B_x$.
    
    Let $a(\vx)$ be the Lipschitz constant of $h(\vw; \vx)$ with respect to $\vw \in \cB(W)$,
    and suppose that $a(\vx) \leq a_{S_x}$ whenever $\|\vx\|_2 \le \sqrt{d}S_x$, for $S_x = \tilde{O}(B_x)$ chosen according to \Cref{lemma:boundedness} and $a_{S_x} = \tilde O(d^2B_x^4)$ as required in \Cref{lemma:sharpness_empirical_heavy_tailed} below.
    
    Then, for sufficiently large $N$ such that 
      \[
    N = \tilde{O}\left( K \cdot \frac{t_0^2 + B_h^{2\tau}}{\epsilon^2} \left(d\log\frac{dW B_x}{\epsilon} + \log\frac{K}{\delta}\right) \right),
    \]
    with probability at least $1-\delta$, for every group $i\in[K]$ and all $\vw\in\cB(W)$:
    \[
    \bigl|\E_{\epi}[h(\vw;\vx)]-\E_{\ppi}[h(\vw;\vx)]\bigr|\le 
    \epsilon.
    \]
\end{lemma}
\begin{proof}
We will follow a clipping argument similar to that in \Cref{lemma:concentration-via-clipping}.
We define a clipping threshold for the data norm, $S_x = B_x \log(3dN^2/\delta)$. 
Let $\mathcal{E}_x$ be the event $\|\vx\|_2 \le \sqrt{d}S_x$. 
We decompose $h$ into a ``clipped'' part and a ``tail'' part---$h_c(\vw; \vx) := h(\vw; \vx) \cdot \mathbb{I}[\mathcal{E}_x]$ and $h_t(\vw; \vx) := h(\vw; \vx) \cdot \mathbb{I}[\neg\mathcal{E}_x]$,
so that $h = h_c + h_t$. 
The total deviation can then be bounded using the triangle inequality:
\[
|\E_{\epi}[h] - \E_{\ppi}[h]| \le \underbrace{|\E_{\epi}[h_c] - \E_{\ppi}[h_c]|}_{\text{(I)}} + \underbrace{|\E_{\epi}[h_t]|}_{\text{(II)}} + \underbrace{|\E_{\ppi}[h_t]|}_{\text{(III)}}.
\]
We will show that (I) is uniformly bounded by 
$\eps$
and that 
(II) and (III) are each uniformly bounded by $\epsilon/3$ with high probability.

\paragraph{Bounding the Tail Contribution (Terms II and III)}

By \Cref{lemma:boundedness}, with probability at least $1-\delta/(3N)$, 
all covariates $\vx$ in the total sample of size $N$ satisfy $\|\vx\|_2 \le \sqrt{d}S_x$, 
where $S_x = B_x\log(3dN^2/\delta)$. 
On this high-probability event, the indicator $\mathbb{I}[\|\vx_j\|_2 > \sqrt{d}S_x]$ 
is zero for all samples $\vx_j$. 
Consequently, for any group $i$, the empirical expectation of the tail part is zero:
\[
\E_{\epi}[h_t] = \frac{K}{N}\sum_{j=1}^{N/K} h(\vw;\vx_{ij})\mathbb{I}[\|\vx_{ij}\|_2 > \sqrt{d}S_x] = 0.
\]
Since this does not depend on $\vw$ or $i$, 
Term (II) is zero for all $\vw$ and all groups $i \in [K]$ simultaneously.

For Term (III), 
we bound the population mean of the tail part for all $\vw$, $|\E_{\ppi}[h_t]|$, 
using the Cauchy-Schwarz inequality:
\begin{align*}
|\E_{\ppi}[h_t]| &\le \E_{\ppi}\bigl[|h(\vw; \vx)| \cdot \mathbb{I}[\|\vx\|_2 > \sqrt{d}S_x]\bigr] \\
&\le \sqrt{\E_{\ppi}[h(\vw; \vx)^2]} \cdot \sqrt{\Pr(\|\vx\|_2 > \sqrt{d}S_x)}\\
&\le \sqrt{\E_{\ppi}[h(\vw; \vx)^2]} \cdot \sqrt{\delta/(3N)}.
\end{align*}
The second inequality is the bound on the data probability, which is bounded by $\delta/(3N)$
by our choice of $S_x$.
The first term is the square root of the second moment of $h(\vw;\vx)$. 
We can bound this moment using its tail parameters. 
Using the tail integral formula, $\E[Z^2] = \int_0^\infty \Pr(|Z|>\sqrt{u})du$, 
we split the integral at $u=t_0^2$:
\[
\E_{\ppi}[h(\vw; \vx)^2] \le \int_0^{t_0^2} 1 du + \int_{t_0^2}^\infty 2\exp(-u^{1/(2\tau)}/B_h) du = t_0^2 + 4\tau B_h^{2\tau} \Gamma(2\tau, t_0^{1/\tau}/B_h) \leq \tilde O(t_0^2 + B_h^{2\tau}), 
\]
where the final inequality is a consequence standard upper bounds on the $\Gamma$ function.
Thus, we have $|\E_{\ppi}[h_t]| \le \sqrt{\tilde O(t_0^2 + B_h^{2\tau})} \sqrt{\delta/(3N)}$. 
To ensure this is less than $\epsilon/3$, 
it suffices to choose a sufficiently large $N = \tilde{O} (t_0^2 + B_h^{2\tau})~\delta / \epsilon^2$. 

\paragraph{Uniform Concentration of the Clipped Part (Term I)}
Let $a_{S_x} > \sup_{\|\vx\|_2 \le \sqrt{d}S_x} a(\vx)$ be a deterministic upper bound on $a(\vx)$.
For any $\vx$ where $\mathbb{I}[\mathcal{E}_x]=1$, 
we have $\|\vx\|_2 \le \sqrt{d}S_x$, 
and therefore $|h_c(\vw_1;\vx) - h_c(\vw_2;\vx)| \le a(\vx)\|\vw_1-\vw_2\|_2 \le a_{S_x}\|\vw_1-\vw_2\|_2$.

Since their random variable 
$Z_{\vw,c} = h_c(\vw;\vx)$ is a clipped version of $Z_\vw=h(\vw;\vx)$, 
its tails are at least as light.
We follow the standard net argument for $h_c$. 
Let $\cN$ be an $r$-net over $\cB(W)$ with $|\cN| \leq (3W/r)^d$. 
We want to achieve $|\E_{\epi}[h_c] - \E_{\ppi}[h_c]| \le \epsilon/3$ for all points in the net and all groups. Using Lemma~\ref{lemma:concentration-via-clipping} with accuracy $\epsilon/3$ and failure probability $\delta' = \frac{\delta/3}{K|\cN|}$, further, substituting $r = \epsilon/6a_{S_x}$, the required sample size $N=nK$ satisfies:
\[
N = \tilde{O}\left( K \cdot \max\left( \frac{B_h^{2\tau}}{\epsilon^2}, \frac{t_0^2}{\epsilon^2} \right) \left(d\log\frac{W a_{S_x}}{\epsilon} + \log\frac{K}{\delta}\right) \right).
\]
Plugging in bound on $a_{S_x}$ from the lemma statement, we recover the stated sample complexity. With this sample size, with probability at least $1-\delta/3$, we have $|\E_{\epi}[h_c(\vw')] - \E_{\ppi}[h_c(\vw')]| \le \epsilon/3$ for all $\vw' \in \cN$.

Extending to the continuum via the triangle inequality and the Lipschitz constant $a_{S_x}$, for any $\vw \in \cB(W)$:
\[ |\E_{\epi}[h_c] - \E_{\ppi}[h_c]| 
\le |\E_{\epi}[h_c(\vw)] - \E_{\epi}[h_c(\vw')]| 
+ \epsilon/3 
+ |\E_{\ppi}[h_c(\vw')] - \E_{\ppi}[h_c(\vw)]| 
\le 2a_{S_x}r + \epsilon/3. \]

We have three events we need to hold: 
the empirical tail is zero, 
the population tail is small,
and the clipped part concentrates. 
By a union bound, all three hold with probability at least $1-\delta$. 
On this combined event, we have:
\[ |\E_{\epi}[h] - \E_{\ppi}[h]| \le (2a_{S_x}r + \epsilon/3) + 0 + \epsilon/3 \le 2a_{S_x}r + (2\epsilon/3) \leq \epsilon. \]

While the stated inequalities for $S_x$ and $N$ lead to a recursive bound on $N$---the inequality for $N$ has a $\log \log(N)$ term on the RHS due to $S_x$ scaling logarithmically with $N$---using standard techniques like the one from the end of \Cref{lemma:concentration-via-clipping}, we can obtain an explicit bound on $N$, at a poly-log cost absorbed by the $\tilde{O}$ notation. 
\end{proof}

	\begin{lemma}[Empirical Sharpness and Moment Bounds with Heavy Tails]\label{lemma:sharpness_empirical_heavy_tailed}
    Under Assumptions~\ref{assump:concentration} and \ref{assump:margin}, let the number of samples 
     \[
    N = \tilde{O}_{\beta, B}\left( \frac{KW^4}{\eps^2} \left(d\log\frac{dW}{\epsilon} + \log\frac{K}{\delta}\right) \right)
    \]
    be sufficiently large. 
    Then with probability at least $1-\delta$, for every group $i\in[K]$, every $\vw\in\cB(3\|\vw_*\|)$ with $\|\vw-\vw_*\|\ge\sqrt\eps$, and every unit vector $\vu$, the following hold:
    \begin{align*}
        \E_{\vx\sim\epi}&\bigl[(\sigma(\vw\cdot\vx)-\sigma(\vw_*\cdot\vx))(\vw\cdot\vx-\vw_*\cdot\vx)\bigr]\ge \tfrac{c_0}{2}\|\vw-\vw_*\|_2^2,\\
        \E_{\vx\sim\epi}&\bigl[(\vu\cdot\vx)^\tau\bigr]\le 6B, \quad \text{for } \tau=2,4.
    \end{align*}
\end{lemma}

\begin{proof}
The proof consists of applying the uniform concentration result for heavy-tailed variables (Lemma~\ref{lemma:uniform_concentration_heavy_tailed}) to three different functions. We will analyze the sample complexity for each and take the maximum to get the final result. 

\paragraph{1. Sharpness Bound.}
Let $h_1(\vw;\vx) = (\sigma(\vw\cdot\vx)-\sigma(\vw_*\cdot\vx))(\vw\cdot\vx-\vw_*\cdot\vx)$. By Assumption~\ref{assump:margin}, its true expectation is lower bounded: $\E_{\ppi}[h_1(\vw;\vx)] \ge c_0\|\vw-\vw_*\|_2^2$. We want to show that with high probability, the empirical expectation is close to this value.

The function class $\{h_1(\vw;\cdot) : \vw \in \cB(W)\}$ is defined over the parameter ball $\vw \in \cB(W)$. Let $\Delta\vw := \vw-\vw_*$, and assume $\|\Delta\vw\|_2 \le 2W$ for all $\vw$ under consideration. 
Consider the random variable $Z_\vw = h_1(\vw;\vx)$. Its tail behavior can be bounded as follows:

Since $\sigma$ is $\beta$-Lipschitz, we have $|\sigma(\vw\cdot\vx)-\sigma(\vw_*\cdot\vx)| \le \beta|\Delta\vw\cdot\vx|$. This implies:
\[
|Z_\vw| = |h_1(\vw;\vx)| \le \beta(\Delta\vw\cdot\vx)^2 = \beta \|\Delta\vw\|_2^2 (\vu\cdot\vx)^2 \le 4\beta W^2 (\vu\cdot\vx)^2,
\]
where $\vu = \Delta\vw/\|\Delta\vw\|_2$.

By Assumption~\ref{assump:concentration}, the random variable $(\vu\cdot\vx)^2$ has a tail bound $\Pr((\vu\cdot\vx)^2 > t) \le 2\exp(-t^{1/2}/B)$. 
The tail bound for $Z_\vw$ is therefore:
\[
\Pr(|Z_\vw| > t) \le \Pr\left((\vu\cdot\vx)^2 > t/(4\beta W^2)\right)\le 2\exp\left(-\frac{(t/(4\beta W^2))^{1/2}}{B}\right) = 2\exp\left(-\frac{t^{1/2}}{2B \sqrt{\beta W^2}}\right).
\]
This shows that $Z_\vw$ has a heavy tail with shape parameter $\tau_{h_1} = 2$ and an effective scale parameter $B_{h_1} = 2\sqrt{\beta} W B$.

To apply Lemma~\ref{lemma:uniform_concentration_heavy_tailed}, we also need the effective Lipschitz constant of the function class, which we denote by $a_1$. 
To find the Lipschitz constant of $h_1(\vw;\vx)$ with respect to $\vw$, we bound its gradient's norm: \begin{align*}
    \|\nabla_\vw h_1(\vw;\vx)\|_2 &= \|\sigma'(\vw\cdot\vx)\vx(\vw\cdot\vx-\vw_*\cdot\vx) + (\sigma(\vw\cdot\vx)-\sigma(\vw_*\cdot\vx))\vx\|_2 \\
    &\le |\sigma'(\vw\cdot\vx)| \cdot \|\vx\|_2 \cdot |\Delta \vw\cdot\vx| + |\sigma(\vw\cdot\vx)-\sigma(\vw_*\cdot\vx)| \cdot \|\vx\|_2 \\
    &\le \beta \|\vx\|_2 \cdot (\|\Delta \vw\|_2 \|\vx\|_2) + (\beta \|\Delta \vw\|_2 \|\vx\|_2) \cdot \|\vx\|_2 \\
    & \le 4\beta W \|\vx\|_2^2.
\end{align*}
This gives a random Lipschitz constant $a_1(\vx) = 4\beta W \|\vx\|_2^2$. 
Note that $a_1(\vx) \leq 4\beta W S^2 d \le \tilde{O}(\beta d W B^2)$ if $\|\vx\|_2^2 \le d S_x^2$ for $S_x = \tilde{O}(B)$.

Since we consider $\|\vw-\vw_*\|^2 \ge \eps$, we can now use Lemma~\ref{lemma:uniform_concentration_heavy_tailed} with a target additive error $\frac{c_0}{4}\eps$ (i.e. replace $\eps$ in C.2 with $c_0\eps/4$), and a net radius $r = \frac{c_0\eps}{8a_1}$. The required number of samples $N_1$ is dominated by the heavy-tailed term in the sample complexity bound, which must use the parameters specific to $h_1$:
Setting $t_0 = 0$ in \Cref{lemma:uniform_concentration_heavy_tailed}, we start with the given sample complexity:
\[
N_1 = \tilde{O}\left( K \cdot \frac{ (B_{h_1})^{2\tau_{h_1}} 
}{(c_0 \eps/4)^2}
\left(\log\frac{K}{\delta} + d\log \frac{dWB}{\epsilon}\right)\right)
\]
Substituting the parameters $\tau_{h_1}=2$, $B_{h_1} = 2\sqrt{\beta} W B$, we proceed with the calculation:
\begin{align*}
    N_1 &= \tilde{O}\left( K \cdot \frac{(2\sqrt{\beta} W B)^{2 \cdot 2} 
    }{(c_0\eps/4)^2} \left(\log\frac{K}{\delta} + d\log \frac{dWB}{\epsilon}\right)\right) \\
    &= \tilde{O}\left( K \cdot \frac{64 \beta^2 W^4 B^4}{c_0^2\eps^2} 
    \left(d\log\frac{dWB}{\eps} + \log\frac{K}{\delta}\right)\right).
\end{align*}
Absorbing the numerical constants and the logarithmic factors into the $\tilde{O}$ notation, we arrive at the simplified expression for the sample complexity:
\[
N_1 = \tilde{O}\left( \frac{K \beta^2 W^4 B^4}{\eps^2} \left(d\log\frac{dWB}{\epsilon} + \log\frac{K}{\delta}\right) \right).
\]
With this sample size, with probability at least $1-\delta/2$, for all $\vw \in \cB(W)$ with $\|\vw-\vw_*\|_2 \ge \sqrt{\eps}$:
\[
\E_{\epi}[h_1(\vw;\vx)] \ge \E_{\ppi}[h_1(\vw;\vx)] - (2a_1r+t) \ge c_0\|\vw-\vw_*\|_2^2 - \left(\tfrac{c_0\eps}{4} + \tfrac{c_0\eps}{4}\right) \ge \frac{c_0}{2}\|\vw-\vw_*\|_2^2.
\]

\paragraph{2. Moment Bounds.}
Let $h_2(\vu;\vx) = (\vu\cdot\vx)^4$ and $h_3(\vu;\vx) = (\vu\cdot\vx)^2$. These are defined over the space of unit vectors, $\vu \in \mathbb{S}^{d-1}$. By Assumption~\ref{assump:concentration}, the population means are bounded, $\E_{\ppi}[h_j] \le 5B$, and their tails are given by:
\[ \Pr(h_2(\vu; \vx) > t) \leq 2\exp(-t^{1/4}/B) \quad (\tau_2=4, B_2=B) \]
\[ \Pr(h_3(\vu; \vx) > t) \leq 2\exp(-t^{1/2}/B) \quad (\tau_3=2, B_3=B) \]

To apply Lemma~\ref{lemma:uniform_concentration_heavy_tailed}, we first need upper bounds on the Lipschitz constants for $h_2$ and $h_3$ in the appropriate bounded region.

\textbf{Lipschitz Constants:} The gradients are $\nabla_\vu h_2(\vu;\vx) = 4(\vu\cdot\vx)^3\vx$ and $\nabla_\vu h_3(\vu;\vx) = 2(\vu\cdot\vx)\vx$. Their norms are bounded by random variables:
\[
\|\nabla_\vu h_2(\vu;\vx)\|_2 \le 4\|\vx\|_2^4, \quad \text{and} \quad \|\nabla_\vu h_3(\vu;\vx)\|_2 \le 2\|\vx\|_2^2.
\]
Similar to getting the upper bound for \(a_1(\vw)\), we have $a_2(\vx) = \tilde{O}(d^2B^4)$ and $a_3(\vx) = \tilde{O}(d B^2)$ if \(\norm{\vw}_2^2\le dS_{x}^2\).

\textbf{Sample Complexity and Concentration:} 
We apply Lemma~\ref{lemma:uniform_concentration_heavy_tailed} with a target deviation of $\epsilon=B/2$ and net radii $r_j = \epsilon/(2a_j) = B/(4a_j)$ for \(j=2,3\). For the functions $h_2$ and $h_3$, the tail bounds hold for all $t \ge 0$, so we can take $t_0=0$. The domain for $\vu$ is the unit sphere, so we take $W=1$.

For $h_2(\vu;\vx) = (\vu\cdot\vx)^4$, we have tail parameters $\tau_2=4$ and $B_2=B$. Since the Lipschitz constant is upper bounded by $a_2 = \tilde{\Theta}(d^2 B^4)$ whenever $\|x\|_2^2 \leq d S_x^2$, the sample complexity $N_2$ is given by the formula from Lemma~\ref{lemma:uniform_concentration_heavy_tailed}:
\begin{align*}
    N_2 &=
\tilde{O}\left( K \cdot \frac{B^{2 \cdot 4}}{(B/2)^2}\left(d\log\frac{4a_2}{B} + \log\frac{K}{\delta}\right) \right) \\
&= \tilde{O}\left( K B^6 \left(d + \log\frac{K}{\delta}\right) \right),
\end{align*}
where we absorbed the polylogarithmic factors in $B$ and $d$ into the $\tilde{O}$ notation.
For $h_3(\vu;\vx) = (\vu\cdot\vx)^2$, we have tail parameters $\tau_3=2$ and $B_3=B$. The Lipschitz constant is upper bounded by $a_3 = \tilde{\Theta}(d B^2)$ whenever $\|x\|_2^2 \leq d S_x^2$. The sample complexity $N_3$ is similarly calculated to be:
\[
N_3 = \tilde{O}\left( K B^2 \left(d + \log\frac{K}{\delta}\right) \right).
\]

The sample complexity for the moment bounds is dominated by $N_2$. With this many samples, with probability at least $1-\delta/2$, we have for all unit vectors $\vu$:
\[
\E_{\epi}[h_2(\vu;\vx)] \le \E_{\ppi}[h_2(\vu;\vx)] + (2a_2r_2+\epsilon) \le 5B + \left(2a_2\frac{\epsilon}{2a_2}+\epsilon\right) = 5B + 2\epsilon.
\]
With our choice of $\epsilon = B/2$, the total bound is $5B + 2(B/2) = 6B$. A similar analysis for $h_3$ also gives us a bound of $6B$ for that case. By combining the three sample complexity bounds and omitting the dependence on 
\(\beta\) and \(B\), we obtain the overall sample complexity in the lemma statement.
\end{proof}

\begin{lemma}[Empirical vs. Population Optima, Heavy-Tailed Version]
	\label{lemma:opt_opthat_heavy_tailed}
    Under Assumptions~\ref{assump:concentration} and \ref{assump:margin}, and in the setting of \Cref{fact:truncation}, let the number of samples be
    \[
        N = \tilde{O}_{\beta,B}\left(\frac{K W^4 }{\epsilon^2} \log \left( \frac 1 \delta \right) \right)
    \]
    Then with probability at least $1-\delta$, it holds that $\mOPThat \le \mOPT + \epsilon$.
\end{lemma}

\begin{proof}
The proof proceeds in three steps: 
first, we derive the tail properties of the loss function; 
second, we apply a concentration inequality with a union bound; and 
third, we use a property of the maximum function to conclude.

Recall that \Cref{fact:truncation} bounds $|y| \le M$ with 
$M = C_M W B \beta \log(\beta B W / \epsilon)$ for a large enough constant $C_M$.

\paragraph{Tail Bound of the Loss Function.} 
Let the loss be the random variable $Z = l(\vw_*;\vx,y)=(\sigma(\vw_*\cdot\vx) - y)^2$. 
We know $\sigma$ is $\beta$-Lipschitz continuous, $\sigma(0) = 0$, 
and $y$ is bounded by $M$.
To find the tail bound for $Z$, we first establish an upper bound on its value:
\[
Z = (\sigma(\vw_*\cdot\vx) - y)^2 \le 2(\sigma(\vw_*\cdot\vx)^2 + y^2) \le 2(\beta^2(\vw_*\cdot\vx)^2 + M^2).
\]
The term $(\vw_*\cdot\vx)^2$ can be written as 
$\|\vw_*\|_2^2 (\vu\cdot\vx)^2 \le W^2 (\vu\cdot\vx)^2$, 
where $\vu = \frac{\vw_*}{\|\vw_*\|_2}$ is a unit vector. 
This gives $Z \le 2(W^2\beta^2(\vu\cdot\vx)^2 + M^2)$. 
We can now bound the tail probability of $Z$ for any $t > 2M^2$:
\begin{align*}
    \Pr(Z > t) 
	&\le \Pr(2(W^2\beta^2(\vu\cdot\vx)^2 + M^2) > t) 
	= \Pr\left((\vu\cdot\vx)^2 > \frac{t/2 - M^2}{(\beta W)^2}\right).
\end{align*}
From Assumption~\ref{assump:concentration}, 
the variable $(\vu\cdot\vx)^2$ has a sub-exponential tail with $\tau=2$ and parameter $B$.
Let $t' = (t/2 - M^2)/(\beta W)^2$. The tail bound is $\Pr((\vu\cdot\vx)^2 > t') \le 2\exp(-(t')^{1/2}/B)$. Substituting $t'$ back gives:
\[
\Pr(Z > t) \le 2\exp\left(-\frac{\sqrt{t/2 - M^2}}{\beta WB}\right).
\]
To establish a clean sub-exponential form of the type $\exp(-t^{1/2}/B_l)$, we consider the tail for $t \ge 4M^2$. For this range of $t$, we have $t/2 - M^2 \ge t/2 - t/4 = t/4$. Therefore, $\sqrt{t/2 - M^2} \ge \sqrt{t/4} = \sqrt{t}/2$. Substituting this into the exponent gives:
\[
\Pr(Z > t) \le 2\exp\left(-\frac{\sqrt{t}/2}{\beta WB}\right) = 2\exp\left(-\frac{t^{1/2}}{2\beta WB}\right).
\]
This bound holds for all $t \ge 4M^2 = \tilde{\Theta}(\beta^2 W^2B^2)$. 
This is a valid sub-exponential tail with exponent $\tau_l=2$ and an effective tail parameter $B_l = 2WB\beta$.

\paragraph{Concentration and Union Bound over Groups.}
Recall that $\mOPThat = \max_{i\in[K]} \E_{\epi}[l(\vw_*;\vx,y)]$ and 
$\mOPT = \max_{i\in[K]} \E_{\ppi}[l(\vw_*;\vx,y)]$.
Also, let \(\vl^{*}=[l^{*}_{[1]},\dots,l^{*}_{[K]}]\), where \(l^{*}_{[i]}=\E_{\ppi}(l(\vw_*;\vx,y))\), for all \(i\in [K]\), and we define \(\hat{\vl}^{*}\), $\hat{l}^*_{[i]}$ for the empirical case similarly. To ensure the empirical loss $\hat{l}^*_{[i]}$ concentrates around the true loss $l^*_{[i]}$ for all $K$ groups simultaneously, we apply Lemma~\ref{lemma:concentration-via-clipping} combined with a union bound. 
We require the deviation for each group to be at most $\epsilon$. 
To achieve a total failure probability of at most $\delta$, 
we set the per-group failure probability to $\delta' = \delta/K$. 
Lemma~\ref{lemma:concentration-via-clipping} states that for $\tau_l=2$ and $t_0 = \tilde{\Theta}(\beta^2 W^2B^2)$, 
the number of samples per group, $n=N/K$, must satisfy:
\[
n = \tilde{O}\left( \frac{ B_l^{2\tau_l}+ \beta^4 W^4B^4}{\epsilon^2}\left(\log\frac{B_l}{\epsilon}\right)^{2\tau_l} \log\left(\frac 1 \delta \right)\right) 
= \tilde{O}\left( \frac{\beta^4 W^4 B^4}{\epsilon^2} \log \left( \frac 1 \delta \right) \right).
\]
Consequently, the total number of samples $N=nK$ must satisfy the bound stated in the lemma. With the sample size \(N\), the union bound guarantees that with probability at least $1-\delta$, 
we have $|\hat{l}^*_{[i]} - l^*_{[i]}| \le \epsilon$ for all $i \in [K]$, i.e.,  $\|\hat{\vl}^* - \vl^*\|_\infty \le \epsilon$.

Finally, note that $\mathbf{a}, \mathbf{b} \in \R^K$, $|\max_i a_i - \max_j b_j| \le \|\mathbf{a} - \mathbf{b}\|_\infty$, we have:
\[
|\mOPThat - \mOPT| = |\max_i \hat{l}^*_{[i]} - \max_j l^*_{[j
]}| \le \|\hat{\vl}^* - \vl^*\|_\infty \le \epsilon.
\]
This implies $\mOPThat \le \mOPT + \epsilon$, completing the proof.
\end{proof}

\section{Proof of Gap Lower Bound}\label{app:gap-lower-bound}
In this section, we adapt techniques from~\citep{Li2024shifts} to establish \Cref{lemma:gap-lower-bound}, restated below for convenience.

\lemgaplb*
\begin{proof}
We split the primal-dual gap into the primal gap and the dual gap:
\(
\Gap(\vw,\evlambda)
= \bigl[L(\vw,\evlambda^*)-L(\vw_*,\evlambda^*)\bigr]
+ \bigl[L(\vw_*,\evlambda^*)-L(\vw_*,\evlambda)\bigr]
\).
We expand the primal gap as
    \begin{align*}
      & \;L(\vw, \evlambda^{*}) - L(\vw_*, \evlambda^{*}) = \sum_{i=1}^K\elambda_{[i]}^{*}\E_{(\vx, y) \sim \epi}[(\sigma(\vw \cdot \vx) - y)^{2} - (\sigma(\vw_* \cdot \vx) - y)^{2}] \\
      = & \sum_{i=1}^K\elambda_{[i]}^{*}\E_{\epi}[((\sigma(\vw \cdot \vx) - \sigma(\vw_* \cdot \vx))^{2}] -2\sum_{i=1}^K\elambda_{[i]}^{*}\E_{\epi}[(\sigma(\vw_* \cdot \vx) - y)(\sigma(\vw \cdot \vx) - \sigma(\vw_* \cdot \vx))].  
    \end{align*}

    We now bound both of the above terms. For the first term, \Cref{eq:relu-square-bound} implies,
    \begin{align}\label{eq:lb-gap-sharpness}
       \sum_{i=1}^K\elambda_{[i]}^{*}\E_{\epi}[((\sigma(\vw \cdot \vx) - \sigma(\vw_* \cdot \vx))^{2}]
       \geq \sum_{i=1}^K\elambda_{[i]}^{*} c_1 \norm{\vw - \vw_*}_{2}^{2}=c_1 \norm{\vw - \vw_*}_{2}^{2}.
    \end{align}
For the second term, by the Cauchy-Schwarz inequality, we have
\begin{align}\label{eq:lb-gap-CS}
       &\; \sum_{i=1}^K\elambda_{[i]}^{*}\E_{\epi}[(\sigma(\vw_* \cdot \vx) - y)
       (\sigma(\vw \cdot \vx) - \sigma(\vw_* \cdot \vx))]\notag\\
       \leq \; & \sum_{i=1}^K\elambda_{[i]}^{*}\sqrt{\E_{\epi}[(\sigma(\vw_* \cdot \vx) - y)^2]\E_{\epi}[(\sigma(\vw \cdot \vx) - \sigma(\vw_* \cdot \vx))^2]}\notag\\
        \leq\; &  \sum_{i=1}^K\elambda_{[i]}^{*} \sqrt{\E_{\epi}[(\sigma(\vw_* \cdot \vx) - y)^2]} \beta\sqrt{6 B}\|\vw - \vw_*\|_2 \le   \sqrt{\mOPThat} \beta\sqrt{6 B}\|\vw - \vw_*\|_2,
    \end{align}
    where the second inequality follows from \Cref{eq:relu-square-bound} and the last inequality applies Jensen's inequality to the concave function $\sqrt{\cdot}$ and then uses the definition of $\mOPThat$.

	Combining \Cref{eq:lb-gap-sharpness} and  \Cref{eq:lb-gap-CS}, the primal gap is bounded below by
    \begin{align}\label{eq:lb-gap-primal}
	- 2 \beta \sqrt{6 B} \|\vw - \vw_*\|_{2} \sqrt{\mOPThat }  + c_1 \|\vw - \vw_*\|_{2}^{2}
      {\ge -\frac{12 \beta^{2} B}{c_{1}} \mOPThat + \frac{c_1}{2} \|\vw - \vw_*\|_{2}^{2}},
    \end{align}
    where we use Young's inequality (\Cref{fact:young}): $2 \beta \sqrt{6 B} \|\vw - \vw_*\|_{2} \sqrt{\mOPThat } \leq \frac{4\beta^2 {6B}}{2c_1}\mOPThat  + \frac{c_1}{2} \|\vw - \vw_*\|_{2}^{2}$.
    For the dual gap, we have
    \begin{align}\label{eq:lb-gap-dual}
        & \; L(\vw_*, \evlambda^{*}) - L(\vw_*, \evlambda)
        = -L(\vw_*, \evlambda) - (-L(\vw_*, \evlambda^{*})) \notag \\
        = &\;  \nabla_{\evlambda}(-L(\vw_*, \evlambda^{*})) \cdot (\evlambda - \evlambda^{*})+D_{-L(\vw_*,\cdot)}(\evlambda, \evlambda^{*}) \geq \nu D_{\phi}(\evlambda^{*}, \evlambda),
    \end{align}

    where the second equality and the last inequality follow from \Cref{fact:bregman-linear-blind,fact:firstOrderNecessary}. Combining \Cref{eq:lb-gap-primal} and \Cref{eq:lb-gap-dual} completes the proof.         
\end{proof}

\section{Proof of Gap Upper Bound}\label{app:gap-upper-bound}
In this section we prove \Cref{lemma:gap-upper-bound}, which shows that the cumulative primal-dual gap 
\(\sum_{t=1}^{n}a_{t}\,\Gap(\vw_{t},\evlambda_{t})\)
can be controlled by telescoping the contributions from the primal updates in \(\vw\), the dual updates in \(\evlambda\), and a small residual term that is absorbed by our choice of step sizes.  The argument proceeds in three steps:
\begin{enumerate}
	\item \textbf{Per-iteration decomposition} We combine \Cref{lemma:gap-upper-bound1,lemma:gap-upper-bound2} (an upper bound on \(a_tL(\vw_{t}, \evlambda^{*})\) and a lower bound on \(a_tL(\vw_*, \evlambda_{t})\), respectively) to bound from above each \(a_{t}\Gap(\vw_{t},\evlambda_{t})\) by a sum of differences in squared distances, Bregman divergences, and a residual error \(E_{t}\). Most of the technical work is devoted to proving \Cref{lemma:gap-upper-bound2}. 
	\item \textbf{Telescoping Sum} Summing the per-iteration bounds from \(t=1\) to \(n\) causes most distance and divergence terms to telescope, leaving only boundary terms at \(t=0\) and \(t=n\), plus one remaining inner product \(-a_n\sum_{i=1}^K(\elambda_{n[i]}-\elambda_{n-1[i]})\E_{\epi}[\innp{\vv(\vw_{n-1};\vx,y), \vw_*-\vw_n}]\).
\item \textbf{Residual Control} We bound that final inner product via Young's inequality and our step-size choice, absorbing it back into the telescoped Bregman and distance terms.
\end{enumerate}

Before we carry out these steps in detail, we first state the two necessary bounds on \(a_tL(\vw_{t}, \evlambda^{*})\) and \(a_tL(\vw_*, \evlambda_{t})\) that are required per-iteration decomposition. We defer their proofs to \Cref{sec:gap-upper-bound1,sec:gap-upper-bound2}.

\begin{restatable}[Upper Bound for \(a_tL(\vw_{t}, \evlambda^{*})\)]{lemma}{lemGapUpperOne}\label{lemma:gap-upper-bound1}
	In each iteration \(t\), let 
	\[
	\evlambda_{t}
	\;=\;
	\argmax_{\evlambda\in\Delta_{K}}\bigl\{\,a_tL(\vw_t, \evlambda)\;-\;(\nu_0+\nu A_{t-1})\,D_{\phi}(\evlambda, \evlambda_{t-1})\bigr\}.
	\]
	Then for all \(t\ge 1\),
	\begin{align*}
		a_tL(\vw_t, \evlambda^{*})
		\;\le\;&
		a_tL(\vw_t, \evlambda_t)
		\;-\;(\nu_0+\nu A_{t-1})\,D_{\phi}(\evlambda_t,\evlambda_{t-1})
		\;-\;(\nu_0+\nu A_t)\,D_{\phi}(\evlambda^{*},\evlambda_t)\\
		&\;+\;(\nu_0+\nu A_{t-1})\,D_{\phi}(\evlambda^{*},\evlambda_{t-1}).
	\end{align*}
\end{restatable}

\begin{restatable}[Lower Bound for \(a_tL(\vw_*, \evlambda_{t})\)]{lemma}{lemGapUpperTwo}\label{lemma:gap-upper-bound2}
	Under \Cref{assump:margin,assump:concentration},  
if the per-group sample size $N/K$ is sufficiently large (see \Cref{lemma:sharpness_empirical_heavy_tailed}), where  \(\vw_t\) is the sequence of iterates in \Cref{alg:main}, we have for each \(t \ge 1\):
	\begin{align*}
		a_tL(\vw_*,\evlambda_t)\ge \; & a_tL(\vw_t,\evlambda_t) - \; \frac{1+0.5c_1A_{t-1}}{2}\|\vw_*-\vw_{t-1}\|_2^2 + \frac{1+0.5c_1A_t}{2}\|\vw_*-\vw_t\|_2^2 \notag \\
		&- \; a_{t-1}\sum_{i=1}^K(\elambda_{t-1[i]}-\elambda_{t-2[i]})\E_{\epi}[\innp{\vv(\vw_{t-2};\vx,y), \vw_*-\vw_{t-1}}] \notag \\
        &+ \; a_t\sum_{i=1}^K(\elambda_{t[i]}-\elambda_{t-1[i]})\E_{\epi}[\innp{\vv(\vw_{t-1};\vx,y), \vw_*-\vw_t}] \notag \\
        &- \; \frac{1+0.5c_1A_{t-1}}{4}\|\vw_{t-2}-\vw_{t-1}\|_2^2+\frac{1+0.5c_1A_t}{4}\|\vw_{t-1}-\vw_t\|_2^2 \notag \\
        &- \; (\nu_0+\nu A_{t-2})D_{\phi}(\evlambda_{t-1}, \evlambda_{t-2})-\frac{28\beta^2B\mOPThat a_t}{c_1}.
	\end{align*}
\end{restatable}

For ease of recall, we restate here the quantities defined in \Cref{alg:main} and \Cref{eq:aux-constants}:
\begin{align*}
    C_{3} &:=31\beta\sqrt{B}/c_{1},\\
    C_4&:=27c_1+2163\beta^4B^2/c_1,\\
    C_W&:=\sqrt{6\beta^2+c_M^2B\log^2\Big(\frac{\beta BW}{\eps}\Big)},\\
    C'_W&:=2\sqrt{3}C_W\beta WB,\\
    a_t &= \min\Big\{\Big(1+\frac{c_1}{8C_4}\Big)^{t-1}\frac{1}{4C_4},\max\Big\{\Big(1+\frac{\sqrt{c_1\nu}}{4\sqrt{2}C'_W}\Big)^{t-1}\frac{\sqrt{\nu_0}}{4C_W'},\frac{c_1\nu_0}{(4\sqrt{2}C'_W)^2}t\Big\}\Big\},\\
    A_n&=\sum_{t=0}^na_t.
\end{align*}

We now carry out these steps in detail, and prove our main upper bound on the gap.

\lemgapub*

\begin{proof}
	\textbf{1. Per-iteration bound.}
	Recall
	\[
	a_t\,\Gap(\vw_t,\evlambda_t)
	\;=\;a_t\,L(\vw_t,\evlambda^*)
	\;-\;a_t\,L(\vw_*,\evlambda_t).
	\]
	Combining Lemmas~\ref{lemma:gap-upper-bound1} and~\ref{lemma:gap-upper-bound2} gives, for each \(t\),
	\begin{align*}
		a_t\,\Gap(\vw_t,\evlambda_t)\;\le\;&
		\frac{1+0.5c_1A_{t-1}}{2}\|\vw_*-\vw_{t-1}\|_2^2
		-\frac{1+0.5c_1A_t}{2}\|\vw_*-\vw_t\|_2^2\\
		&+ \; a_{t-1}\sum_{i=1}^K(\elambda_{t-1[i]}-\elambda_{t-2[i]})\E_{\epi}[\innp{\vv(\vw_{t-2};\vx,y), \vw_*-\vw_{t-1}}] \\
        &- \; a_t\sum_{i=1}^K(\elambda_{t[i]}-\elambda_{t-1[i]})\E_{\epi}[\innp{\vv(\vw_{t-1};\vx,y), \vw_*-\vw_t}] \\ 
		&+\;(\nu_0+\nu A_{t-2})\,D_{\phi}(\evlambda_{t-1},\evlambda_{t-2})
		-(\nu_0+\nu A_{t-1})\,D_{\phi}(\evlambda_t,\evlambda_{t-1})\\
        &+\;\frac{1+0.5c_1A_{t-1}}{4}\|\vw_{t-1}-\vw_{t-2}\|_2^2-\frac{1+0.5c_1A_t}{4}\|\vw_t-\vw_{t-1}\|_2^2\\
		&+\;(\nu_0+\nu A_{t-1})\,D_{\phi}(\evlambda^*,\evlambda_{t-1})
		-(\nu_0+\nu A_t)\,D_{\phi}(\evlambda^*,\evlambda_t)\\
		&+\;\frac{28\beta^2B\mOPThat a_t}{c_1}.
	\end{align*}
	
	\textbf{2. Telescoping over \(t=1,\dots,n\).}
	Summing the above inequalities from \(t=1\) to \(n\) causes all intermediate distance and divergence terms to cancel, leaving only the boundary terms at \(t=0\) and \(t=n\), plus the term \(-a_n\sum_{i=1}^K(\elambda_{n[i]}-\elambda_{n-1[i]})\E_{\epi}[\innp{\vv(\vw_{n-1};\vx,y), \vw_*-\vw_n}]\).
    Concretely, one obtains
	\begin{align*}
		\sum_{t=1}^n a_t\,\Gap(\vw_t,\evlambda_t)\;\le\;&
		\frac12\|\vw_*-\vw_0\|_2^2
		-\frac{1+0.5c_1A_n}{2}\|\vw_*-\vw_n\|_2^2\\
		&- \; a_n\sum_{i=1}^K(\elambda_{n[i]}-\elambda_{n-1[i]})\E_{\epi}[\innp{\vv(\vw_{n-1};\vx,y), \vw_*-\vw_n}] \\ 
		&\;-\;(\nu_0+\nu A_{n-1})\,D_{\phi}(\evlambda_n,\evlambda_{n-1})
		+\nu_0\,D_{\phi}(\evlambda^*,\evlambda_0)
		-(\nu_0+\nu A_n)\,D_{\phi}(\evlambda^*,\evlambda_n)\\
		&\; -\frac{1+0.5c_1A_n}{4}\|\vw_n-\vw_{n-1}\|_2^2+\frac{28\beta^2B\,\OPThat\,A_n}{c_1}.
	\end{align*}
	
	\textbf{3. Bounding the final residual.}
	It remains only to absorb the term \(-a_n\sum_{i=1}^K(\elambda_{n[i]}-\elambda_{n-1[i]})\cdot\allowbreak\E_{\epi}[\innp{\vv(\vw_{n-1};\vx,y), \vw_*-\vw_n}]\).
	By Young's inequality and our choice of \(\alpha_1\) and step size \(a_t\) in \Cref{alg:main},
    we can show that (see the derivation of \Cref{eq:St-term2-bound2} in the proof of \Cref{lemma:gap-upper-bound2} for the complete argument)
	\begin{align}\label{eq:final-residual}
		&-a_n\sum_{i=1}^K(\elambda_{n[i]}-\elambda_{n-1[i]})\E_{\epi}[\innp{\vv(\vw_{n-1};\vx,y), \vw_*-\vw_n}]\notag \\
		\leq \;& 4\sqrt{3}C_W\beta WB a_n(\frac{\alpha_1}{2}2D_{\phi}(\evlambda_n,\evlambda_{n-1})+\frac{1}{2\alpha_1}\|\vw_*-\vw_n\|_2^2) \notag\\
		\leq \;& (\nu_0+\nu A_{n-1})D_{\phi}(\evlambda_n, \evlambda_{n-1})+\frac{1+0.5c_1A_n}{4}\|\vw_*-\vw_{n}\|_2^2, 
	\end{align}
	substituting this back into the telescoped sum completes the proof.
\end{proof}

\subsection{Proof of Lemma \ref{lemma:gap-upper-bound1}}
\label{sec:gap-upper-bound1}
We now prove the upper bound on \(a_tL(\vw_{t}, \evlambda^{*})\).
\lemGapUpperOne*
\begin{proof}
	We begin by introducing the quantity
	\[
	h(\evlambda)\;:=\;a_t\,L(\vw_t,\evlambda)\;-\;(\nu_0+\nu A_{t-1})\,D_{\phi}(\evlambda,\evlambda_{t-1}).
	\]
	Then observe that
	\[
	a_tL(\vw_t,\evlambda^{*})
	= h(\evlambda^{*})
	+(\nu_0+\nu A_{t-1})\,D_{\phi}(\evlambda^{*},\evlambda_{t-1}),
	\]
	since we have simply added and subtracted the same Bregman term. 
	
	By construction, $h(\evlambda)$ is the sum of a linear function in $\evlambda$ and a $-(\nu_0+\nu A_{t-1})$-strongly concave term, so $h$ itself is strongly concave. Hence its maximizer is exactly
	\[
	\evlambda_t
	\;=\;\argmax_{\evlambda\in\Delta_K}h(\evlambda).
	\]
	We now apply the standard Bregman-divergence expansion around this maximizer:
	\[
	h(\evlambda^{*})
	= h(\evlambda_t)
	+\bigl\langle\nabla h(\evlambda_t),\,\evlambda^{*}-\evlambda_t\bigr\rangle
	+D_{h}(\evlambda^{*},\evlambda_t).
	\]
	Since $\evlambda_t$ maximizes $h$, by \Cref{fact:firstOrderNecessary}, the linear term is nonpositive, and we have
	\begin{align*}
		h(\evlambda^{*})
		&\; \le h(\evlambda_t) +D_{h}(\evlambda^*,\evlambda_t) \\ \notag
		&\; = a_tL(\vw_t,\evlambda_t)
		-(\nu_0+\nu A_{t-1})\,D_{\phi}(\evlambda_t,\evlambda_{t-1})
		-(\nu_0+\nu A_t)\,D_{\phi}(\evlambda^{*},\evlambda_t),
	\end{align*}
	where in the last step we used Fact~\ref{fact:bregman-linear-blind} to replace $D_h(\evlambda^{*},\evlambda_t)$ by the appropriate shift in the Bregman divergence $D_{\phi}$. 
	
	Putting these pieces together gives
	\begin{align*}
		a_tL(\vw_t, \evlambda^{*})
		\;\le\;&
		a_tL(\vw_t, \evlambda_t)
		\;-\;(\nu_0+\nu A_{t-1})\,D_{\phi}(\evlambda_t,\evlambda_{t-1})
		\;-\;(\nu_0+\nu A_t)\,D_{\phi}(\evlambda^{*},\evlambda_t)\\
		&\;+\;(\nu_0+\nu A_{t-1})\,D_{\phi}(\evlambda^{*},\evlambda_{t-1}).
	\end{align*}
	as claimed. Notice that the last two terms telescope when summed over \(t\), and the negative term involving \(D_{\phi}(\evlambda_t,\evlambda_{t-1})\) can be used to absorb error contributions in the overall gap analysis.
\end{proof}

\subsection{Proof of Lemma \ref{lemma:gap-upper-bound2}}\label{sec:gap-upper-bound2}
In this section, we prove the lower bound on \(a_tL(\vw_*, \evlambda_{t})\). This is the most challenging part of the convergence analysis that motivates and analyzes the extrapolation of the dual variable (\Cref{line:extrapolation}),  determines the convergence rate (\Cref{lemma:convergence-rate}), and proves the key linearization lemma (\Cref{lemma:main-auxiliary-main-body}) that underlies the definition of the surrogate gradient in \Cref{line:surrogate_gradient}. Note that in the analysis below we assume the labels are bounded by $M$ when we expand the definition of \(\vv(\vw;\vx,y)\), which is w.l.o.g.\ due to \Cref{fact:truncation} and can be ensured by the appropriate pre-processing of samples, as discussed in \Cref{sec:prelims}.

\lemGapUpperTwo*

\begin{proof} Since \(L(\vw_*, \evlambda_t)=\sum_{i=1}^K\elambda_{t[i]}\E_{\epi}[(\sigma(\vw_*\cdot \vx)-y)^2]-\nu d_f(\evlambda_t, \evlambda_0)\), we first expand the square 
\begin{align}\label{eq:diff-of-squares}
		(\sigma(\vw_*\cdot \vx)-y)^2
		= \; & ((\sigma(\vw_{t-1}\cdot \vx)-y)+(\sigma(\vw_*\cdot \vx)-\sigma(\vw_{t-1}\cdot \vx)))^2\notag\\
		= \; & (\sigma(\vw_{t-1}\cdot \vx)-y)^2+(\sigma(\vw_*\cdot \vx)-\sigma(\vw_{t-1}\cdot \vx))^2\notag \\
		&+ \;2(\sigma(\vw_{t-1}\cdot \vx)-y)(\sigma(\vw_*\cdot \vx)-\sigma(\vw_{t-1}\cdot \vx))
\end{align}
    and obtain
\begin{align*}
		L(\vw_*, \evlambda_t)
		= \; & \sum_{i=1}^K{\elambda}_{t[i]}\E_{\epi}[2(\sigma(\vw_{t-1}\cdot \vx)-y)(\sigma(\vw_*\cdot \vx)-\sigma(\vw_{t-1}\cdot \vx))] -\nu d_f(\evlambda_t, \evlambda_0)\notag \\
&+ \; \sum_{i=1}^K{\elambda}_{t[i]}\E_{\epi}[(\sigma(\vw_{t-1}\cdot \vx)-y)^2]+\sum_{i=1}^K{\elambda}_{t[i]}\E_{\epi}[(\sigma(\vw_*\cdot \vx)-\sigma(\vw_{t-1}\cdot \vx))^2].\notag
	\end{align*}    

    We use \Cref{lemma:sharpness-empirical} to bound below the last term, use \Cref{lemma:main-auxiliary-main-body} to bound below the first term, and get
    \begin{align}
        L(\vw_*, \evlambda_t)
		\ge \; & \sum_{i=1}^K{\elambda}_{t[i]}\E_{\epi}[\innp{\vv(\vw_{t-1};\vx,y), \vw_*-\vw_{t-1}}]  -\nu d_f(\evlambda_t, \evlambda_0) - \frac{24\beta^2B\mOPThat}{c_1} \notag \\
&+ \; \sum_{i=1}^K{\elambda}_{t[i]}\E_{\epi}[(\sigma(\vw_{t-1}\cdot \vx)-y)^2] + c_1\norm{\vw_*-\vw_{t-1}}_2^2  -\frac{c_1}{4}\norm{\vw_*-\vw_{t-1}}_2^2. \notag
    \end{align}
	
Recall that the extrapolated group weights at \((t-1)\)-th iteration are \(\bar{\vlambda}_{t-1}:=\evlambda_{t-1}+\frac{a_{t-1}}{a_t}(\evlambda_{t-1}-\evlambda_{t-2})\), thus 
    \begin{align}
        L(\vw_*, \evlambda_t)
		\ge \; & 
        \sum_{i=1}^K{(\elambda}_{t[i]} -\bar\lambda_{t-1[i]})\E_{\epi}[\innp{\vv(\vw_{t-1};\vx,y), \vw_*-\vw_{t-1}}]  -\nu d_f(\evlambda_t, \evlambda_0) \notag \\
&+ \; \sum_{i=1}^K\elambda_{t[i]}  \E_{\epi}[(\sigma(\vw_{t-1}\cdot \vx)-y)^2] + \frac34 c_1\norm{\vw_*-\vw_{t-1}}_2^2 - \frac{24\beta^2B\mOPThat}{c_1} \notag \\
        &  + \sum_{i=1}^K\bar\lambda_{t-1[i]}\E_{\epi}[\innp{\vv(\vw_{t-1};\vx,y), \vw_*-\vw_{t-1}}].\label{eq:lower-bound}
    \end{align}

Now we observe that the last term in \eqref{eq:lower-bound} is a linearization term defining the primal update. This allows us to carry out a similar proof as for the primal gap upper bound in \Cref{lemma:gap-upper-bound1}, and avoid an implicit dependency issue by choosing \(\vv(\vw_{t-1};\vx,y)\) instead of \(\vv(\vw_{t};\vx,y)\) to define \(\psi(\vw)\), since \(\vv(\vw_t;\vx,y)\) depends on \(\vw_t\). Let \(\psi(\vw)=a_t\sum_{i=1}^K\bar{\lambda}_{t-1[i]}\E_{\epi}[\innp{\vv(\vw_{t-1};\vx,y), \vw}]+\frac{1+0.5c_1A_t}{2}\|\vw-\vw_{t-1}\|_2^2\), then \(\vw_t=\argmin_{\vw\in\cB(W)}\psi(\vw)\) is the update rule for \(\vw_t\). Since \(\psi(\vw)\) is \((1+0.5c_1A_{t-1})\)-strongly convex in \(\vw\), we have 
	\begin{align}
		&a_t\sum_{i=1}^K\bar{\lambda}_{t-1[i]}\E_{\epi}[\innp{\vv(\vw_{t-1};\vx,y), \vw_*-\vw_{t-1}}] \notag \\ 
        \geq\;& a_t\sum_{i=1}^K\bar{\lambda}_{t-1[i]}\E_{\epi}[\innp{\vv(\vw_{t-1};\vx,y), \vw_t-\vw_{t-1}}] \notag \\
        &+ \;\frac{1+0.5c_1A_t}{2}\|\vw_t-\vw_{t-1}\|_2^2-\frac{1+0.5c_1A_{t-1}}{2}\|\vw_*-\vw_{t-1}\|_2^2 \notag \\
		&+ \;\frac{1+0.5c_1A_t}{2}\|\vw_*-\vw_t\|_2^2-\frac{c_1a_t}{4}\|\vw_*-\vw_{t-1}\|_2^2.\label{eq:St-term1}
	\end{align}

Multiplying  both sides by \(a_t\) and plugging \Cref{eq:St-term1} into the last term in RHS in \Cref{eq:lower-bound}, we get
   \begin{align}
        a_t L(\vw_*, \evlambda_t)
		\ge \; & a_t \sum_{i=1}^K{(\elambda}_{t[i]} -\bar\lambda_{t-1[i]})\E_{\epi}[\innp{\vv(\vw_{t-1};\vx,y), \vw_*-\vw_{t-1}}]  -\nu a_t  d_f(\evlambda_t, \evlambda_0) - a_t \frac{24\beta^2B\mOPThat}{c_1} \notag \\
		&+ \; a_t \sum_{i=1}^K\elambda_{t[i]}  \E_{\epi}[(\sigma(\vw_{t-1}\cdot \vx)-y)^2] + \frac34 c_1a_t\norm{\vw_*-\vw_{t-1}}_2^2 \notag \\
        &   + a_t\sum_{i=1}^K\bar{\lambda}_{t-1[i]}\E_{\epi}[\innp{\vv(\vw_{t-1};\vx,y), \vw_t-\vw_{t-1}}] \notag \\
        &+ \;\frac{1+0.5c_1A_t}{2}\|\vw_t-\vw_{t-1}\|_2^2-\frac{1+0.5c_1A_{t-1}}{2}\|\vw_*-\vw_{t-1}\|_2^2 \notag \\
		&+ \;\frac{1+0.5c_1A_t}{2}\|\vw_*-\vw_t\|_2^2-\frac{c_1a_t}{4}\|\vw_*-\vw_{t-1}\|_2^2.\notag
    \end{align}
Recalling that, by definition, $L(\vw_t,\evlambda_t) = \sum_{i=1}^K\elambda_{t[i]}  \E_{\epi}[(\sigma(\vw_{t}\cdot \vx)-y)^2]$, we further have:
\begin{align}\label{eq:lb-three-terms}
         a_t L(\vw_*, \evlambda_t)
		\ge \; & a_tL(\vw_t,\evlambda_t) + a_t \sum_{i=1}^K{(\elambda}_{t[i]} -\bar\lambda_{t-1[i]})\E_{\epi}[\innp{\vv(\vw_{t-1};\vx,y), \vw_*-\vw_{t-1}}]    \notag \\
		&+ \; a_t \sum_{i=1}^K\elambda_{t[i]}  \E_{\epi}[(\sigma(\vw_{t-1}\cdot \vx)-y)^2-(\sigma(\vw_t\cdot \vx)-y)^2]  \notag \\
        &   + a_t\sum_{i=1}^K\bar{\lambda}_{t-1[i]}\E_{\epi}[\innp{\vv(\vw_{t-1};\vx,y), \vw_t-\vw_{t-1}}] \notag \\
        &+ \;\frac{1+0.5c_1A_t}{2}\|\vw_t-\vw_{t-1}\|_2^2-\frac{1+0.5c_1A_{t-1}}{2}\|\vw_*-\vw_{t-1}\|_2^2 \notag \\
		&+ \;\frac{1+0.5c_1A_t}{2}\|\vw_*-\vw_t\|_2^2 + \frac{c_1a_t}{2}\norm{\vw_*-\vw_{t-1}}_2^2- a_t \frac{24\beta^2B\mOPThat}{c_1}.
    \end{align}

 Now we need to deal with the second, the third, and the fourth term in RHS of \Cref{eq:lb-three-terms}. For the second term, we carry out the following decomposition 
	\begin{align}\label{eq:St-term2}
		&a_t \sum_{i=1}^K{(\elambda}_{t[i]} -\bar\lambda_{t-1[i]})\E_{\epi}[\innp{\vv(\vw_{t-1};\vx,y), \vw_*-\vw_{t-1}}]   \notag \\ 
        =\;& a_t \sum_{i=1}^K{(\elambda}_{t[i]} -\elambda_{t-1[i]})\E_{\epi}[\innp{\vv(\vw_{t-1};\vx,y), \vw_*-\vw_{t-1}}]  \notag \\
         &-a_{t-1} \sum_{i=1}^K{(\elambda}_{t-1[i]} -\elambda_{t-2[i]})\E_{\epi}[\innp{\vv(\vw_{t-2};\vx,y), \vw_*-\vw_{t-2}}] \notag \\
        &+a_{t-1}\sum_{i=1}^K(\elambda_{t-1[i]}-{\elambda}_{t-2[i]})\E_{\epi}[\innp{\vv(\vw_{t-2};\vx,y)-\vv(\vw_{t-1};\vx,y), \vw_*-\vw_{t-1}}] \notag \\
	       &+a_{t-1}\sum_{i=1}^K(\elambda_{t-1[i]}-{\elambda}_{t-2[i]})\E_{\epi}[\innp{\vv(\vw_{t-2};\vx,y), \vw_{t-1}-\vw_{t-2}}], 
	\end{align}
	where the last equality follows the definition of \(\Bar{\vlambda}_{t-1}\). Since the first two terms telescope, we only need to focus on bounding the last two terms in \Cref{eq:St-term2}.  
For the second to last term in \Cref{eq:St-term2}, we have 
    \begin{align}\label{eq:St-term2-bound1}
		&-a_{t-1}\sum_{i=1}^K(\elambda_{t-1[i]}-{\elambda}_{t-2[i]})\E_{\epi}[\innp{\vv(\vw_{t-2};\vx,y)-\vv(\vw_{t-1};\vx,y), \vw_*-\vw_{t-1}}]\notag \\
        \leq\;& 2a_{t-1}\max_{i\in[K]}|\E_{\epi}[2\beta(\sigma(\vw_{t-2}\cdot \vx)-\sigma(\vw_{t-1}\cdot \vx))(\vw_*\cdot \vx-\vw_{t-1}\cdot \vx)]| \notag \\
		\leq\;& 4a_{t-1}\beta \sqrt{6\beta^2B\|\vw_{t-2}-\vw_{t-1}\|_2^2 \cdot 6B\|\vw_*-\vw_{t-1}\|_2^2} \notag \\
        \leq\;& \frac{24\cdot 36\beta^4 B^2a_{t-1}}{c_1}\|\vw_{t-2}-\vw_{t-1}\|_2^2+\frac{c_1a_t}{6}\|\vw_*-\vw_{t-1}\|_2^2,
	\end{align}
    where in the first inequality, we applied the definition of \(\vv(\vw;\vx,y)\) and H\"older's inequality with the fact that \(\sum_{i=1}^K|\elambda_{t-1[i]}-{\elambda}_{t-2[i]}|\le 2\). The second inequality follows from \Cref{lemma:sharpness-empirical} and Cauchy-Schwarz inequality, the third follows from Young's inequality and the fact that \(a_{t-1}\le a_t\) for all \(t\ge 1\).\\
    For the last term in \Cref{eq:St-term2}, by Cauchy-Schwarz inequality and \Cref{lemma:sharpness-empirical}, we have 
	  \begin{align}\label{eq:aux_wt-2}
		&-a_{t-1}\sum_{i=1}^K(\elambda_{t-1[i]}-{\elambda}_{t-2[i]})\E_{\epi}[\innp{\vv(\vw_{t-2};\vx,y), \vw_{t-1}-\vw_{t-2}}]\notag \\
        \leq\;& 2\beta a_{t-1}\max_{i\in[K]}|\E_{\epi}[(\sigma(\vw_{t-2}\cdot \vx)-y)(\vw_{t-1}\cdot \vx-\vw_{t-2}\cdot \vx)]|\sum_{i=1}^K|\elambda_{t-1[i]}-{\elambda}_{t-2[i]}| \notag \\
		\leq\;& 2\beta a_{t-1} \max_{i\in[K]}\sqrt{\E_{\epi}[(\sigma(\vw_{t-2}\cdot \vx)-y)^2]\cdot 6B\|\vw_{t-1}-\vw_{t-2}\|_2^2}\sum_{i=1}^K|\elambda_{t-1[i]}-{\elambda}_{t-2[i]}|.
	\end{align}

	Moreover, following again from \Cref{lemma:sharpness-empirical} and \Cref{fact:truncation}, for all groups \(i\in[K]\) and any \(\vw\in\cB(3\norm{\vw_*}_2)\),
	\begin{align*}
		\E_{\epi}[(\sigma(\vw\cdot \vx)-y)^2] \le\;& 2(\E_{\epi}[(\sigma(\vw\cdot \vx))^2]+ \E_{\epi}[y^2])\notag \\
		=\;& 2\big(6\beta^2W^2B+C_M^2W^2B^2\log^2(\frac{\beta BW}{\eps})\big) \notag.
	\end{align*}
    Plugging this bound into \eqref{eq:aux_wt-2} above, and recalling that  \(C_W=\sqrt{6\beta^2+C_M^2B\log^2(\frac{\beta BW}{\eps})}\), we get
    \begin{align}\label{eq:St-term2-bound2}
		&-a_{t-1}\sum_{i=1}^K(\elambda_{t-1[i]}-{\elambda}_{t-2[i]})\E_{\epi}[\innp{\vv(\vw_{t-2};\vx,y), \vw_{t-1}-\vw_{t-2}}] \notag \\
        \leq\;& 2a_{t-1}\beta\sqrt{2W^2B 
        \big(6\beta^2+C_M^2B\log^2(\frac{\beta BW}{\eps})\big)}\sqrt{6B}\|\vw_{t-1}-\vw_{t-2}\|_2\sum_{i=1}^K|\elambda_{t-1[i]}-\elambda_{t-2[i]}| \notag \\
        \leq\;& 4\sqrt{3}a_{t-1}C_W\beta WB\|\vw_{t-1}-\vw_{t-2}\|_2\sum_{i=1}^K|\elambda_{t-1[i]}-\elambda_{t-2[i]}|\notag \\
        \leq\;&4\sqrt{3}C_W\beta WB a_{t-1}(\frac{\alpha_1}{2}2D_{\phi}(\evlambda_{t-1},\evlambda_{t-2})+\frac{1}{2\alpha_1}\|\vw_{t-1}-\vw_{t-2}\|_2^2),
	\end{align}
    where we apply Young's inequality (\Cref{fact:young}) and \Cref{claim:ub-TV} (see below) for the last inequality.

    To bound the third term on the RHS of \Cref{eq:lb-three-terms}, we first use a similar method as in \Cref{eq:diff-of-squares} to expand the difference of the square terms in the expectation to get:
    \begin{align}\label{eq:St-term3}
		&\;a_t \sum_{i=1}^K\elambda_{t[i]}  \E_{\epi}[(\sigma(\vw_{t}\cdot \vx)-y)^2-(\sigma(\vw_{t-1}\cdot \vx)-y)^2] \notag \\ 
        \leq\;& a_t\max_{i\in [K]}\E_{\epi}[(\sigma(\vw_t\cdot \vx)-\sigma(\vw_{t-1}\cdot \vx))^2] \notag \\
        &+\;a_t\max_{i\in [K]}|\E_{\epi}[2(\sigma(\vw_{t-1}\cdot \vx)-y)(\sigma(\vw_t\cdot \vx)-\sigma(\vw_{t-1}\cdot \vx))]| \notag \\
        \leq\;& 6a_t\beta^2B\|\vw_t-\vw_{t-1}\|_2^2+2a_t\beta\max_{i\in[K]}|\E_{\epi}[(\sigma(\vw_{t-1}\cdot \vx)-y)(\vw_t\cdot\vx-\vw_{t-1}\cdot\vx))]|,
    \end{align}
    where the second inequality follows from \Cref{lemma:sharpness-empirical} and Lipschitzness of \(\sigma(\cdot)\). We split the second term in RHS above and apply the definition of \(\mOPThat\) together with the Cauchy-Schwarz inequality, respectively, for all groups \(i\in[K]\),
    \begin{align*}
		&|\E_{\epi}[(\sigma(\vw_{t-1}\cdot \vx)-y)(\vw_t\cdot\vx-\vw_{t-1}\cdot\vx))]| \notag \\
		\leq\;& |\E_{\epi}[(\sigma(\vw_*\cdot \vx)-y)(\vw_t\cdot\vx-\vw_{t-1}\cdot\vx)]|+|\E_{\epi}[(\sigma(\vw_{t-1}\cdot \vx)-\sigma(\vw_*\cdot \vx))(\vw_t\cdot\vx-\vw_{t-1}\cdot\vx))]| \notag \\
		\leq\;& \sqrt{\mOPThat \cdot 6B\|\vw_t-\vw_{t-1}\|_2^2}+\sqrt{6\beta^2B\|\vw_{t-1}-\vw_*\|_2^2\cdot 6B\|\vw_t-\vw_{t-1}\|_2^2}. \notag 
    \end{align*}
    Combining with \Cref{eq:St-term3} and applying Young's inequality (\Cref{fact:young}), we have
    \begin{align}\label{eq:St-term3-final-bound}
		&a_t \sum_{i=1}^K\elambda_{t[i]}  \E_{\epi}[(\sigma(\vw_{t}\cdot \vx)-y)^2-(\sigma(\vw_{t-1}\cdot \vx)-y)^2] \notag \\ 
         \leq\;& 6\beta^2B a_t\|\vw_t-\vw_{t-1}\|_2^2+12a_t\beta^2B\Big(\frac{36\beta^2B}{2c_1}\|\vw_{t}-\vw_{t-1}\|_2^2+\frac{c_1}{2\cdot 36\beta^2B}\|\vw_{t-1}-\vw_*\|_2^2\Big) \notag \\
        &+\; 2a_t\Big(\frac{c_1}{2\beta^2B}6\beta^2B\|\vw_t-\vw_{t-1}\|_2^2+\frac{\beta^2B}{2c_1}\mOPThat\Big).
    \end{align}
    Now we bound the fourth term in \Cref{eq:lb-three-terms}. Noting that \(\sum_{i=1}^K|\bar{\lambda}_{t-1[i]}|
    \le 3\) since \(a_{t-1}\le a_t\) for any \(t\ge 0\), we have 
    \begin{align}\label{eq:St-term4}
		&-a_t\sum_{i=1}^K\bar{\lambda}_{t-1[i]}\E_{\epi}[\innp{\vv(\vw_{t-1};\vx,y), \vw_t-\vw_{t-1}}] \notag \\ 
        \leq\;& 6\beta a_t \max_{i\in[K]}|\E_{\epi}[(\sigma(\vw_{t-1}\cdot \vx)-\sigma(\vw_*\cdot \vx))(\vw_t\cdot \vx-\vw_{t-1}\cdot \vx)| \notag \\
        &+6\beta a_t\max_{i\in [K]}|\E_{\epi}[(\sigma(\vw_*\cdot \vx)-y)(\vw_t\cdot \vx-\vw_{t-1}\cdot \vx)| \notag \\
        \leq\;& 6\beta a_t\sqrt{6\beta^2B\|\vw_{t-1}-\vw_*\|_2^2\cdot 6B\|\vw_t-\vw_{t-1}\|_2^2}+6\beta a_t\sqrt{\mOPThat \cdot 6B\|\vw_t-\vw_{t-1}\|_2^2} \notag \\
        \leq\;& 36\beta^2Ba_t\Big(\frac{c_1}{2\cdot 108 \beta^2B}\|\vw_{t-1}-\vw_*\|_2^2+\frac{108\beta^2B}{2c_1}\|\vw_t-\vw_{t-1}\|_2^2\Big) \notag \\
        &+ \; 6\beta a_t\Big(\frac{\beta B}{2c_1}\mOPThat +\frac{6Bc_1}{2\beta B}\|\vw_t-\vw_{t-1}\|_2^2\Big).    
    \end{align}
    We again here used the definition of \(\vv(\vw;\vx,y)\), a similar method as in \Cref{eq:diff-of-squares} to split terms, and used H\"older's inequality to get the first inequality, then applied Cauchy-Schwarz inequality in the second inequality, and used Young's inequality (\Cref{fact:young}) in the third. 
    
    It remains to choose the appropriate \(\alpha_1\) and argue that the step sizes \(a_t\) satisfy \(4\sqrt{3}C_W\beta BWa_{t-1}\leq \nu_0+\nu A_{t-2}\), \(\frac{2\sqrt{3}C_W\beta WBa_{t-1}}{\alpha_1}+\frac{36\cdot24\beta^4B^2a_{t-1}}{c_1}\leq \frac{1+0.5c_1A_{t-1}}{4}\), and \((24c_1+6\beta^2B+\frac{36\cdot 60\beta^4B^2}{c_1})a_t\leq \frac{1+0.5c_1A_t}{4}\),  where the first inequality is to construct a term to telescope with \(-(\nu_0+\nu A_{t-1})D_{\phi}(\evlambda_t,\evlambda_{t-1})\) in \Cref{lemma:gap-upper-bound1}, the second and the third inequality are to construct telescoping terms \(\|\vw_t-\vw_{t-1}\|_2^2\) and \(\|\vw_{t-1}-\vw_{t-2}\|_2^2\).
    We use \Cref{lemma:convergence-rate} below to find and justify an appropriate choice of \(\alpha_1\) and \(a_t\). Combining \Cref{eq:St-term2}, \Cref{eq:St-term2-bound1}, \Cref{eq:St-term2-bound2}, \Cref{eq:St-term3-final-bound} and \Cref{eq:St-term4} together and then substituting them back into \Cref{eq:lb-three-terms}, we get
	\begin{align*}
		a_tL(\vw_*,\evlambda_t)\ge \; & a_tL(\vw_t,\evlambda_t) - \; \frac{1+0.5c_1A_{t-1}}{2}\|\vw_*-\vw_{t-1}\|_2^2 + \frac{1+0.5c_1A_t}{2}\|\vw_*-\vw_t\|_2^2 \notag \\
		&- \; a_{t-1}\sum_{i=1}^K(\elambda_{t-1[i]}-\elambda_{t-2[i]})\E_{\epi}[\innp{\vv(\vw_{t-2};\vx,y), \vw_*-\vw_{t-1}}] \notag \\
        &+ \; a_t\sum_{i=1}^K(\elambda_{t[i]}-\elambda_{t-1[i]})\E_{\epi}[\innp{\vv(\vw_{t-1};\vx,y), \vw_*-\vw_t}] \notag \\
        &- \; \frac{1+0.5c_1A_{t-1}}{4}\|\vw_{t-2}-\vw_{t-1}\|_2^2+\frac{1+0.5c_1A_t}{4}\|\vw_{t-1}-\vw_t\|_2^2 \notag \\
        &- \; (\nu_0+\nu A_{t-2})D_{\phi}(\evlambda_{t-1}, \evlambda_{t-2})-\frac{28\beta^2B\mOPThat a_t}{c_1},
	\end{align*}
	which completes our proof. 
\end{proof}

\lemubaux*

\begin{proof}
	We split the expectation on the left hand side according to whether $\sigma(\vw_t\!\cdot\vx)-y\ge0$.  Define 
    \(\cG = \{(\vx,y)\mid \sigma(\vw_t\!\cdot\vx)-y\ge0\}\). Then
	\begin{align*}
		&\; \LHS_{i} :=  \E_{\epi}\bigl[2(\sigma(\vw_t\!\cdot\vx)-y)\,(\sigma(\vw_*\!\cdot\vx)-\sigma(\vw_t\!\cdot\vx))\bigr]\\
		=&\;\E_{\epi}\Bigl[2(\sigma(\vw_t\!\cdot\vx)-y)\,(\sigma(\vw_*\!\cdot\vx)-\sigma(\vw_t\!\cdot\vx))\Ind_{\cG}\Bigr]\\
	&\;+
		\E_{\epi}\Bigl[2(\sigma(\vw_t\!\cdot\vx)-y)\,(\sigma(\vw_*\!\cdot\vx)-\sigma(\vw_t\!\cdot\vx))\Ind_{\cG^c}\Bigr].
	\end{align*}
	On $\cG$, we apply convexity of $\sigma$ at $\vw_t\cdot\vx$:
	\[
	\sigma(\vw_*\!\cdot\vx)-\sigma(\vw_t\!\cdot\vx)
	\;\ge\;
	\sigma'(\vw_t\!\cdot\vx)\,(\vw_*\!\cdot\vx-\vw_t\!\cdot\vx),
	\]
	and on $\cG^c$ we similarly use convexity at $\vw_*\cdot\vx$ after flipping signs, where \(\sigma'\) is the subderivative of \(\sigma\) which must exist due to convexity and Lipschitzness. Hence
	\begin{align*}
		& \; \LHS_{i} = \E_{\epi}[2(\sigma(\vw_t\!\cdot\vx)-y)(\sigma(\vw_*\cdot \vx)-\sigma(\vw_t\cdot \vx))]  
		\\
		\ge & \; \E_{\epi}\bigl[2(\sigma(\vw_t\!\cdot\vx)-y)\,\sigma'(\vw_t\!\cdot\vx)\,(\vw_*\!\cdot\vx-\vw_t\!\cdot\vx)\Ind_{\cG}\bigr]\\
		&\;+
		\E_{\epi}\bigl[2(\sigma(\vw_t\!\cdot\vx)-y)\,\sigma'(\vw_*\!\cdot\vx)\,(\vw_*\!\cdot\vx-\vw_t\!\cdot\vx)\Ind_{\cG^c}\bigr].
	\end{align*}
	Now recall $\vv(\vw_t;\vx,y)=2\beta (\sigma(\vw_t\!\cdot\vx)-y)\vx$.  Adding and subtracting $\beta$ in the derivatives, we rewrite the left hand side:
	\begin{align}\label{eq:break-term2-St}
		& \E_{\epi}[2(\sigma(\vw_t\!\cdot\vx)-y)(\sigma(\vw_*\cdot \vx)-\sigma(\vw_t\cdot \vx))]  \notag\\
		\ge  \; & \E_{\epi}[\innp{\vv(\vw_t;\vx,y), \vw_*-\vw_t}] \notag \\
		&\;+2\E_{\epi}[(\sigma(\vw_t\!\cdot\vx)-y)(\sigma'(\vw_t\cdot \vx)-\beta)(\vw_*\cdot \vx-\vw_t\cdot \vx)\Ind_{\cG}] \notag \\
		&\;+ 2\E_{\epi}[(\sigma(\vw_t\!\cdot\vx)-y)(\sigma'(\vw_*\cdot \vx)-\beta)(\vw_*\cdot \vx-\vw_t\cdot \vx)\Ind_{\cG^{c}}].
	\end{align}
	where we consider the last two terms involving factors $(\sigma'(\cdot)-\beta)$ to be the “error terms” .  Since $\sigma$ is nondecreasing and $\beta$-Lipschitz, $-\beta\le\sigma'(\cdot)-\beta\le0$.  We show that both error terms are bounded in absolute value by
	$\beta\;\E_{\epi}\bigl[\,|\sigma(\vw_*\!\cdot\vx)-y|\,|\vw_*\!\cdot\vx-\vw_t\!\cdot\vx|\bigr]$. On the event $\cG$,
	\begin{align}\label{eq:break-event-g}
		&\; (\sigma(\vw_t\!\cdot\vx)-y)(\sigma'(\vw_t\cdot \vx)-\beta)(\vw_*\cdot \vx-\vw_t\cdot \vx) \notag \\
		=\; &(\sigma(\vw_*\!\cdot\vx)-y)(\sigma'(\vw_t\cdot \vx)-\beta)(\vw_*\cdot \vx-\vw_t\cdot \vx) \notag \\
		&\;+ (\sigma(\vw_t\!\cdot\vx)-\sigma(\vw_*\!\cdot\vx))(\sigma'(\vw_t\cdot \vx)-\beta)(\vw_*\cdot \vx-\vw_t\cdot \vx) \notag \\
		\ge \; & (\sigma(\vw_*\!\cdot\vx)-y)(\sigma'(\vw_t\cdot \vx)-\beta)(\vw_*\cdot \vx-\vw_t\cdot \vx).
	\end{align}
	The last inequality holds because \((\sigma(\vw_t\!\cdot\vx)-\sigma(\vw_*\!\cdot\vx))(\vw_*\cdot \vx-\vw_t\cdot \vx) 
	\le 0\)  by monotonicity of $\sigma$.
	Similarly, we have 
	\begin{align}\label{eq:break-event-g_c}
		&\; (\sigma(\vw_t\!\cdot\vx)-y)(\sigma'(\vw_*\cdot \vx)-\beta)(\vw_*\cdot \vx-\vw_t\cdot \vx) \notag \\
		\ge \; & (\sigma(\vw_*\!\cdot\vx)-y)(\sigma'(\vw_*\cdot \vx)-\beta)(\vw_*\cdot \vx-\vw_t\cdot \vx).
	\end{align}
    Plugging \Cref{eq:break-event-g} and \Cref{eq:break-event-g_c} into \Cref{eq:break-term2-St}, we have
    	\begin{align}\label{eq:break-term2-St-cont1}
		& \E_{\epi}[2(\sigma(\vw_t\!\cdot\vx)-y)(\sigma(\vw_*\cdot \vx)-\sigma(\vw_t\cdot \vx))]  \notag\\
        \ge  \; & \E_{\epi}[\innp{\vv(\vw_t;\vx,y), \vw_*-\vw_t}] \notag \\
		&\;+2\E_{\epi}[(\sigma(\vw_*\!\cdot\vx)-y)(\sigma'(\vw_t\cdot \vx)-\beta)(\vw_*\cdot \vx-\vw_t\cdot \vx)\Ind_{\cG}] \notag \\
		&\;+ 2\E_{\epi}[(\sigma(\vw_*\!\cdot\vx)-y)(\sigma'(\vw_*\cdot \vx)-\beta)(\vw_*\cdot \vx-\vw_t\cdot \vx)\Ind_{\cG^{c}}].
	\end{align}
	Taking absolute values and using $-\beta \leq \sigma'(\cdot) - \beta \le 0$, for any \(\vw\in\{\vw_t, \vw_*\}\), we have 
	\begin{align}\label{eq:bounding-two-terms}
		&-(\sigma(\vw_*\!\cdot\vx)-y)(\sigma'(\vw\cdot \vx)-\beta)(\vw_*\cdot \vx-\vw_t\cdot \vx) \notag \\
		\le\; & |(\sigma(\vw_*\!\cdot\vx)-y)(\sigma'(\vw\cdot \vx)-\beta)(\vw_*\cdot \vx-\vw_t\cdot \vx)| \notag \\
		\le \; & \beta|(\sigma(\vw_*\!\cdot\vx)-y)(\vw_*\cdot \vx-\vw_t\cdot \vx)|,
	\end{align}
	Therefore, plugging \Cref{eq:bounding-two-terms} into \Cref{eq:break-term2-St-cont1},
	\begin{align*}
		& \E_{\epi}[2(\sigma(\vw_t\!\cdot\vx)-y)(\sigma(\vw_*\cdot \vx)-\sigma(\vw_t\cdot \vx))] \notag \\
		\ge \;  & \E_{\epi}[\innp{\vv(\vw_t;\vx,y), \vw_*-\vw_t}] \\
		&\;-  2\E_{\epi}[\abs{(\sigma(\vw_*\!\cdot\vx)-y)(\sigma'(\vw_t\cdot \vx)-\beta)(\vw_*\cdot \vx-\vw_t\cdot \vx)}\Ind_{\cG}] \\
		&\;- 2\E_{\epi}[\abs{(\sigma(\vw_*\!\cdot\vx)-y)(\sigma'(\vw_*\cdot \vx)-\beta)(\vw_*\cdot \vx-\vw_t\cdot \vx)}\Ind_{\cG^{c}}] \\
        \ge \;  & \E_{\epi}[\innp{\vv(\vw_t;\vx,y), \vw_*-\vw_t}] \\
		&\;- 2\beta\E_{\epi}[\abs{(\sigma(\vw_*\!\cdot\vx)-y)(\vw_*\cdot \vx-\vw_t\cdot \vx)}\Ind_{\cG}] \\
		&\;- 2\beta\E_{\epi}[\abs{(\sigma(\vw_*\!\cdot\vx)-y)(\vw_*\cdot \vx-\vw_t\cdot \vx)}\Ind_{\cG^{c}}] \\
		\ge \;  & \E_{\epi}[\innp{\vv(\vw_t;\vx,y), \vw_*-\vw_t}] \\
        & \;-  2\beta \E_{\epi}[(\Ind_\cG + \Ind_{\cG^c}) |(\sigma(\vw_*\!\cdot\vx)-y)(\vw_*\cdot \vx-\vw_t\cdot \vx)|]\\
		= \; &  \E_{\epi}[\innp{\vv(\vw_t;\vx,y), \vw_*-\vw_t}] \\
        &\; -  2\beta \E_{\epi}[|(\sigma(\vw_*\!\cdot\vx)-y)(\vw_*\cdot \vx-\vw_t\cdot \vx)|].
	\end{align*}
	We apply Cauchy-Schwarz and \Cref{eq:moment-bounds-empirical} to get,
	\begin{align*}		&\E_{\epi}\bigl[\,|\sigma(\vw_*\!\cdot\vx)-y|\,|\vw_*\!\cdot\vx-\vw_t\!\cdot\vx|\bigr] \\
		\le\;& 
		\sqrt{\E_{\epi}[(\sigma(\vw_*\!\cdot\vx)-y)^2]}\;\sqrt{\E_{\epi}[(\vw_*\!\cdot\vx-\vw_t\!\cdot\vx)^2]}\\
		\le\;& 
		\sqrt{\mOPThat}\,\sqrt{6B}\,\|\vw_*-\vw_t\|_2.
	\end{align*}
	Young's inequality (Fact~\ref{fact:young}) then gives
    \begin{align*}
       \sqrt{\mOPThat}\,\sqrt{6B}\,\|\vw_*-\vw_t\|_2 
	\le 
	\frac{12\beta B}{c_1}\,\mOPThat
	+\frac{c_1}{8\beta}\|\vw_*-\vw_t\|_2^2.
    \end{align*}
	Combining these estimates, for each $i \in [K]$,
    \begin{align*}
        \LHS_{i} \ge \E_{\epi}[\innp{\vv(\vw_t;\vx,y),\,\vw_*-\vw_t}] -
	\Bigl(\frac{24\beta^2 B}{c_1}\,\mOPThat+\frac{c_1}{4}\|\vw_*-\vw_t\|_2^2\Bigr),	
    \end{align*}
	which completes the proof.
\end{proof}

\begin{claim}[\text{TV} via \(\chi^2\)/\text{KL} on the simplex]\label{claim:ub-TV}
    For \(\evlambda_0=\frac{\bm{1}}{K}\) and any \(t_1\), \(t_2\ge 0\), s.t. \(\evlambda_{t_1}\), \(\evlambda_{t_2}\in\Delta_K\), \(\big(\sum_{i=1}^K|\elambda_{t_1[i]}-\elambda_{t_2[i]}|\big)^2\le 2D_{\phi}(\evlambda_{t_1},\evlambda_{t_2})\), where \(\phi = \textit{KL}(\cdot,\evlambda_0)\) or \(\phi = \chi^2(\cdot,\evlambda_0)\).
\end{claim}
\begin{proof}
    \noindent\textbf{Case 1:} $\phi = \textit{KL}(\cdot,\evlambda_0)$. By Pinsker's inequality (\Cref{lem:pinsker}),
	\[
    (\sum_{i=1}^K|\elambda_{t_1[i]}-\elambda_{t_2[i]}|)^2=
	\| \evlambda_{t_1}-\evlambda_{t_2}\|_1^2
	\;\le\; \bigl(2\,TV(\evlambda_{t_1},\evlambda_{t_2})\bigr)^2
	\;\le\;2\,KL(\evlambda_{t_1},\evlambda_{t_2})
	\;=\;2\,D_{\text{KL}(\cdot,\evlambda_0)}(\evlambda_{t_1},\evlambda_{t_2}),
	\]
	so altogether
	\[
	(\sum_{i=1}^K|\elambda_{t_1[i]}-\elambda_{t_2[i]}|)^2\le 2D_{\phi}(\evlambda_{t_1},\evlambda_{t_2}).
	\]
	
	\medskip
	
	\noindent\textbf{Case 2:} $\phi = \chi^2(\cdot,\evlambda_0)$.
	Alternatively, divide and multiply each coordinate by \(\elambda_{0[i]}\) and use Jensen's inequality,
	\[
    (\sum_{i=1}^K|(\elambda_{t_1[i]}-\elambda_{t_2[i]})|)^2
	=(\sum_{i=1}^K|\frac{\elambda_{t_1[i]}-\elambda_{t_2[i]}}{\elambda_{0[i]}}|\;\elambda_{0[i]})^2 \le \sum_{i=1}^K
	\Bigl(\frac{\elambda_{t_1[i]}-\elambda_{t_2[i]}}{\elambda_{0[i]}}\Bigr)^{\!2}
	\,\elambda_{0[i]}
	= D_{\chi^2(\cdot,\evlambda_0)}(\evlambda_{t_1},\evlambda_{t_2}).
	\]
	Since $D_{\chi^2}\le2D_{\chi^2}$ trivially, this completes the bound
	\(
	(\sum_{i=1}^K|\elambda_{t_[i]}-\elambda_{t_2[i]}|)^2\le 2D_{\phi}(\evlambda_{t_1},\evlambda_{t_2}).
	\)
\end{proof}

\begin{claim}[Convergence Rate]\label{lemma:convergence-rate} For all \(t\ge 0\), let \(a_t\) be defined as in \Cref{line:choice-of-at}, then it holds that \(4\sqrt{3}C_W\beta BWa_{t-1}\leq \nu_0+\nu A_{t-2}\), \(\frac{2\sqrt{3}C_W\beta WBa_{t-1}}{\alpha_1}+\frac{36\cdot24\beta^4B^2a_{t-1}}{c_1}\leq \frac{1+0.5c_1A_{t-1}}{4}\), and \((24c_1+6\beta^2B+\frac{36\cdot 60\beta^4B^2}{c_1})a_t\leq \frac{1+0.5c_1A_t}{4}\). Note that the choice of \(a_t\) also makes the last inequality of \Cref{eq:final-residual} hold.
\end{claim}
\begin{proof}
    By Young's inequality (\Cref{fact:young}), \(6\beta^2B \leq 3c_1+3\beta^2B^4/c_1\), which means that to make the third inequality hold, it suffices that
    \[
        (27c_1+\frac{2163\beta^4B^2}{c_1})a_t\leq \frac{1+0.5c_1A_t}{4}.
    \]Let \(C_4=27c_1+2163\beta^4B^2/c_1\), it suffices to enforce
    \begin{align}\label{eq:a_t1}
         a_t\le(1+\frac{c_1}{8C_4})^{t-1}\frac{1}{4C_4}.
    \end{align}
    For the second inequality to hold, it suffices to enforce the two inequalities below:
    \[
        \frac{2\sqrt{3}C_W\beta WBa_t}{\alpha_1}\le \frac{1+0.5c_1A_t}{8}, \quad \frac{864\beta^4B^2}{c_1}a_t\le \frac{1+0.5c_1A_{t}}{8},
    \] where the second one holds if the first inequality listed in the proof holds. Therefore, to satisfy both the first and the second inequalities in the statement of the lemma,
    let \(C'_W=2\sqrt{3}C_W\beta WB\), it suffices:
    \[
        2C'_W\alpha_1 a_t\le \nu_0+\nu A_{t-1} , \quad \frac{C'_W a_t}{\alpha_1}\le \frac{1+0.5c_1A_t}{8}.
    \]
    For each iteration \(t\), we choose \(\alpha_1=2\sqrt{(\nu_0+\nu A_{t-1})/(2+c_1A_t)}\)
    and it suffices:
    \begin{align}\label{eq:a_t2}
         a_t\le \frac{1}{4\sqrt{2}C_W'}\sqrt{(\nu_0+\nu A_{t-1})(2+c_1A_t)}.
    \end{align}
    Using \(A_{t-1}\le A_t\) again, it suffices that \(a_1\le \sqrt{\nu_0}/(4C_W')\) and \(a_t \le \sqrt{c_1\nu}A_{t-1}/(4\sqrt{2}C_W')\) for \(t\ge 2\), which means
    \[
        a_t\le (1+\frac{\sqrt{c_1\nu}}{4\sqrt{2}C'_W})^{t-1}\frac{\sqrt{\nu_0}}{4C_W'},
    \]	
    Note that when \(\nu=0\), to satisfy \Cref{eq:a_t2}, it also suffices  \(a_t^2\le \nu_0c_1A_t/(4\sqrt{2}C_W')^2\), which means it suffices to have \(a_t \le \nu_0c_1t/(4\sqrt{2}C_W')^2\). Therefore, together with \Cref{eq:a_t1}, it suffices
    \begin{align*}
        a_t= \min\Big\{\Big(1+\frac{c_1}{8C_4})^{t-1}\frac{1}{4C_4},\max\Big\{\Big(1+\frac{\sqrt{c_1\nu}}{4\sqrt{2}C'_W}\Big)^{t-1}\frac{\sqrt{\nu_0}}{4C_W'},\frac{c_1\nu_0}{(4\sqrt{2}C'_W)^2}t\Big\}\Big\}, \quad\;A_n=\sum_{t=0}^na_t.
    \end{align*}
    The analysis above provides a valid choice for \(a_t\) that satisfies the conditions of the claim, which concludes the proof.
\end{proof}

\subsection{Proof of \Cref{claim:iteration-norm-ub}}\label{sec: restate-of-iteration-norm-ub}
Finally, we inductively establish the boundedness of the iterates that is necessary for the sharpness results (\Cref{fact:sharpness} and \Cref{lemma:sharpness-empirical}).

\claimiterub*

\begin{proof}
We prove this lemma by induction on $t$. When \(t=0\), it is trivial that \(\vzero=\vw_0\in \cB(3\|\vw_*\|)\). Now assume \(\|\vw_t\|\le 3\|\vw_*\|\) holds for all \(0\le t \le n\), we would like to prove \(\|\vw_{n+1}\|\le 3\|\vw_*\|\) also holds. We apply \Cref{lemma:gap-lower-bound} and \Cref{lemma:gap-upper-bound} to sandwich the quantity \(\sum_{t=1}^{n+1}a_t\Gap(\vw_t,\evlambda_t)\). Note that we only have \(\|\vw_{n}\|\le 3\|\vw_*\|\), which means that we can only add the term \(a_{n+1}\Gap(\vw_{n+1},\evlambda_{n+1})\) to the lower bound in \Cref{lemma:gap-lower-bound}. However, we can directly apply \Cref{lemma:gap-upper-bound} for the \((n+1)\)-th iteration since it only requires \(\|\vw_n\|\le 3\|\vw_*\|\). Therefore, we get the inequality below:
    \begin{align}\label{eq:sandwitch}
          & \; -\frac{12\beta^2 {B}}{c_1}\mOPThat A_{n} +  \sum_{t=1}^n a_{t}\frac{c_1}{2} \|\vw_{t} - \vw_*\|_{2}^{2} +  \sum_{t=1}^n \nu a_{t}D_{\phi}(\evlambda^*,\evlambda_{t})  + a_{n+1} \Gap(\vw_{n+1}, \evlambda_{n+1})\notag\\ 
          \le & \sum_{t=1}^{n+1} a_t \Gap(\vw_t, \evlambda_t)\notag \\
            \leq & \frac{1}{2}\|\vw_* - \vw_0\|_2^2 + \nu_0 D_\phi(\evlambda^*, \evlambda_0)
             - \frac{1 + 0.5 c_1A_{n+1}}{4}\|\vw_* - \vw_{k+1}\|_2^2 - (\nu_0 + \nu A_{n+1})D_\phi(\evlambda^*, \evlambda_{n+1})\notag\\ 
             &+ \frac{28 \beta^2 B\mOPThat A_{n+1}}{c_1}. 
    \end{align}
Also, similar to \Cref{lemma:gap-lower-bound} (see the proof of \Cref{lemma:gap-lower-bound} in \Cref{app:gap-lower-bound}), we split \(a_{n+1}\Gap(\vw_{n+1},\evlambda_{n+1)}\) into two terms and get
    \begin{align*}
        a_{n+1}\Gap(\vw_{n+1},\evlambda_{n+1})&=\;[L(\vw_{n+1},\evlambda^*)-L(\vw_*,\evlambda^*)]+[L(\vw_*,\evlambda^*)-L(\vw_*,\evlambda_{n+1})] \\
        &=\;\sum_{i=1}^K\elambda_{[i]}^{*}\E_{\epi}[((\sigma(\vw_{n+1} \cdot \vx)-y)^{2}-(\sigma(\vw_* \cdot \vx)-y)^{2})]+\nu D_{\phi}(\evlambda^*,\evlambda_{n+1}) \\
        &\ge \;-\mOPThat, \notag
    \end{align*}
where we lower bound the first term as \(-\mOPThat\) and simply ignore the second term due to the nonnegativity of Bregman divergence. Since \(\sum_{i=1}^na_{t}\frac{c_1}{2} \|\vw_{t} - \vw_*\|_{2}^{2}\) and \(\sum_{t=1}^n \nu a_{t}D_{\phi}(\evlambda^*,\evlambda_{t})\)in LHS of \Cref{eq:sandwitch} are nonnegative, \(-(\nu_0 + \nu A_{n+1})D_\phi(\evlambda^*, \evlambda_{n+1})\) and \((1+0.5c_1A_{n+1})\|\vw_{n+1}-\vw_n\|_2^2/4\) in RHS are nonpositive, we can ignore them. Plugging the inequality above into \Cref{eq:sandwitch} and rearranging the terms, we get
    \begin{align*}
       &\; \frac{2 + c_1A_{n+1}}{8}\|\vw_* - \vw_{n+1}\|_2^2\\
       \le\; & \frac{1}{2}\|\vw_* - \vw_0\|_2^2 + \nu_0 D_\phi(\evlambda^*, \evlambda_0)+\big(\frac{40\beta^2B}{c_1}A_n+(1+\frac{28\beta^2B}{c_1})a_{n+1}\big)\mOPThat.
    \end{align*}
Let both sides be divided by \((2+c_1A_{n+1})/8\), for the first two terms in RHS, we use \(2+c_1A_{n+1}\ge 2\), for the third term, we use \(2+c_1A_{n+1}\ge c_1A_n\), and for the last term, following the choice of \(a_t\), we have \(\frac{a_{n+1}}{2+c_1A_{n+1}}\le \frac{a_{n+1}}{c_1A_n} \le \max\{\frac{1}{n}, \frac{1}{8C_4}\} \le 1\). Therefore, the inequality above becomes
    \begin{align*}
        \|\vw_* - \vw_{n+1}\|_2^2 \le 2\|\vw_* - \vw_0\|_2^2 + 4\nu_0 D_\phi(\evlambda^*, \evlambda_0)+\big(\frac{544\beta^2B}{c_1^2}+\frac{1}{\min\{\frac{n}{8},C_4\}}\big)\mOPThat.
    \end{align*}
Since \(\frac{1}{\min\{\frac{n}{8},C_4\}}\) is at most of constant order, the coefficient of \(\mOPThat\) is also a constant. Choosing \(\nu_0=\frac{\eps}{4K}\), following the similar logic of claim E.2 in \citep{Li2024shifts}, we can assume without loss of generality that \(\big(\frac{544\beta^2B}{c_1^2}+\frac{1}{\min\{\frac{n}{8},C_4\}}\big)\mOPThat+\eps\le \norm{\vw_*}_2^2\), otherwise we can compare the empirical risk of the output from our algorithm and of \(\hat{\vw} = \vzero\) and output the solution with the lower risk to obtain an \(O(\OPT)+\eps\) solution.

Now we complete the induction step showing \(\|\vw_* - \vw_{n+1}\|_2^2 \le 3\|\vw_*\|_2^2\), which means \(\|\vw_{n+1}\|_2 \le (1+\sqrt{3})\|\vw_*\|_2 \le 3\|\vw_*\|_2\), and we finish the proof. 
\end{proof}

\section{Supplementary Details of Experiments}\label{sec:supply-experiments-details}
\subsection{Existing Assets Used}
We used the publicly available RedPajama dataset \citep{together2023redpajama}, which is released under a combination of open licenses consistent with the licenses of the original data sources (e.g., CC-BY for Wikipedia). We followed the official RedPajama license statement. Since the data from the \texttt{book} domain is no longer publicly available, we instead downloaded approximately 160GB of raw data from the remaining six domains, using normalized weights based on the original dataset’s initial proportions.

For the code, we built on the code base of \citep{xia2024shearedllamaacceleratinglanguage} (see \citep{shearedllama_code}), adding our implementation of the primal–dual methods below and modifying the function \texttt{update\_proportion} in the \texttt{dynamic\_loading\allowbreak\_callback.py} file accordingly.
\begin{lstlisting}[language=Python, caption={PD-KL updates}]
elif self.update_type == "pd-kl":
    new_lambdas = torch.log(new_lambdas + 1e-6) + eta * diff 
    new_lambdas = torch.nn.functional.softmax(new_lambdas, dim=0)
    updated_domain_weights = \
        new_lambdas + extrapolation_factor * (new_lambdas - torch.tensor(current_lambdas)) # extrapolation
    updated_domain_weights = (1-c) * updated_domain_weights + c / self.n_domains     
\end{lstlisting}

\subsection{Additional Experiment Details}
We largely followed the pipelines and instructions provided in the aforementioned code base. For data preparation and model setup, we used virtual machines on Google Cloud Platform (GCP): a VM with 8 vCPUs and 64GB memory for tokenization and data sampling, and a VM with a single NVIDIA A100 80GB GPU for converting checkpoints into HuggingFace format and running model evaluations. For training, we employed 4 NVIDIA A100 80GB GPUs on the high-performance computing clusters of the Center for High Throughput Computing (CHTC) at UW-Madison, equipped with 32 CPUs and 256GB memory. We followed all training parameters from \citep{xia2024shearedllamaacceleratinglanguage}, except that we set the evaluation interval to 8.4M tokens instead of 16.8M tokens, i.e., twice as frequent as in their setup.

\end{document}